\documentclass[11pt]{article}

\usepackage[margin=1in]{geometry}
\usepackage[T1]{fontenc}
\usepackage{lmodern}
\usepackage{microtype}
\usepackage{array}
\usepackage{booktabs}
\usepackage{longtable}
\usepackage{pdflscape}
\usepackage{amsmath,amssymb,amsthm}
\usepackage{graphicx}
\usepackage{xcolor}
\usepackage[section]{placeins}
\usepackage{needspace}
\usepackage{pgfplots}
\pgfplotsset{compat=1.17}
\usepgfplotslibrary{groupplots}
\usetikzlibrary{arrows.meta}
\usepackage[numbers,sort&compress]{natbib}
\usepackage[colorlinks=true,linkcolor=blue!50!black,citecolor=blue!50!black,urlcolor=blue!50!black]{hyperref}

\definecolor{frontier}{RGB}{0,113,227}
\newtheorem{proposition}{Proposition}
\DeclareRobustCommand{\code}[1]{\texttt{\detokenize{#1}}}
\title{{\Huge\bfseries Think Twice}\\[8pt]
{\Large The Metis-1.6 Technical Report}\\[6pt]
{\normalsize\itshape Two looped mixture-of-experts models, one trillion tokens,\\ and everything I learned the hard way}}
\author{Giulianno Vollmer\\Lernex Research\\\texttt{giulianno@lernex.com}}
\date{October 27, 2026}

\begin{document}
\maketitle

\begin{abstract}
Metis is a bet that better data, better training methods and new architectures can buy what a bigger parameter count cannot, at a fraction of the cost. Metis-1.6 is its latest step: two language models built from scratch by a single researcher, Praxis (2.94B parameters, 0.46B active per pass) and Logos (10.23B parameters, 1.18B active per pass). Both run a hybrid Mamba-2 and attention mixture-of-experts stack twice for every token, so each token gets a second look. Each was trained on 1.02T tokens in total: 804.76B of pretraining, 18B of context extension to 131,072 tokens, and 200B of supervised fine-tuning. Metis-1.6 was never meant to challenge frontier models, and across 40 post-trained benchmarks it mostly does not: Logos averages 52.93, seventh of 17 models in its comparison set and within a point of Qwen3-4B. But it gets there on roughly a twelfth of the median training data of the comparators that disclose it, and with fewer active parameters per token than any model that beats it, which puts it on the Pareto frontier for both. In mathematics it ranks second of 30 models, with 67.50\% on AIME 2025 and 63.64\% on HMMT February 2026. It is weak at code, instruction following and long context. The rest of the report is the part I found most valuable: what went wrong. The optimizer I configured never ran, and attention trained with AdamW for all of pretraining. A 1.77B-parameter n-gram memory learned to switch itself off. Five fine-tuning choices cost capability. And a length reward under group-mean normalization punished correct solutions for being long once a group's pass rate passed 2/3, a sign inversion I prove and fix with a one-line clamp. I close with the 10B-token ablations behind MoRE (Mixture of Recursive Experts), the looped architecture of Metis-1.7.
\end{abstract}

\tableofcontents
\newpage

%=====================================================================
\section{Introduction}

Most of the progress in language models over the past few years has been bought with scale. When a model is not good enough, the standard answer is to make it bigger, and parameter counts are announced like trophies, as if size were the achievement. Bigger models are more capable, but they are also far more expensive to train and to serve, and the price of the most capable models keeps climbing even as the labs building them promise to bring it down. Too often the safe choice wins: another trillion parameters on a recipe from 2023, rather than a better recipe.

Metis is Lernex's bet on the opposite direction. Better training data, better training algorithms and genuinely new architectures can buy capability that parameters alone cannot, and at a fraction of the cost. A single researcher cannot outspend a frontier lab, but can out-iterate one: try an idea the day it occurs, build on the newest results instead of the safest ones, and take risks on designs a large organization would not. Every part of the Metis stack is written from scratch, from data acquisition and decontamination to tokenization, training kernels, reinforcement-learning environments and evaluation, so any of it can change the moment a better idea arrives.

Metis-1.6 was never meant to beat frontier models. It was built to test that bet at small scale, and to answer a more personal question: can I, starting from nothing, build each model better than the last, learn from every mistake, and get measurably better at the craft? Each Metis release is one step of that answer. Metis-1.5 was the first Metis model you could really talk to. Metis-1.6 is the first one you can realistically use: it works through olympiad problems, calls tools and follows instructions well enough to be genuinely helpful in places. Metis-1.7 comes next, and this report is, in large part, the story of how I got ready to build it.

Four generations of Metis were trained in 2026:
\begin{itemize}
  \item \textbf{Metis-1.3} (201M parameters): a 28-layer dense hybrid of 21 Mamba-2 and 7 grouped-query attention blocks, with an 8,192-token vocabulary and a 4,096-token context.
  \item \textbf{Metis-1.4} (503.8M parameters): a mixture-of-recursions transformer that explored weight-tied recurrent depth and taught me a great deal about label-shift and residual-scaling instabilities in looped models.
  \item \textbf{Metis-1.5} (897.4M total, 339.6M active parameters): a 19-layer latent mixture-of-experts decoder ($d_{\text{model}} = 1{,}536$, 32 routed experts with top-4 routing and one shared expert, 32,768-token vocabulary), pretrained on 50B tokens and fine-tuned as \code{Metis-1.5-think}.
  \item \textbf{Metis-1.6} Praxis and Logos: the two-pass looped hybrids of this report, pretrained on sixteen times the data of Metis-1.5, on AMD Instinct MI300A accelerators.
\end{itemize}

This report has two jobs. The first is to measure the models honestly. Every comparison comes from one harness that ran Metis and every comparison model with identical prompts and scorers, and every comparison table is generated from the benchmark data published with the models. The second job is to write down what went wrong. Release reports usually describe the recipe that worked. This one also describes an optimizer that silently differed from its configuration for the entire pretraining run, a component the model learned to switch off, and post-training choices that cost capability. I would rather publish my mistakes than repeat them, and I suspect they are the most useful pages here for anyone else building models without a frontier budget.

\needspace{8\baselineskip}
\paragraph{Contributions.}
\begin{enumerate}
  \item Two open models with a complete training lineage, checkpoint hashes and exact parameter accounting (Sections~\ref{sec:arch} to~\ref{sec:praxis}), evaluated with one harness against open models of similar size: 27 at the base stage and 28 after post-training (Section~\ref{sec:eval} and Appendices~\ref{app:base} and~\ref{app:post}).
  \item An audit of the pretraining optimizer. The declared AdaMuon optimizer was never instantiated, and because of a parameter-routing mismatch every attention matrix trained with AdamW (Section~\ref{sec:optim}).
  \item A negative result on hashed n-gram memory. In both models the learned injection gates closed, and deleting the 1.77B-parameter table from Logos changed held-out loss by less than $10^{-3}$ nats (Section~\ref{sec:ngram}).
  \item Post-training lessons: five supervised fine-tuning failure modes; a verifier imbalance that shortened reasoning on unseen science questions; and a sign inversion of length rewards under group-mean normalization, proved and observed, with a one-line fix (Sections~\ref{sec:sft} and~\ref{sec:rl}).
  \item Controlled 10B-token ablations of the loop design and optimizer behind MoRE, the architecture of Metis-1.7 (Section~\ref{sec:ablations}).
\end{enumerate}

\paragraph{Results in brief.} Figure~\ref{fig:frontier} is the picture I would put on the fridge. Across 40 post-trained benchmarks, Logos averages 52.93: seventh of 17 models in its comparison set and 0.75 points behind Qwen3-4B. That is not a frontier result, and the open models from large labs beat it more often than not. But Logos was trained on 1.02T tokens, against a median of 12T for the comparators that disclose a count, and every model above it that discloses its data was trained on at least 28T. It is also a mixture-of-experts model running against mostly dense ones: it activates 2.36B parameters per token, while every model that beats it activates between 1.5 and 4.6 times as many. On both axes, Logos sits on the Pareto frontier, and in the shaded corner of each plot: a higher score than the median model, on far less data and compute than it.

Where it shines, it shines properly. In mathematics Logos ranks second of 30 models, behind only Gemma~4~12B, with the highest score in its set on HMMT February 2025 (62.92), HMMT February 2026 (63.64) and IMO-AnswerBench (57.00). Where it struggles, it struggles properly too: code, instruction following, long context and factual recall all trail models trained on a small fraction of the data that the large labs use (Figure~\ref{fig:domains}). Praxis averages 36.62 and, like Logos, sits on the compute frontier. As base models, both trail open models of similar size, which I attribute mainly to training on one to two orders of magnitude fewer tokens (Appendix~\ref{app:base}).

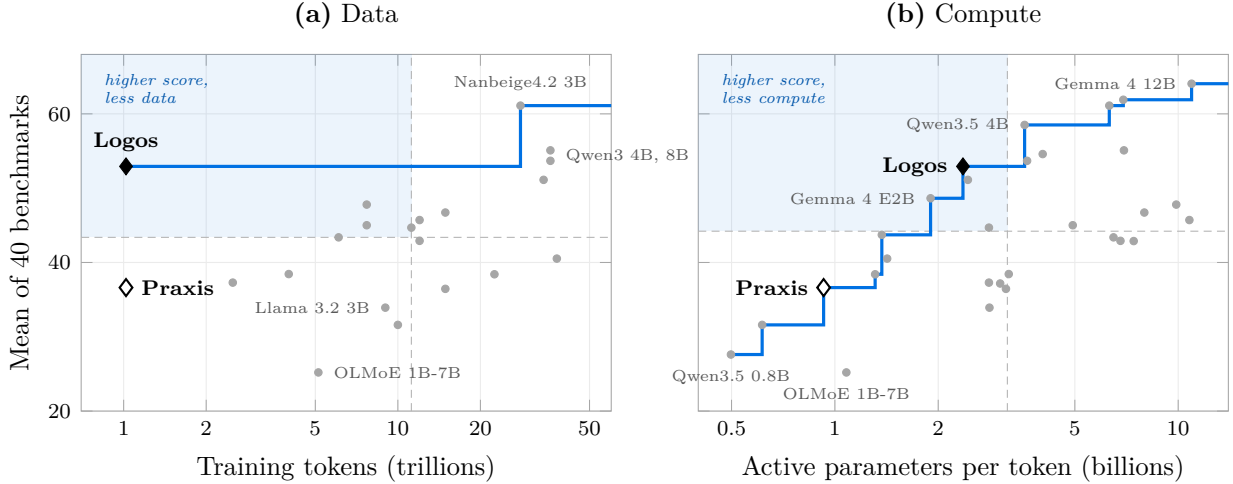
\begin{figure}[tbp]
\centering
\begin{tikzpicture}
\begin{groupplot}[group style={group size=2 by 1,horizontal sep=1.15cm},width=0.52\linewidth,height=6.3cm,
  grid=major,grid style={black!8},axis line style={black!35},tick style={black!35},tick label style={font=\scriptsize},label style={font=\small},clip=false,ymin=20,ymax=68,log ticks with fixed point]
\nextgroupplot[xmode=log,xmin=0.7,xmax=60,xtick={1,2,5,10,20,50},
  xlabel={Training tokens (trillions)},ylabel={Mean of 40 benchmarks},title={\small\textbf{(a)} Data}]
\fill[frontier,fill opacity=0.075] (axis cs:0.7,43.38) rectangle (axis cs:11.200,68);
\draw[black!30,densely dashed,line width=0.4pt] (axis cs:11.200,20) -- (axis cs:11.200,68);
\draw[black!30,densely dashed,line width=0.4pt] (axis cs:0.7,43.38) -- (axis cs:60,43.38);
\node[font=\tiny\itshape,text=frontier!75!black,anchor=north west,align=left] at (rel axis cs:0.025,0.975) {higher score,\\less data};
\addplot[const plot,very thick,frontier] coordinates { (1.02,52.93) (28.0,61.11) (60,61.11)};
\addplot[only marks,mark=*,mark size=1.4pt,black!35] coordinates { (36.0,53.68) (38.0,40.52) (22.5,38.41) (10.0,31.60) (5.13,25.20) (36.0,55.09) (12.0,45.71) (14.9,46.72) (6.08,43.38) (12.0,42.91) (28.0,61.11) (14.9,36.45) (2.5,37.29) (7.7,45.01) (11.2,44.69) (34.0,51.12) (4.0,38.44) (9.0,33.92) (7.7,47.79)};
\addplot[only marks,mark=diamond*,mark size=3.2pt,black] coordinates { (1.02,52.93)};
\addplot[only marks,mark=diamond*,mark size=3.2pt,mark options={fill=white,draw=black,thick}] coordinates { (1.02,36.62)};
\node[font=\scriptsize\bfseries,text=black,anchor=south,yshift=2pt] at (axis cs:1.02,52.93) {Logos};
\node[font=\scriptsize\bfseries,text=black,anchor=west,xshift=2pt] at (axis cs:1.02,36.62) {Praxis};
\node[font=\tiny,text=black!65,anchor=south,yshift=2pt] at (axis cs:28.0,61.11) {Nanbeige4.2 3B};
\node[font=\tiny,text=black!65,anchor=west,xshift=2pt] at (axis cs:5.13,25.20) {OLMoE 1B-7B};
\node[font=\tiny,text=black!65,anchor=east,xshift=-2pt] at (axis cs:9.0,33.92) {Llama 3.2 3B};
\node[font=\tiny,text=black!65,anchor=west,xshift=2pt] at (axis cs:36.0,54.38) {Qwen3 4B, 8B};
\nextgroupplot[xmode=log,xmin=0.4,xmax=14,xtick={0.5,1,2,5,10},yticklabels={},
  xlabel={Active parameters per token (billions)},title={\small\textbf{(b)} Compute}]
\fill[frontier,fill opacity=0.075] (axis cs:0.4,44.21) rectangle (axis cs:3.180,68);
\draw[black!30,densely dashed,line width=0.4pt] (axis cs:3.180,20) -- (axis cs:3.180,68);
\draw[black!30,densely dashed,line width=0.4pt] (axis cs:0.4,44.21) -- (axis cs:14,44.21);
\node[font=\tiny\itshape,text=frontier!75!black,anchor=north west,align=left] at (rel axis cs:0.025,0.975) {higher score,\\less compute};
\addplot[const plot,very thick,frontier] coordinates { (0.498,27.61) (0.614,31.60) (0.927204204,36.62) (1.31,38.41) (1.37,43.72) (1.9,48.64) (2.36097303,52.93) (3.57,58.51) (6.3,61.11) (6.92,61.90) (10.94,64.06) (14,64.06)};
\addplot[only marks,mark=*,mark size=1.4pt,black!35] coordinates { (3.57,58.51) (3.63,53.68) (1.42,40.52) (1.31,38.41) (0.614,31.60) (1.08,25.20) (6.92,61.90) (6.95,55.09) (10.94,64.06) (10.78,45.71) (7.97,46.72) (7.42,42.88) (6.48,43.38) (6.79,42.91) (4.03,54.59) (6.3,61.11) (3.03,37.17) (3.15,36.45) (2.81,37.29) (4.93,45.01) (2.81,44.69) (2.44,51.12) (3.21,38.44) (1.37,43.72) (2.82,33.92) (1.9,48.64) (0.498,27.61) (9.87,47.79)};
\addplot[only marks,mark=diamond*,mark size=3.2pt,black] coordinates { (2.36097303,52.93)};
\addplot[only marks,mark=diamond*,mark size=3.2pt,mark options={fill=white,draw=black,thick}] coordinates { (0.927204204,36.62)};
\node[font=\scriptsize\bfseries,text=black,anchor=east,xshift=-2pt] at (axis cs:2.36097303,52.93) {Logos};
\node[font=\scriptsize\bfseries,text=black,anchor=east,xshift=-2pt] at (axis cs:0.927204204,36.62) {Praxis};
\node[font=\tiny,text=black!65,anchor=east,xshift=-2pt] at (axis cs:10.94,64.06) {Gemma 4 12B};
\node[font=\tiny,text=black!65,anchor=east,xshift=-2pt] at (axis cs:3.57,58.51) {Qwen3.5 4B};
\node[font=\tiny,text=black!65,anchor=north,yshift=-2pt] at (axis cs:0.498,27.61) {Qwen3.5 0.8B};
\node[font=\tiny,text=black!65,anchor=east,xshift=-2pt] at (axis cs:1.9,48.64) {Gemma 4 E2B};
\node[font=\tiny,text=black!65,anchor=north,yshift=-2pt] at (axis cs:1.08,25.20) {OLMoE 1B-7B};
\end{groupplot}
\end{tikzpicture}
\caption{Mean score on the 40 post-trained benchmarks against (a) published training tokens and (b) active parameters per token. The blue staircase is the Pareto frontier: the best score reached with at most that much data or compute. Logos (filled diamond) and Praxis (open diamond) are both at 1.02T tokens; gray dots are the comparison models. Dashed lines mark the median of the plotted models on each axis, and the shaded corner holds the models that beat the median score with less data (a) or compute (b) than the median. Panel (a) omits models that do not publish a training-token count (Qwen3.5 and Gemma~4) or publish only a range (Ministral~3).}
\label{fig:frontier}
\end{figure}

\paragraph{Availability.} The four models, \code{Metis-1.6-Logos}, \code{Metis-1.6-Praxis} and their \code{-Base} counterparts, are published at \url{https://huggingface.co/Lernex}. All four ship without the n-gram memory (Section~\ref{sec:ngram}). Each repository includes its benchmark data, from which the comparison tables and figures in this report are generated.

%=====================================================================
\section{Related work}

\paragraph{Hybrid state-space models.} Mamba-2~\cite{dao2024mamba2} recasts selective state-space models as structured attention and makes them efficient on matrix hardware. Hybrids that keep a few attention layers among many state-space layers, such as Jamba~\cite{lieber2024jamba} and Nemotron-H~\cite{nvidia2025nemotronh,nemotron3super2026}, retain most of the quality of attention at a fraction of the inference cost. Metis-1.6 places three attention layers among twenty (Logos) and two among twelve (Praxis), with no positional encoding.

\paragraph{Looped and recurrent-depth models.} Applying a shared block repeatedly goes back to the Universal Transformer~\cite{dehghani2018universal}. Recent work scales recurrent depth for latent reasoning~\cite{geiping2025recurrent}, pretrains looped language models at scale~\cite{zhu2025looped}, routes tokens to different recursion depths~\cite{bae2025mor}, and adds gated highways between loop iterations of mixture-of-experts models~\cite{loom2026}. Metis-1.6 loops its entire stack a fixed two times.

\paragraph{Mixture-of-experts.} Metis-1.6 follows fine-grained experts with a shared expert~\cite{dai2024deepseekmoe}, auxiliary-loss-free load balancing through routing biases~\cite{wang2024auxloss} and a router $z$-loss~\cite{zoph2022stmoe}. Experts operate in a compressed latent space as in LatentMoE~\cite{latentmoe2026}, with SwiGLU units~\cite{shazeer2020glu}.

\paragraph{Residual streams and memory.} Hyper-connections~\cite{zhu2024hyperconnections} widen the residual stream into several streams with learned mixing. Manifold-constrained hyper-connections~\cite{xie2025mhc,deepseekv4_2026} project the mixing matrix toward the doubly stochastic matrices with Sinkhorn--Knopp iterations~\cite{sinkhorn1967doubly}. Conditional memory through hashed n-gram lookup~\cite{engram2026} adds a sparse parameter axis; Section~\ref{sec:ngram} reports that the Metis models learned not to use it.

\paragraph{Optimizers.} Muon~\cite{jordan2024muon,liu2025muonscalable} orthogonalizes the momentum updates of matrix parameters with Newton--Schulz iterations. AdaMuon~\cite{si2025adamuon} and NorMuon~\cite{li2025normuon} add second-moment normalization, and Polar Express~\cite{polarexpress2026} derives optimal polynomial iterations for the orthogonalization.

\paragraph{Reinforcement learning with verifiable rewards.} GRPO~\cite{shao2024deepseekmath} normalizes rewards within a group of samples for the same prompt, and GSPO~\cite{zheng2025gspo} moves importance ratios and clipping to the sequence level. Dr.~GRPO~\cite{liu2025drgrpo} and DAPO~\cite{yu2025dapo} analyse and correct biases of the group objective. Length has been controlled with length penalties~\cite{kimi2025k15,arora2025efficient} and length-conditioned training~\cite{aggarwal2025l1}. The sign inversion of Section~\ref{sec:signinv} is a direct consequence of mean-centring and is closely related to these analyses. On-policy distillation~\cite{agarwal2023onpolicy} trains a student on its own samples against a teacher's distribution.

%=====================================================================
\section{Architecture: one stack, two passes}
\label{sec:arch}

Praxis and Logos are the same architecture at two sizes (Figure~\ref{fig:arch} and Table~\ref{tab:arch}). Instead of stacking more layers, Metis-1.6 runs its hybrid trunk twice for every token. Each token passes through 40 sequence-mixing and expert stages in Logos (20 layers, two passes) and 24 in Praxis (12 layers, two passes), while only half as many are stored. Parameter counts in this section are read directly from the tensor shapes of the released checkpoints.

\begin{figure}[!t]
\centering
% Metis-1.6 architecture: (a) the stack run twice, (b) one layer on four residual streams,
% (c) the latent mixture-of-experts. Logos is drawn; Praxis differs only in the numbers.
\definecolor{mambaL}{RGB}{226,238,255}\definecolor{mambaD}{RGB}{0,113,227}
\definecolor{attnL}{RGB}{255,238,220}\definecolor{attnD}{RGB}{222,118,20}
\definecolor{moeL}{RGB}{222,245,235}\definecolor{moeD}{RGB}{18,148,104}
\definecolor{memL}{RGB}{240,234,255}\definecolor{memD}{RGB}{122,82,228}
\definecolor{strm}{RGB}{78,90,114}
\tikzset{
  arch/.style={font=\sffamily\scriptsize, >={Stealth[length=4pt,width=3.6pt]}, line cap=round},
  card/.style={rounded corners=7pt, fill=black!2.5, draw=black!11},
  blk/.style={draw, rounded corners=2.5pt, line width=0.6pt, align=center, inner sep=1.5pt},
  mamba/.style={blk, draw=mambaD, fill=mambaL},
  attn/.style={blk, draw=attnD, fill=attnL},
  moe/.style={blk, draw=moeD, fill=moeL},
  mem/.style={blk, draw=memD, fill=memL},
  plain/.style={blk, draw=black!45, fill=white},
  ghost/.style={blk, draw=black!40, dashed, fill=white, text=black!55},
  flow/.style={->, line width=0.6pt, black!60},
  stream/.style={line width=0.95pt, strm},
  port/.style={draw=strm, fill=white, rounded corners=1.5pt, line width=0.6pt},
  ptitle/.style={font=\sffamily\small\bfseries, anchor=west},
  pnote/.style={font=\sffamily\tiny, text=black!55, anchor=west},
  note/.style={font=\sffamily\tiny, text=black!60},
}
% one layer of the stack: the sequence mixer, with its mixture-of-experts tab on the right
% #1 centre y, #2 height, #3 fill, #4 outline, #5 label
\newcommand{\metislayer}[5]{%
  \begin{scope}
    \clip[rounded corners=2.5pt] (1.5,#1-#2/2) rectangle (3.9,#1+#2/2);
    \fill[#3] (1.5,#1-#2/2) rectangle (3.35,#1+#2/2);
    \fill[moeL] (3.35,#1-#2/2) rectangle (3.9,#1+#2/2);
  \end{scope}
  \draw[moeD!60, line width=0.5pt] (3.35,#1-#2/2) -- (3.35,#1+#2/2);
  \draw[#4, line width=0.6pt, rounded corners=2.5pt] (1.5,#1-#2/2) rectangle (3.9,#1+#2/2);
  \node at (2.425,#1) {#5};
  \node[font=\sffamily\tiny\bfseries, text=moeD] at (3.625,#1) {MoE};}
% a stream-mixing matrix (Sinkhorn-projected, so rows and columns each sum to one), as a heat grid
\newcommand{\metismix}[1]{%
  \foreach \r/\vals in {0/{55,20,15,10},1/{20,50,20,10},2/{15,20,50,15},3/{10,10,15,65}}{%
    \foreach \v [count=\c from 0] in \vals {%
      \fill[strm!\v!white, draw=white, line width=0.4pt]
        (6.5+0.2375*\c, #1+0.5625-0.1875*\r) rectangle ++(0.2375,0.1875);}}%
  \draw[strm, line width=0.6pt, rounded corners=1.5pt] (6.5,#1) rectangle ++(0.95,0.75);}

\begin{tikzpicture}[arch]
\path[card] (0,0) rectangle (5.9,9.05);
\path[card] (6.15,0) rectangle (10.75,9.05);
\path[card] (11.0,0) rectangle (16.4,9.05);

% ================================================================ (a) the stack
\node[ptitle] at (0.15,8.72) {(a) One stack, run twice};
\node[pnote] at (0.15,8.38) {every weight is shared between the two passes};

\node[plain, minimum width=2.4cm, minimum height=0.5cm] (emb) at (2.7,0.5)
  {Token embedding\\[-1pt]{\tiny fanned out to 4 residual streams}};
\metislayer{1.700}{1.00}{mambaL}{mambaD}{Mamba-2\enspace$\times$5}
\metislayer{2.530}{0.42}{attnL}{attnD}{Attention}
\metislayer{3.285}{0.85}{mambaL}{mambaD}{Mamba-2\enspace$\times$4}
\metislayer{4.040}{0.42}{attnL}{attnD}{Attention}
\metislayer{4.795}{0.85}{mambaL}{mambaD}{Mamba-2\enspace$\times$4}
\metislayer{5.550}{0.42}{attnL}{attnD}{Attention}
\metislayer{6.305}{0.85}{mambaL}{mambaD}{Mamba-2\enspace$\times$4}
\node[plain, minimum width=2.4cm, minimum height=0.5cm] (head) at (2.7,7.8)
  {Final norm, LM head\\[-1pt]{\tiny tied to the embedding}};

\draw[flow] (emb.north) -- (2.7,1.2);
\draw[flow] (2.7,6.73) -- node[note, left, pos=0.72, xshift=-1pt] {after pass 2} (head.south);

% the loop: pass 1's output goes round past the depth memory and the stack runs again
\draw[line width=1.15pt, mambaD, rounded corners=8pt] (2.7,7.05) -- (4.85,7.05) -- (4.85,5.0);
\draw[line width=1.15pt, mambaD, rounded corners=8pt, ->] (4.85,3.3) -- (4.85,0.98) -- (2.77,0.98);
\fill[mambaD] (2.7,7.05) circle (1.5pt);
\node[circle, fill=mambaD, text=white, inner sep=1.3pt, font=\sffamily\tiny\bfseries] at (3.85,7.05) {$\times$2};
\node[mem, minimum width=1.25cm, minimum height=1.7cm] at (4.85,4.15)
  {\textbf{Depth}\\\textbf{memory}\\[1pt]{\tiny 4 heads, 256-d}\\[-1pt]{\tiny pass 1 writes,}\\[-2pt]{\tiny pass 2 reads}};

% the n-gram memory the models switched off
\node[ghost, minimum width=0.95cm, minimum height=1.3cm] (ng) at (0.62,2.12)
  {\textbf{N-gram}\\memory\\[1pt]{\tiny gates closed}};
\draw[->, dashed, black!40, line width=0.55pt] (ng.east |- 0,1.7) -- (1.48,1.7);
\draw[->, dashed, black!40, line width=0.55pt] (ng.east |- 0,2.53) -- (1.48,2.53);
\node[note, text=black!45, align=center, anchor=north] at (0.62,1.4) {into layers\\[-1pt]2 and 5};

% ================================================================ (b) one layer
\node[ptitle] at (6.3,8.72) {(b) Inside one layer};
\node[pnote] at (6.3,8.38) {each read, write and mix is pass-aware};

\foreach \x in {6.6,6.85,7.1,7.35}
  \draw[stream, -{Stealth[length=3pt,width=3pt]}] (\x,0.75) -- (\x,8.05);
\node[note, anchor=north] at (6.975,0.62) {4 residual streams};

\metismix{3.75}
\metismix{7.0}
\node[note, anchor=west, align=left] at (7.55,4.12) {mix: $4{\times}4$ doubly\\stochastic (Sinkhorn)};

% sub-layer 1: the sequence mixer
\node[blk, draw=black!45, left color=mambaL, right color=attnL, minimum width=2.0cm, minimum height=1.45cm] (mix) at (9.55,2.6)
  {\textbf{Mixer}\\Mamba-2 or\\NoPE attention\\[-1pt]{\tiny + rank-8 pass LoRA}};
\filldraw[port] (6.5,2.2) rectangle ++(0.95,0.2);
\filldraw[port] (6.5,2.9) rectangle ++(0.95,0.2);
\draw[flow] (7.45,2.3) -- node[note, above, pos=0.45] {read} (mix.west |- 0,2.3);
\draw[flow] (mix.west |- 0,3.0) -- node[note, above, pos=0.55] {write} (7.45,3.0);

% sub-layer 2: the latent mixture-of-experts
\node[moe, minimum width=2.0cm, minimum height=1.45cm] (moe) at (9.55,5.9)
  {\textbf{Latent MoE}\\top-4 of 192\\+ 1 shared};
\filldraw[port] (6.5,5.5) rectangle ++(0.95,0.2);
\filldraw[port] (6.5,6.2) rectangle ++(0.95,0.2);
\draw[flow] (7.45,5.6) -- node[note, above, pos=0.45] {read} (moe.west |- 0,5.6);
\draw[flow] (moe.west |- 0,6.3) -- node[note, above, pos=0.55] {write} (7.45,6.3);
\draw[dashed, moeD, line width=0.6pt, ->] (moe.east) -- (10.95,5.9);

% ================================================================ (c) latent MoE
\node[ptitle] at (11.15,8.72) {(c) Latent mixture-of-experts};
\node[pnote] at (11.15,8.38) {experts work in a shared 1,024-d space};

\node (x) at (13.75,0.42) {$x$\; {\tiny 2,560-d}};
\filldraw[draw=black!45, fill=white, line width=0.6pt, rounded corners=1.5pt]
  (12.5,1.0) -- (15.0,1.0) -- (14.45,1.75) -- (13.05,1.75) -- cycle;
\node at (13.75,1.5) {$W_{\text{down}}$};
\node[note] at (13.75,1.17) {2,560 $\to$ 1,024};
\draw[flow] (x.north) -- (13.75,1.0);

% router
\node[plain, minimum width=0.82cm, minimum height=0.62cm] (rt) at (11.62,3.375) {Router\\[-1pt]{\tiny top-4}};
\draw[flow] (x.west) -| (rt.south);

% routed experts: 21 of 192 drawn, the four this token picked filled in
\foreach \r in {0,1,2} \foreach \c in {0,...,6}
  \filldraw[draw=moeD!45, fill=white, line width=0.45pt, rounded corners=1pt]
    (12.42+0.38*\c, 2.9+0.34*\r) rectangle ++(0.32,0.27);
\foreach \r/\c in {0/1, 1/4, 2/2, 2/6}
  \filldraw[draw=moeD, fill=moeD!75, line width=0.45pt, rounded corners=1pt]
    (12.42+0.38*\c, 2.9+0.34*\r) rectangle ++(0.32,0.27);
\node[note, anchor=north west, align=left] at (12.37,2.84) {192 routed\\[-1pt]experts};
\draw[flow] (rt.east) -- node[note, above, inner sep=1.5pt] {$\hat g_e$} (12.4,3.375);

% shared expert
\node[moe, fill=moeD!22, minimum width=0.8cm, minimum height=0.95cm] (sh) at (15.75,3.375) {Shared\\[-1pt]{\tiny always on}};

% z into the routed and shared experts
\draw[flow] (13.75,1.75) -- node[note, right, pos=0.3] {$z$} (13.75,2.88);
\draw[flow] (13.75,2.3) -| (sh.south);

% combine, then project back up
\node[circle, draw=black!50, fill=white, inner sep=0pt, minimum size=0.36cm, line width=0.6pt] (sum) at (13.75,4.7) {$+$};
\draw[flow] (13.75,3.86) -- (sum.south);
\draw[flow] (sh.north) |- (sum.east);
\filldraw[draw=black!45, fill=white, line width=0.6pt, rounded corners=1.5pt]
  (13.05,5.15) -- (14.45,5.15) -- (15.0,5.9) -- (12.5,5.9) -- cycle;
\node at (13.75,5.4) {$W_{\text{up}}$};
\node[note] at (13.75,5.73) {1,024 $\to$ 2,560};
\draw[flow] (sum.north) -- (13.75,5.15);
\draw[flow] (13.75,5.9) -- node[note, right, pos=0.55] {back to the streams} (13.75,6.5);

\node at (13.7,7.45)
  {$\mathrm{FFN}(x) = W_{\text{up}}\big(S(z) + \textstyle\sum_{\text{top-}4}\, \hat g_e\, E_e(z)\big)$};
\node[note] at (13.7,7.0) {$z = W_{\text{down}}\, x$, \ $S$ shared, \ $E_e$ routed};

% ================================================================ legend
\begin{scope}[yshift=-0.42cm]
  \foreach \sty/\lab/\x in {mamba/Mamba-2/0.55, attn/Attention/3.0, moe/Mixture-of-experts/5.45, mem/Memory/8.55}{
    \node[\sty, minimum width=0.42cm, minimum height=0.24cm] at (\x,0) {};
    \node[note, anchor=west, text=black!70] at (\x+0.3,0) {\lab};}
  \draw[stream] (10.75,0) -- (11.2,0);
  \node[note, anchor=west, text=black!70] at (11.3,0) {Residual stream};
  \node[ghost, minimum width=0.42cm, minimum height=0.24cm] at (13.9,0) {};
  \node[note, anchor=west, text=black!70] at (14.2,0) {Removed at release};
\end{scope}
\end{tikzpicture}
\caption{Metis-1.6 on one page, drawn for Logos. (a) Twenty layers, three of them attention, each ending in a mixture-of-experts block. The stack runs twice per token with every weight shared, and a depth memory carries what the first pass wrote down into the second. The n-gram memory that fed layers 2 and 5 during pretraining is drawn as a ghost: the models closed its gates on their own, and I removed it before release (Section~\ref{sec:ngram}). (b) Inside a layer, each sub-layer reads from and writes to four residual streams, which are then mixed by a doubly stochastic $4\times4$ matrix; a controller that knows the current pass sets all three. (c) The mixture-of-experts block works in a 1,024-dimensional latent space: one shared expert and four of 192 routed experts run there before the result is projected back to model width. Praxis is the same design with 12 layers (attention at layers 4 and 8), width 2,048 and 128 routed experts.}
\label{fig:arch}
\end{figure}
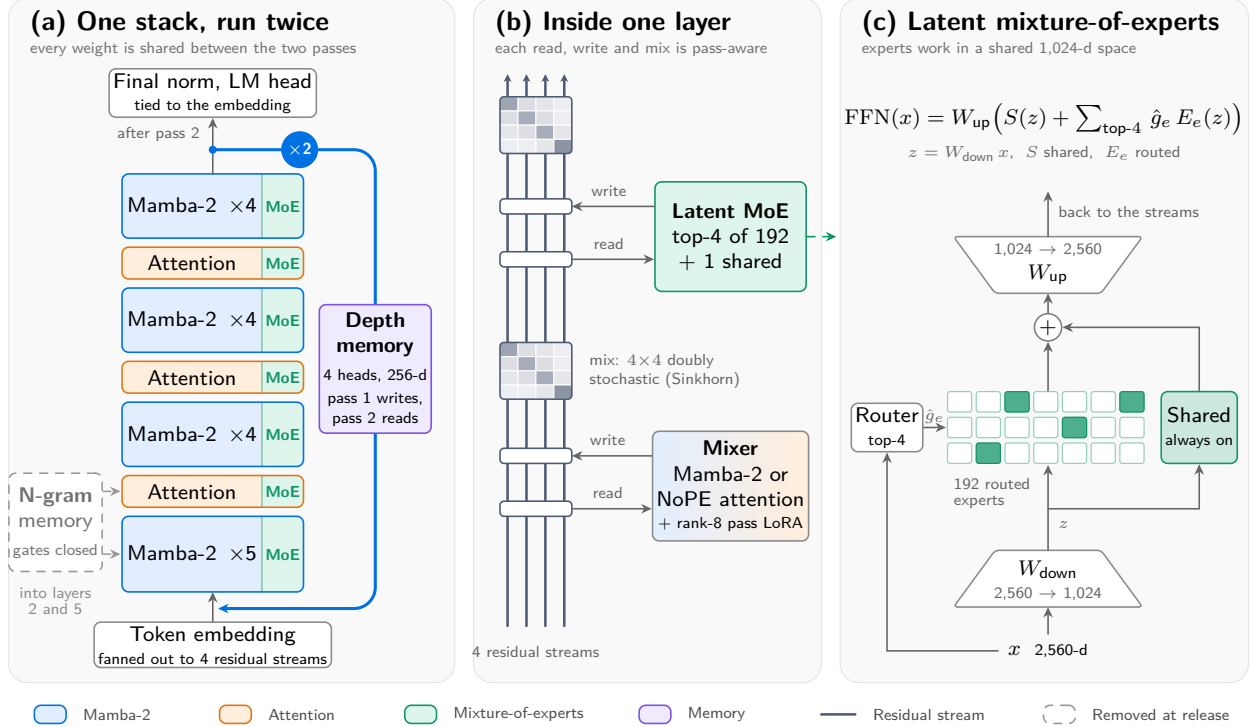

\paragraph{Hybrid trunk.} Logos has 20 layers of width 2,560: grouped-query attention~\cite{ainslie2023gqa} at layers 5, 10 and 15, and Mamba-2 everywhere else. Praxis has 12 layers of width 2,048, with attention at layers 4 and 8. The Mamba-2 blocks use expansion 2 ($d_{\text{inner}} = 5{,}120$ in Logos), state size 128, head dimension 64, a width-4 causal convolution, chunk size 256, and 10 (Logos) or 8 (Praxis) B/C groups. Attention has head dimension 64, with 40 query and 10 key-value heads in Logos and 32 and 8 in Praxis, and uses no positional encoding at all. The Mamba-2 layers already know where each token sits, which is what lets the models stretch from 4,096 to 131,072 tokens without touching position handling (Section~\ref{sec:cpt}). It also means only 3 of Logos's 20 layers (2 of Praxis's 12) keep a key-value cache, so memory grows slowly with context. To let attention behave differently on the second pass without storing a second copy, each attention layer adds a pass-indexed rank-8 low-rank update to its fused query-key-value projection.

\paragraph{Latent mixture-of-experts.} Every layer ends in a mixture-of-experts block whose experts work in a shared 1,024-dimensional latent space. With $z = W_{\text{down}}x$,
\begin{equation}
  \operatorname{FFN}(x) = W_{\text{up}}\Big(S(z) + \sum_{e \in \operatorname{TopK}(r(x),\,4)} \hat g_e(x)\, E_e(z)\Big),
\end{equation}
where $S$ is one always-active shared SwiGLU expert and $E_e$ are routed SwiGLU experts: 192 of width 768 in Logos and 128 of width 512 in Praxis. The weights $\hat g_e$ are the router probabilities renormalized over the four selected experts. Load is balanced by auxiliary-loss-free bias updates (rate $10^{-3}$) with a router $z$-loss (coefficient $10^{-3}$). This is where the efficiency comes from: Logos stores 10.23B parameters but touches only 2.36B of them per token, across both passes.

\paragraph{Residual streams.} The model keeps four residual streams of model width. At every sub-layer boundary, a controller conditioned on a 64-dimensional pass embedding reads from, writes to and mixes the streams. The $4\times4$ mixing matrix is projected toward the doubly stochastic matrices with 8 Sinkhorn--Knopp iterations~\cite{xie2025mhc}.

\paragraph{Loop and depth memory.} The whole stack runs twice per token, sharing every weight except the attention low-rank updates and the pass embeddings. A cross-pass depth memory (4 heads, 256 dimensions, output gate initialized at $-6$) lets the second pass consult notes written during the first. Every token, in training and inference alike, takes exactly two passes and four routed experts per layer.\footnote{The checkpoints also contain an expert-count router and a halting head from an adaptive-compute variant of the design. Neither affects the released models, because training and inference both fix two passes and four experts.}

\paragraph{Vocabulary and context.} Both models share a 65,536-token byte-level BPE vocabulary with tied input and output embeddings. Pretraining used 4,096-token sequences, context extension used sequences of up to 163,840 tokens, and the released models support 131,072 tokens.

\begin{table}[tbp]
\centering
\caption{Configuration and exact parameter counts of the released (n-gram-free) checkpoints. Active parameters count four routed experts per layer and exclude the embedding table and the depth memory, as on the model cards.}
\label{tab:arch}
\small
\begin{tabular}{@{}lrr@{}}
\toprule
 & \textbf{Logos} & \textbf{Praxis} \\
\midrule
Width / layers per pass & 2,560 / 20 & 2,048 / 12 \\
Mamba-2 / attention layers & 17 / 3 (at 5, 10, 15) & 10 / 2 (at 4, 8) \\
Attention heads (query / key-value) & 40 / 10 & 32 / 8 \\
Latent width; routed experts (width); active & 1,024; 192 (768); 4 & 1,024; 128 (512); 4 \\
Residual streams; passes per token & 4; 2 & 4; 2 \\
\midrule
Token embedding (tied) & 167,772,160 & 134,217,728 \\
Sequence mixers (Mamba-2, attention, pass LoRA) & 833,409,535 & 316,315,530 \\
Latent down and up projections & 104,857,600 & 50,331,648 \\
Routed experts & 9,059,696,640 & 2,415,919,104 \\
Shared experts & 47,185,920 & 18,874,368 \\
Routers & 6,110,900 & 2,485,356 \\
Residual-stream controllers & 176,320 & 95,680 \\
Depth memory and unused halting head & 6,244,102 & 5,061,638 \\
Final norm & 2,560 & 2,048 \\
\textbf{Total} & \textbf{10,225,455,737} & \textbf{2,943,303,100} \\
\midrule
N-gram memory, removed at release (Section~\ref{sec:ngram}) & 1,774,549,800 & 602,180,520 \\
\midrule
\textbf{Active per pass} & \textbf{1,180,486,515} & \textbf{463,602,102} \\
Active per token (two passes) & 2,360,973,030 & 927,204,204 \\
\bottomrule
\end{tabular}
\end{table}

%=====================================================================
\section{Pretraining and context extension}

\subsection{Data}
Both models were pretrained on the sealed \code{metis-1.6-data-r2} release: 806 provenance-hashed shards totalling 804,755,808,835 tokens (SHA-256 \code{65ef27f6...}), served in three curriculum phases:
\begin{enumerate}
  \item \textbf{Phase A} (554,849,376,946 tokens, 68.9\%): broad filtered web text, books, scientific articles and code.
  \item \textbf{Phase B} (209,668,641,449 tokens, 26.1\%): up-weighted mathematics, science, software engineering and structured analytical text.
  \item \textbf{Phase C} (40,237,790,440 tokens, 5.0\%): an annealing mix of curated technical prose, formal proofs, competition mathematics and reference documentation.
\end{enumerate}
Logos took 219,280 optimizer steps and ended at checkpoint \code{ebddf9827d1a}; Praxis took 191,870. Every shard was decontaminated against the evaluation benchmarks by 13-gram shingle matching. Each shard also carries a sidecar (\code{NGRAM_CANONICAL_IDS.uint16}) that maps every token to a canonical id, collapsing case, whitespace and punctuation variants, for the n-gram memory of Section~\ref{sec:ngram}.

\subsection{The optimizer I thought I had}
\label{sec:optim}
The run configuration said \code{optimizer.name: adamuon}. I meant to carry over the AdaMuon implementation of Metis-1.5~\cite{si2025adamuon}, which keeps an elementwise second-moment buffer on the orthogonalized direction. Partway through the campaign I audited the PyTorch trainer and the 100B- and 600B-token Logos checkpoints, and found that the optimizer that actually ran was a different animal in four ways, all present for the whole of pretraining:
\begin{enumerate}
  \item \textbf{The optimizer name was never read.} The trainer always built a Muon/AdamW hybrid with plain Muon and no second-moment state: Nesterov momentum 0.95; five Newton--Schulz steps with coefficients $(3.4445, -4.7750, 2.0315)$ in BF16; update RMS matched by a $0.2\sqrt{\max(r,c)}$ scale; constant decoupled weight decay 0.1; and an unswept learning rate of $1.2\times10^{-4}$ (Logos) or $1.8\times10^{-4}$ (Praxis).
  \item \textbf{Attention trained with AdamW.} The parameter router assigned Muon to attention matrices by matching names containing \code{.mixer.impl.}, but the attention module was mounted directly at \code{layers.N.mixer}. Every attention projection, including the pass-specific low-rank updates, quietly fell through to the default group and trained with AdamW ($\beta = (0.9, 0.95)$, $\epsilon = 10^{-8}$, weight decay 0.1) for all 804.76B tokens.
  \item \textbf{Fused projections were orthogonalized as one matrix.} The Mamba-2 input projection ($12{,}880 \times 2{,}560$ in Logos) concatenates five functionally different projections ($z$, $x$, $B$, $C$ and $\Delta t$). Muon orthogonalized it as a single block, coupling singular vectors across gate, state and step-size projections.
  \item \textbf{Weight norms drifted apart.} Between 100B and 600B tokens, weight RMS changed by a factor of 0.88 to 0.91 for routed experts, 1.12 for Mamba input projections, 1.24 for Mamba output projections, 0.80 for shared-expert down projections, 0.57 for convolutions and 1.49 for the final norm.
\end{enumerate}
None of this stopped training, and the loss curve never said a word. Every architectural conclusion in this report is therefore conditional on this optimizer. The lesson is cheap to apply: assert every parameter's optimizer group at start-up. Section~\ref{sec:ablations} measures what the corrected optimizer is worth.

\subsection{Context extension}
\label{sec:cpt}
From the 4,096-token checkpoints, both models continued pretraining on 18B tokens of long documents packed to at most 163,840 tokens. Logos took 1,968 optimizer steps (17,937,499,304 tokens) and Praxis 3,456 steps (18,000,000,000 tokens). Table~\ref{tab:cpt} reports the extension gate on 384 held-out documents of 131,072 tokens, short-context retention on seven zero-shot tasks, and the 13-task RULER suite~\cite{hsieh2024ruler}.

On Logos, top-1 retrieval of a planted needle at 131,072 tokens rose from 19.53\% to 68.75\% at the 12B-token gate and 66.93\% at 18B, while loss on the tail of full-length documents fell from 5.689 to 1.495 nats. RULER reached 84.48\% at 4K, 73.25\% at 32K and 60.14\% at 128K, 1.32 points above the 12B-token gate at 128K, and the seven-task short-context mean barely moved (62.65 to 62.24). On Praxis, needle accuracy rose from 17.97\% to 67.71\% at the 6B-token gate, 66.15\% at 12B and 65.63\% at 18B; full-context tail loss fell from 5.735 to 1.614 nats and base-window loss from 1.703 to 1.665 nats. RULER reached 77.01\% at 4K, 58.42\% at 32K and 39.01\% at 128K. Appendix~\ref{app:base} gives the complete 50-benchmark evaluation of both extended base models.

\begin{table}[tbp]
\centering
\caption{Context extension. Gate metrics are measured on 384 held-out documents at 131,072 tokens (NLL in nats per token; lower is better). The seven tasks are zero-shot. RULER uses 13 tasks with 1,300, 650 and 325 samples at 4K, 32K and 128K. -- means beyond the training window.}
\label{tab:cpt}
\small
\setlength{\tabcolsep}{4pt}
\begin{tabular}{@{}lrrrrr@{}}
\toprule
 & \multicolumn{3}{c}{\textbf{Logos}} & \multicolumn{2}{c}{\textbf{Praxis}} \\
\cmidrule(lr){2-4}\cmidrule(l){5-6}
 & \textbf{4K base} & \textbf{12B ext.} & \textbf{18B ext.} & \textbf{4K base} & \textbf{18B ext.} \\
\midrule
Needle top-1 accuracy (\%) & 19.53 & \textbf{68.75} & 66.93 & 17.97 & 65.63 \\
Needle log-probability & $-5.839$ & $\mathbf{-1.671}$ & $-1.743$ & $-5.918$ & $-1.871$ \\
Full-context tail NLL & 5.689 & 1.508 & \textbf{1.495} & 5.735 & 1.614 \\
Short-context tail NLL & 1.711 & 1.709 & \textbf{1.695} & 1.860 & 1.833 \\
Base-window NLL & 1.550 & 1.538 & \textbf{1.525} & 1.703 & 1.665 \\
\midrule
MMLU ($N = 14{,}042$) & \textbf{54.01} & 52.69 & 52.53 & 50.85 & 38.49 \\
HellaSwag ($N = 10{,}042$) & \textbf{75.19} & 74.30 & 74.69 & 69.92 & 69.31 \\
PIQA ($N = 1{,}838$) & 79.49 & \textbf{80.03} & 79.43 & 77.69 & 77.31 \\
ARC-Easy ($N = 2{,}376$) & 74.41 & 74.45 & \textbf{74.66} & 71.93 & 69.65 \\
ARC-Challenge ($N = 1{,}172$) & \textbf{47.27} & 45.31 & 45.22 & 44.11 & 38.65 \\
WinoGrande ($N = 1{,}267$) & 65.98 & 65.43 & \textbf{66.54} & 65.59 & 61.80 \\
OpenBookQA ($N = 500$) & 42.20 & \textbf{43.20} & 42.60 & 42.20 & 39.60 \\
\textbf{Seven-task mean} & \textbf{62.65} & 62.20 & 62.24 & 60.33 & 56.40 \\
\midrule
RULER 4K & 81.92 & 83.61 & \textbf{84.48} & 74.24 & 77.01 \\
RULER 32K & -- & 71.48 & \textbf{73.25} & -- & 58.42 \\
RULER 128K & -- & 58.82 & \textbf{60.14} & -- & 39.01 \\
\bottomrule
\end{tabular}
\end{table}

\begin{table}[tbp]
\centering
\caption{Training lineage of Metis-1.6 Logos. The end cursor counts every token processed in the campaign, so it keeps rising across rows 5 to 10, which all branch from row 4. Loss and gradient norm are the last logged values; hashes are the first 12 hex digits of each checkpoint's SHA-256. The released model is row 3. The step-40 (\code{389c7007}) and step-50 (\code{9c708434}) checkpoints of row 5 were also evaluated (Table~\ref{tab:rl_suite}).}
\label{tab:logos_lineage}
\scriptsize
\setlength{\tabcolsep}{3pt}
\resizebox{\textwidth}{!}{%
\begin{tabular}{@{}rlrrrrrll@{}}
\toprule
\textbf{\#} & \textbf{Stage} & \textbf{End cursor} & \textbf{Stage tokens} & \textbf{Steps} & \textbf{Loss} & \textbf{Grad} & \textbf{Parent} & \textbf{Output} \\
\midrule
0 & \code{pretraining} & 804,755,808,835 & 804,755,808,835 & 219,280 & -- & -- & -- & \code{ebddf9827d1a} \\
1 & \code{context_extension} & 822,693,308,139 & 17,937,499,304 & 1,968 & 1.4984 & 0.4507 & \code{ebddf9827d1a} & \code{216ab39b53b8} \\
2 & \code{cold_start_sft} & 942,701,111,551 & 120,007,803,412 & 8,182 & 0.2868 & 0.1838 & \code{216ab39b53b8} & \code{6958ef5907eb} \\
3 & \code{overall_sft} (released) & 1,022,704,091,713 & 80,002,980,162 & 2,735 & 0.4095 & 0.2011 & \code{6958ef5907eb} & \code{44b8e759242f} \\
4 & \code{hybrid_mode_gspo} & 1,025,955,849,911 & 3,251,758,198 & 264 & 0.0009 & 0.9774 & \code{44b8e759242f} & \code{ea50ea259f48} \\
5 & \code{specialist_reasoning} & 1,026,828,195,987 & 872,346,076 & 78 & 0.0007 & 0.3240 & \code{ea50ea259f48} & \code{c640b1064747} \\
6 & \code{specialist_code} & 1,026,836,703,500 & 8,507,513 & 2 & $-0.0326$ & 1.2911 & \code{ea50ea259f48} & \code{92320a77d22c} \\
7 & \code{specialist_knowledge} & 1,026,842,468,788 & 5,765,288 & 18 & $-0.0052$ & 11.8811 & \code{ea50ea259f48} & \code{4a7c689dc0cf} \\
8 & \code{specialist_writing} & 1,026,842,468,788 & 0 & 0 & -- & -- & \code{ea50ea259f48} & \code{ea50ea259f48} \\
9 & \code{specialist_agentic} & 1,026,845,529,737 & 3,060,949 & 20 & $-0.1760$ & 3.7435 & \code{ea50ea259f48} & \code{0f3d8610306b} \\
10 & \code{opd_consolidation} & 1,026,846,087,232 & 557,495 & 1 & 0.0011 & 1.5650 & \code{ea50ea259f48} & \code{75c4344324a9} \\
\bottomrule
\end{tabular}}
\end{table}

%=====================================================================
\section{A memory the models chose to forget}
\label{sec:ngram}

Following the Engram line of work~\cite{engram2026}, both models carried a hashed n-gram conditional memory. For each position, the canonicalized preceding 2-grams and 3-grams are hashed into eight prime-modulus tables per order. In Logos the 2-gram tables have 1,723,303 to 1,730,101 rows and the 3-gram tables 1,730,507 to 1,737,161, each holding 64-dimensional values. The 16 retrieved vectors are concatenated, projected to model width, and the result $m_t$ is added to each residual stream $s$ at two layers of each pass (layers 2 and 5 in Logos, 2 and 4 in Praxis) through a learned gate:
\begin{equation}
  g^{(\ell,p)}_{t,s} = \sigma\!\left(\frac{\langle h^{(\ell,p)}_{t,s},\, v^{(\ell,p)}_{s}\rangle}{\sqrt{d_{\text{model}}}} + b^{(\ell,p)}_{s}\right), \qquad
  \tilde h^{(\ell,p)}_{t,s} = h^{(\ell,p)}_{t,s} + g^{(\ell,p)}_{t,s}\, m_{t}.
\end{equation}
The tables hold 1,774,549,800 parameters in Logos and 602,180,520 in Praxis, and were trained with sparse Adam at twice the base learning rate.

The models had other plans. Both drove the gate biases negative and closed the gates (Table~\ref{tab:ngram}). After pretraining, the mean gate at every injection site was between 0.0001 and 0.0031. After post-training it was at most 0.0141, and the memory's update to a residual stream stayed below 0.85\% of the stream's norm everywhere. I see two structural reasons:
\begin{enumerate}
  \item \textbf{Redundancy with early Mamba-2 convolutions.} Layers 0 and 1 are Mamba-2 blocks whose width-4 causal convolutions already compute local 2-, 3- and 4-gram compositions in continuous hidden space before the first injection at layer 2.
  \item \textbf{Hash-collision noise.} With about 1.73M rows per table across 805B tokens, prime-modulus hashing aliases frequent and rare n-grams into the same rows. Injected into four residual streams that are refined over two passes, the noisy lookup costs more than it adds, and the gate bias receives a steady gradient toward closing.
\end{enumerate}

\begin{table}[tbp]
\centering
\caption{Mean n-gram gate and relative residual update $\lVert\Delta\rVert/\lVert h\rVert$ at each injection site (pass 1 / pass 2).}
\label{tab:ngram}
\small
\begin{tabular}{@{}llrr@{}}
\toprule
\textbf{Checkpoint} & \textbf{Site} & \textbf{Mean gate} & \textbf{Update} \\
\midrule
Logos base (18B extension) & layer 2 & 0.000400 / 0.003100 & 0.26\% / 0.21\% \\
 & layer 5 & 0.000100 / 0.002100 & 0.12\% / 0.15\% \\
Logos released (\code{overall_sft}) & layer 2 & 0.000417 / 0.005928 & 0.17\% / 0.83\% \\
 & layer 5 & 0.000131 / 0.002544 & 0.04\% / 0.34\% \\
Logos after RL (\code{opd_consolidation}) & layer 2 & 0.000416 / 0.005971 & 0.17\% / 0.84\% \\
 & layer 5 & 0.000131 / 0.002478 & 0.04\% / 0.33\% \\
Praxis released (\code{overall_sft}) & layer 2 & 0.000643 / 0.001654 & 0.21\% / 0.08\% \\
 & layer 4 & 0.000411 / 0.014051 & 0.11\% / 0.74\% \\
\bottomrule
\end{tabular}
\end{table}

So removing the memory is close to free. Deleting the 20 \code{ngram_memory.*} tensors moved bits per byte on pretraining prose from 0.733898 to 0.733894 on Logos and from 0.78110 to 0.78123 on Praxis. On 44,684 tokens of post-training text it changed the loss of released Logos by $-0.000647$ nats (1.184412 to 1.183764, 97.65\% top-1 agreement), of Logos after reinforcement learning by $+0.000112$ nats (1.404278 to 1.404390, 96.70\%), and of released Praxis by $-0.001559$ nats (1.281273 to 1.279714, 97.72\%). For released Logos, the model is very slightly better without it. The released checkpoints ship without the memory: in BF16, Logos is 20.80\,GB instead of 24.35\,GB (\code{model.safetensors} SHA-256 \code{965efc2f...}), and Praxis 6.16\,GB instead of 7.37\,GB (SHA-256 \code{861af5db...}).

I have not given up on the idea, only on this version of it. For Metis-1.7 the table grows to 50B parameters, stored off-accelerator in a per-row ternary format (BitNet b1.58, 1.85 bits per value) and trained with quantization-aware straight-through estimation, so that it can hold long-tail facts instead of competing with the network for accelerator memory.

%=====================================================================
\section{Supervised fine-tuning, and five things I would do differently}
\label{sec:sft}

Both models were fine-tuned with a response-only cross-entropy loss under the \code{metis-chat/v1} template, which has two modes: in \code{think} mode the response opens with a reasoning block, and in \code{direct} mode it answers straight away.

\paragraph{Stage 1: cold-start fine-tuning} used 120.01B tokens of long reasoning traces in mathematics, science and code. On Logos it took 8,182 steps and ended at \code{6958ef5907eb}, with gates at 60B and 120B tokens. The 120B gate reached a held-out response loss of 0.4612 (final training loss 0.2868) and, in think mode, 59.58\% on pooled AIME, 50.82\% on GPQA Diamond, 42.61\% on HMMT February 2026 and 62.86\% on MMLU. It was a capable reasoner and a poor assistant: 32.76\% on IFEval and 12.40\% on BFCL v4.

\paragraph{Stage 2: overall fine-tuning} used 80.00B tokens across every domain, including instruction following, tool calling and conversation. On Logos it took 2,735 steps and ended at \code{44b8e759242f}, with gates at 40B (\code{4538f48d884d}) and 80B tokens. The 80B checkpoint kept the reasoning (60.00\% on AIME, 52.02\% on GPQA Diamond, 68.00\% on MMLU and 56.14\% on MMLU-Pro in think mode) and became an assistant: 71.07\% on IFEval and 58.60\% on BFCL v4. The reinforcement-learning stages of Section~\ref{sec:rl} traded breadth for shorter reasoning, so for both models I released the end of Stage 2.

\subsection{Five failure modes and their fixes in Metis-1.7}
\begin{enumerate}
  \item \textbf{Direct-mode examples made by deletion.} To build direct-mode examples, a script (\code{strip.py}) deleted the reasoning block from 15--25\% of reasoning traces. Those traces define variables, helper functions and intermediate equations inside the reasoning block, so the remaining answer often refers to work the model never sees. That teaches direct mode to guess. At the cold-start gate, direct-mode IFEval trailed think mode by 10.92 points (21.84\% against 32.76\%). \emph{Fix:} the deletion script is retired in favour of 12.0B tokens of direct-mode answers generated natively without reasoning and verified against ground-truth answers and unit tests.
  \item \textbf{Verification loops inherited from teachers.} At the released Logos checkpoint, median think-mode lengths on the gate suite were 23,400 tokens on AIME (21.67\% truncated at 64K), 42,377 on HMMT February 2026 (36.36\% truncated) and 13,857 on GPQA Diamond. QwQ-32B and DeepSeek-R1-0528 traces in the fine-tuning pool recount or double-check four to seven times after the boxed answer is already derived, and Logos learned the habit faithfully. \emph{Fix:} a filter rejects traces with more than two post-solution verification passes, or more than 4,096 thinking tokens on prompts of difficulty at most 2/5, while keeping genuine 16K to 64K-token derivations in software engineering and olympiad mathematics.
  \item \textbf{A hidden teacher mix.} Of the 9.34M rows of \code{openbmb/UltraData-SFT-2605}, 6.93M (74.2\%) came from early-2025 teachers: DeepSeek-R1-0528 (3.89M), Qwen2.5-72B (2.09M) and QwQ-32B (740K). \emph{Fix:} all 6.93M rows (about 18B tokens) are removed, and only late-2026 teachers are kept (Qwen3.8-Flash-Next, DeepSeek-V4.1-Flash-1017 and GLM-5.3-Flash).
  \item \textbf{No voice of its own.} Only 4,800 identity rows (0.008\% of 60B tokens) and 15M tokens of open conversation (0.04\% of 80B) entered training, and in open-ended chat Logos sounds like the average of its teachers. \emph{Fix:} 2.5B tokens of writing in a defined house voice (no trailing hook questions, sparing punctuation, no therapeutic clich\'es) and 600M epoch-weighted tokens of multi-turn identity data.
  \item \textbf{Off-policy cross-entropy in Stage 2.} Token-level cross-entropy on teacher text applies unbounded $-\log \pi_\theta(y^*_t\mid x,y^*_{<t})$ gradients to low-probability stylistic tokens, which collapses generation entropy before reinforcement learning even starts. \emph{Fix:} Stage 2 moves to a reference-anchored sampling objective,
  \begin{equation}
    \mathcal{L}(x,y^*) = -\frac{1}{|y^*|}\sum_{t=1}^{|y^*|} \operatorname{clip}\!\left(\left(\frac{\pi_\theta(y^*_t\mid x,y^*_{<t})}{\pi_{\text{ref}}(y^*_t\mid x,y^*_{<t})}\right)^{0.25}\!,\, 0.2,\, 2.5\right)\log\pi_\theta(y^*_t\mid x,y^*_{<t}),
  \end{equation}
  mixed with 10\% standard cross-entropy. The reference log-probabilities are precomputed offline in FP16, so the reference model costs no accelerator memory during training.
\end{enumerate}

%=====================================================================
\section{Reinforcement learning on Logos}
\label{sec:rl}

Starting from the released Logos checkpoint, I ran a six-stage reinforcement-learning and distillation study: hybrid-mode GSPO~\cite{zheng2025gspo}, five domain specialists, and on-policy distillation of the specialists back into one model. None of it became the released model. All of it taught me something. Table~\ref{tab:rl_suite} gives the eight-suite gate for every checkpoint.

\subsection{Teaching brevity with the wrong curriculum}
\code{hybrid_mode_gspo} ran for 264 optimizer steps over 79 rounds and ended at \code{ea50ea259f48}. It sampled 16 responses per prompt and alternated two short rounds at a short response cap with one long round at the full cap.

\paragraph{Systems.} Three changes kept rollout generation from eating the run:
\begin{enumerate}
  \item \textbf{Fused decode kernels.} I fused routing statistics and ragged expert dispatch, injected depth memory without copies and precomputed dispatch constants, taking each decode step from about 8,000 small kernel launches to one compact captured graph.
  \item \textbf{Two-cohort cap ratcheting.} The response cap tightened by 4,096 tokens whenever two consecutive closed cohorts had less than 1\% truncation among correct samples and less than 2\% among hard correct ones. It stepped from 49,152 to 20,480 tokens (Appendix~\ref{app:telemetry}) and was inherited by every downstream specialist.
  \item \textbf{Memory.} The rollout cache is released before the packed backward pass, and a device-memory guard is anchored at 100.01\,GiB (26\,GB reserve). Peak reserved memory stayed at 74.60\,GiB in GSPO and 60.57 to 66.38\,GiB in the specialists, with no out-of-memory failures.
\end{enumerate}

\paragraph{Verifier imbalance.} At step 198 (\code{75e8704e}) the gate suite told two stories at once. Think-mode IFEval had jumped 13.41 points, from 71.07\% to 84.48\% at a median of 243 tokens, and direct-mode GPQA Diamond had risen from 32.83\% to 35.35\%. Meanwhile think-mode GPQA Diamond had fallen from 52.02\% to 43.43\% (median 2,687 tokens, 5.2 times shorter), and pooled AIME from 60.00\% to 40.00\% (3,713 tokens, 6.3 times shorter). The prompt pool explains it:
\begin{itemize}
  \item 75.3\% of prompts were formatting and tool-calling tasks: 48.3\% function calls, 19.1\% instruction constraints and 7.8\% JSON schemas;
  \item 24.7\% were mathematics with numeric answers;
  \item none were science, because all 15,235 science problems of \code{PrimeIntellect/INTELLECT-3-RL} had been set aside for the reasoning specialist.
\end{itemize}
Three quarters of the reward came from tasks that a 150- to 400-token checklist solves. Rewarding brevity there taught the policy to answer graduate-level science with a checklist too. \emph{Fix:} in Metis-1.7 this stage becomes a 48- to 64-step domain-balanced bridge (30\% mathematics, 20\% science and STEM multiple choice, 15\% code with unit tests, 15\% instruction following, 10\% tool calls, 10\% knowledge questions), so science verifiers guard reasoning depth while length falls.

\subsection{Punished for being right}
\label{sec:signinv}
The specialists used a gated length reward with $\lambda = 0.5$. Within a prompt group that passes the gate, a correct sample receives reward 1 if its length is at most the median length of the group's correct samples, and $1-\lambda$ if longer; an incorrect sample receives 0. Advantages are group-mean centred, $A_i = (r_i - \bar r)/\sigma$, with $\sigma > 0$ a batch-level scale. That sounds harmless. It is not.

\begin{proposition}
\label{prop:signinv}
Consider a group of $G$ samples of which a fraction $p$ is correct. Let a fraction $q$ of the correct samples be short (reward 1) and $1-q$ long (reward $1-\lambda$, $0 < \lambda < 1$). Then every long correct sample has negative advantage if and only if
\[
  p > \frac{1-\lambda}{1 - \lambda(1-q)}.
\]
With the median split ($q = 1/2$) and $\lambda = 1/2$, the threshold is $p > 2/3$, that is, at least 11 correct samples out of 16.
\end{proposition}

\begin{proof}
The group mean is $\bar r = p\,[\,q + (1-q)(1-\lambda)\,] = p\,[1 - \lambda(1-q)]$. A long correct sample has $A = (1 - \lambda - \bar r)/\sigma$, which is negative exactly when $1-\lambda < p\,[1-\lambda(1-q)]$.
\end{proof}

In words: once a prompt becomes easy, the policy gradient lowers the probability of every correct solution longer than the group median, however valid its reasoning. The reward punishes the model for showing its work. The reasoning specialist, trained on mathematics and science verifiers (\code{INTELLECT-3-RL} mathematics and science, \code{UltraData-RL-2609} mathematics), followed the script almost exactly:
\begin{itemize}
  \item \textbf{Step 40} (\code{389c7007}). The first 40 steps won back what late GSPO had lost. Think-mode GPQA Diamond rose 8.08 points, from 35.86\% to 43.94\%, at a median of 1,122 tokens, 12.3 times shorter than the released model. The rest of GPQA Main reached 38.40\%, 2.80 points above the released model at 16.2 times fewer tokens. MMLU-Pro reached 57.29\% at 843 tokens, against the released model's 56.14\% at 5,590. HMMT February 2026 tripled, from 12.12\% to 36.36\%.
  \item \textbf{Step 50} (\code{9c708434}). HMMT February 2026 rose to 42.42\%, 3.03 points above the released model at a median of 7,964 tokens against 42,377. Pooled AIME reached 35.00\%, direct-mode IFEval 84.29\% and direct-mode MMLU-Pro 47.29\%.
  \item \textbf{Step 78} (\code{c640b106}). The mean group pass rate had crossed 56.7\%, and high-pass groups now exceeded $p > 2/3$. With the length budget at 1,240 tokens, the inversion pushed down every correct derivation longer than that. GPQA Diamond fell back to 34.34\% (median 676 tokens) and HMMT February 2026 to 24.24\%.
\end{itemize}
Step 40 is the tell. Short reasoning was never the problem: the decline came from negative advantages on valid 1,200- to 3,000-token derivations, not from the response cap.

The fix is one line. Clamp the advantage of correct samples at zero,
\begin{equation}
  A_i \leftarrow \max(0, A_i) \quad \text{for every correct sample } i,
\end{equation}
which keeps the positive signal for short correct samples, removes the negative one for long correct samples, and still pushes down incorrect ones. The code, knowledge and agentic specialists trained with the clamp. The knowledge specialist raised HMMT February 2026 by 21.21 points over the GSPO checkpoint, to 33.33\%, and reached 84.10\% on direct-mode IFEval. The reasoning specialist had already finished before the clamp existed and was not rerun with it.

\subsection{Agents waiting on containers}
The first rounds of the agentic specialist took more than 15 minutes each, for three reasons:
\begin{enumerate}
  \item sandbox containers for all 64 episodes were started one by one before rollout began, although 95\% of prompts were single-call tool-schema checks that never needed a sandbox at all;
  \item episodes cancelled when a round reached its target never released their filesystem mounts;
  \item records cancelled in short rounds were deferred to later rounds, which forced full, uncancelled long rounds every third round.
\end{enumerate}
Starting a sandbox only at an episode's first tool call, closing every episode environment in a \code{finally} block, and retiring cancelled records instead of deferring them fixed all three. Round times fell to those in Table~\ref{tab:agentic_telemetry}, the deferral queue stayed empty, and the cap ratcheted four times, from 20,480 to 4,096 tokens, with no truncation. Among post-GSPO checkpoints, the agentic specialist set the highest scores on BFCL v4 in think mode (52.80\%) and direct mode (47.20\%), IFEval in think mode (84.10\%), GPQA Diamond in direct mode (38.89\%) and MMLU in think mode (67.14\%).

\subsection{A distillation step too small to matter}
The final stage (\code{opd_consolidation}, \code{75c43443}) distilled the GSPO checkpoint toward the four active specialists, with reverse KL over the union of each teacher's top 32 tokens, stop-token probabilities merged and routers frozen (route entropy ratio 1.0000, loss 0.001137). It did help: direct-mode AIME from 0.00\% to 3.33\%, HMMT February 2026 from 12.12\% to 21.21\%, think-mode GPQA Diamond from 35.86\% to 36.87\%, and think-mode IFEval from 80.84\% to 83.14\%. It did not capture the peaks of the step-40 reasoning or agentic specialists, and for three unglamorous reasons:
\begin{enumerate}
  \item \textbf{One step, at warm-up speed.} 224 trajectories on 224 data-parallel ranks with no gradient accumulation made exactly one optimizer step. With 10 warm-up steps, that step ran at $10^{-7}$, a tenth of the target rate, and moved the weights about 100 times less than a real run.
  \item \textbf{The best teacher had been deleted.} Automatic checkpoint pruning had removed the reasoning specialist's step-40 and step-50 checkpoints, so the weaker step-78 checkpoint served as teacher.
  \item \textbf{Truncated trajectories.} Sequences were cut at 8,192 tokens, removing the tails and stop tokens of longer trajectories.
\end{enumerate}
\emph{Fix:} periodic specialist checkpoints are pinned until the gate suite picks each specialist's best, and distillation runs for 40 to 80 optimizer steps (9,000 to 18,000 untruncated student rollouts, two updates per round) at the full learning rate of $10^{-6}$ past warm-up.

\begin{table}[tbp]
\centering
\caption{Eight-suite gate of the Logos post-training study: pass@1 (\%), with median generated tokens in parentheses. The released model is the first column. Bold marks the best entry in each row.}
\label{tab:rl_suite}
\resizebox{\textwidth}{!}{%
\begin{tabular}{@{}lrrrrrrrrr@{}}
\toprule
\textbf{Suite : mode} & \textbf{Overall SFT} & \textbf{Cold start} & \textbf{GSPO 198} & \textbf{GSPO 264} & \textbf{Reason. 40} & \textbf{Reason. 50} & \textbf{Knowledge} & \textbf{Agentic} & \textbf{Distilled} \\
\midrule
AIME : think & 60.00 (23400) & 59.58 (23965) & 40.00 (3713) & 33.33 (2224) & 31.67 (4073) & \textbf{35.00} (3541) & 25.00 (2809) & 33.33 (2012) & 28.33 (2848) \\
AIME : direct & 3.33 (707) & 3.02 (720) & 3.33 (419) & 0.00 (338) & \textbf{6.67} (381) & \textbf{6.67} (341) & 5.00 (259) & 1.67 (240) & 3.33 (339) \\
HMMT Feb 26 : think & 39.39 (42377) & 42.61 (51033) & 36.36 (12967) & 12.12 (7864) & 36.36 (9862) & \textbf{42.42} (7964) & 33.33 (4168) & 18.18 (6982) & 21.21 (4722) \\
GPQA Diamond : think & 52.02 (13857) & 50.82 (16205) & 43.43 (2687) & 35.86 (978) & \textbf{43.94} (1122) & 41.41 (1002) & 35.35 (870) & 33.84 (1072) & 36.87 (951) \\
GPQA Diamond : direct & 32.83 (747) & 32.13 (720) & 35.35 (291) & 31.82 (145) & 31.31 (164) & 32.32 (173) & 32.83 (152) & \textbf{38.89} (150) & 26.77 (152) \\
GPQA Main rest : think & 35.60 (12521) & 44.80 (15032) & 37.60 (1556) & 34.00 (539) & \textbf{38.40} (774) & 32.80 (723) & 37.60 (615) & 38.00 (590) & 34.80 (591) \\
GPQA Main rest : direct & 32.00 (698) & 24.80 (691) & 30.00 (226) & 30.80 (119) & \textbf{35.60} (147) & 34.80 (143) & 28.40 (133) & 28.00 (125) & 32.00 (133) \\
MMLU-Pro : think & 56.14 (5590) & 54.57 (6855) & 54.29 (1192) & 53.43 (693) & \textbf{57.29} (843) & 57.00 (799) & 53.29 (748) & 53.57 (737) & 53.00 (732) \\
MMLU-Pro : direct & 44.00 (562) & 38.43 (482) & 44.71 (173) & 41.57 (101) & 46.43 (115) & \textbf{47.29} (110) & 42.43 (100) & 43.29 (93) & 42.29 (102) \\
MMLU : think & 68.00 (1465) & 62.86 (1983) & 66.14 (255) & 65.43 (163) & 67.00 (182) & 67.00 (157) & 65.86 (150) & \textbf{67.14} (160) & 65.43 (157) \\
MMLU : direct & 61.57 (279) & 52.14 (219) & 61.00 (65) & 61.57 (48) & 62.86 (53) & \textbf{63.29} (50) & 61.86 (45) & 60.71 (46) & 60.86 (47) \\
IFEval : think & 71.07 (1197) & 32.76 (1740) & \textbf{84.48} (243) & 80.84 (119) & 81.99 (130) & 83.52 (130) & 81.03 (128) & 84.10 (133) & 83.14 (127) \\
IFEval : direct & 69.92 (277) & 21.84 (263) & 81.61 (86) & 82.57 (60) & 83.52 (63) & \textbf{84.29} (63) & 84.10 (62) & 82.38 (62) & 80.84 (58) \\
BFCL v4 : think & \textbf{58.60} (260) & 12.40 (385) & 53.00 (84) & 51.00 (63) & 49.60 (67) & 51.20 (67) & 52.00 (65) & 52.80 (66) & 50.20 (64) \\
BFCL v4 : direct & \textbf{56.40} (82) & 21.60 (126) & 45.40 (54) & 46.60 (49) & 47.00 (51) & 46.20 (50) & 46.20 (49) & 47.20 (50) & 46.80 (49) \\
\bottomrule
\end{tabular}}
\end{table}

%=====================================================================
\section{Praxis post-training}
\label{sec:praxis}

Overall fine-tuning had been the best general-purpose checkpoint for Logos, so Praxis took the same release path without the reinforcement-learning detour: context extension, cold-start fine-tuning and overall fine-tuning, to 1,022,774,982,313 cumulative tokens (Table~\ref{tab:praxis_lineage}). It was released at the end of Stage 2 (\code{8e7512a491e9}).

\paragraph{Cold-start fine-tuning} (120.00B tokens, 14,318 steps) took Praxis from 0.00\% on contest mathematics to 27.08\% on pooled AIME 2025 and 2026 in think mode (14.69\% at the 60B gate). It also reached 17.05\% on HMMT February 2026, 31.57\% on GPQA Diamond (29.51\% direct) and 46.86\% on MMLU (Table~\ref{tab:praxis_suite}). As with Logos, instruction following and tool calling stayed low before Stage 2 (22.22\% on IFEval and 14.40\% on BFCL v4 in think mode).

\paragraph{Overall fine-tuning} (80.02B tokens, 4,787 steps, final loss 0.1828, held-out response loss 0.6844) turned it into an assistant. IFEval rose by 42.34 points in think mode (to 64.56\%) and 50.39 points in direct mode (to 66.48\%); BFCL v4 by 38.00 points in think mode (to 52.40\%; 55.40\% at the 40B gate) and 23.20 points in direct mode (to 50.40\%). HMMT February 2026 rose to 24.24\%, MMLU-Pro to 34.57\% and MMLU to 53.43\% (think) and 50.43\% (direct).

\paragraph{Release evaluation.} On the post-trained suite (Appendix~\ref{app:post}), the released Praxis scores 27.08\% on AIME 2025, 31.25\% on AIME 2026, 24.17\% on HMMT November 2025, 60.31\% on AMC 2023 and 41.25\% on Brumo 2025, all avg@8. It also scores 76.80\% on GSM8K, 78.80\% on MATH-500, 62.80\% on HumanEval (50.00\% on HumanEval+), 81.60\% on ARC-Easy and 79.20\% on ARC-Challenge. On RULER it scores 75.54\%, 56.23\% and 42.38\% at 4K, 32K and 128K.

\begin{table}[tbp]
\centering
\caption{Training lineage of Metis-1.6 Praxis. Rows 2a and 3a are intermediate gates that share a parent with the final row of their stage. The released model is row 3b.}
\label{tab:praxis_lineage}
\scriptsize
\setlength{\tabcolsep}{3pt}
\resizebox{\textwidth}{!}{%
\begin{tabular}{@{}rlrrrrrll@{}}
\toprule
\textbf{\#} & \textbf{Stage} & \textbf{End cursor} & \textbf{Stage tokens} & \textbf{Steps} & \textbf{Loss} & \textbf{Grad} & \textbf{Parent} & \textbf{Output} \\
\midrule
0 & \code{pretraining} & 804,755,808,835 & 804,755,808,835 & 191,870 & 1.6810 & 0.2410 & -- & \code{91aec83eac8f} \\
1 & \code{context_extension} & 822,755,808,835 & 18,000,000,000 & 3,456 & 1.1832 & 0.2342 & \code{91aec83eac8f} & \code{0ecb0918cd31} \\
2a & \code{cold_start_sft} (60B gate) & 882,757,458,866 & 60,001,650,031 & 7,159 & 0.7618 & 0.1714 & \code{0ecb0918cd31} & \code{265983c10bd7} \\
2b & \code{cold_start_sft} & 942,759,418,075 & 120,003,609,240 & 14,318 & 0.7144 & 0.1602 & \code{0ecb0918cd31} & \code{4959ad842036} \\
3a & \code{overall_sft} (40B gate) & 982,775,588,761 & 40,016,170,686 & 2,394 & 0.2104 & 0.1755 & \code{4959ad842036} & \code{37f14d20288e} \\
3b & \code{overall_sft} (released) & 1,022,774,982,313 & 80,015,564,238 & 4,787 & 0.1828 & 0.1703 & \code{4959ad842036} & \code{8e7512a491e9} \\
\bottomrule
\end{tabular}}
\end{table}

\begin{table}[tbp]
\centering
\caption{Eight-suite gate of Praxis post-training, as in Table~\ref{tab:rl_suite}. The released model is the last column.}
\label{tab:praxis_suite}
\small
\resizebox{\textwidth}{!}{%
\begin{tabular}{@{}lrrrrr@{}}
\toprule
\textbf{Suite : mode} & \textbf{Extended base} & \textbf{Cold start 60B} & \textbf{Cold start 120B} & \textbf{Overall 40B} & \textbf{Overall 80B} \\
\midrule
AIME : think & 0.00 (4096) & 14.69 (63982) & \textbf{27.08} (64101) & 18.33 (58796) & 25.00 (56893) \\
AIME : direct & 0.00 (512) & 0.10 (620) & \textbf{0.83} (754) & 0.00 (636) & 0.00 (648) \\
HMMT Feb 26 : think & 0.00 (4096) & 10.42 (65413) & 17.05 (65392) & 9.09 (65404) & \textbf{24.24} (65404) \\
GPQA Diamond : think & 0.00 (4096) & 31.16 (18572) & \textbf{31.57} (20354) & 28.28 (14873) & 27.78 (17030) \\
GPQA Diamond : direct & 24.24 (512) & 26.83 (699) & \textbf{29.51} (935) & 25.76 (668) & 22.22 (897) \\
GPQA Main rest : think & 0.40 (4822) & 14.40 (15881) & 15.60 (16034) & \textbf{25.60} (12950) & 14.80 (15120) \\
GPQA Main rest : direct & 3.60 (4778) & \textbf{27.20} (754) & 26.00 (776) & 26.40 (600) & 26.00 (728) \\
MMLU-Pro : think & 0.00 (6031) & 29.57 (8628) & 28.71 (8336) & 33.14 (7210) & \textbf{34.57} (7277) \\
MMLU-Pro : direct & 1.29 (3748) & 25.43 (541) & 27.86 (575) & 28.43 (546) & \textbf{32.29} (538) \\
MMLU : think & 0.29 (4510) & 44.14 (4507) & 46.86 (4720) & 45.57 (3273) & \textbf{53.43} (3407) \\
MMLU : direct & 3.71 (2974) & 39.14 (221) & 40.43 (249) & 46.14 (214) & \textbf{50.43} (239) \\
IFEval : think & 2.11 (1275) & 19.73 (1272) & 22.22 (1641) & 61.88 (1162) & \textbf{64.56} (1197) \\
IFEval : direct & 8.43 (1245) & 14.18 (210) & 16.09 (195) & 60.73 (253) & \textbf{66.48} (239) \\
BFCL v4 : think & 6.40 (2680) & 15.40 (629) & 14.40 (730) & \textbf{55.40} (268) & 52.40 (289) \\
BFCL v4 : direct & 26.60 (1748) & 28.40 (228) & 27.20 (266) & 47.00 (67) & \textbf{50.40} (74) \\
\bottomrule
\end{tabular}}
\end{table}

%=====================================================================
\section{How good is Metis-1.6, really?}
\label{sec:eval}

\subsection{Protocol}
\label{sec:evalprotocol}
\paragraph{Base models.} The context-extended base checkpoints were evaluated on 50 benchmarks with one harness. Multiple-choice tasks use likelihood scoring and the rest greedy generation, with each model's own tokenizer and identical prompts and scorers. No model is run past its context window. Appendix~\ref{app:protocols} gives the protocol of every benchmark. I report the mean of the 44 short-context benchmarks and, for models that can run every length, the mean of all 50.

\paragraph{Post-trained models.} The post-trained releases of Metis and of the comparison models were run on one harness with identical prompts, scorers and verifiers. The suite covers competition mathematics, MATH-500~\cite{hendrycks2021math}, GSM8K~\cite{cobbe2021gsm8k}, GPQA~\cite{rein2023gpqa}, MMLU~\cite{hendrycks2020mmlu}, MMLU-Pro~\cite{wang2024mmlupro}, Humanity's Last Exam~\cite{phan2025humanity}, LiveCodeBench~\cite{jain2024livecodebench}, HumanEval and HumanEval+~\cite{chen2021codex,liu2023evalplus}, BFCL~\cite{patil2023gorilla}, IFEval~\cite{zhou2023ifeval}, RULER~\cite{hsieh2024ruler}, LongBench v2~\cite{bai2024longbenchv2} and more. Metis is reported in think and direct modes, sampled at temperature 0.6, top-p 0.95 and top-k 20. Contest mathematics (AIME, HMMT, AMC, Brumo, Apex) is scored as the mean over eight samples per problem (avg@8); everything else uses one sample. Every model first ran with a 32,768-token generation limit; items that hit the limit were rerun with a higher limit, or none, until no response was truncated, for Metis and the comparators alike. ScaleAI MultiChallenge, which appears on the model cards, is omitted here. Scores on this harness are not directly comparable with those published by model developers.

\subsection{Where it shines and where it struggles}
Table~\ref{tab:main} compares the released models with open post-trained models of similar size on the benchmarks I consider most telling, and Appendix~\ref{app:post} has all of them. Figure~\ref{fig:domains} breaks the 40-benchmark mean into eight domains and asks, for each, how much training data it takes to beat Logos. Figure~\ref{fig:domains_active} asks the same question of compute.

\begin{table}[tbp]
\centering
\caption{Released Metis-1.6 models (think mode) against open post-trained models on selected benchmarks. The last row is the mean over the 40 benchmarks of Appendix~\ref{app:post}. Bold marks the best score in each row. Training tokens are as published.}
\label{tab:main}
\scriptsize
\setlength{\tabcolsep}{3pt}
\textbf{(a) Logos}\par\smallskip
\resizebox{\textwidth}{!}{\begin{tabular}{@{}lrrrrrrrrrrr@{}}
\toprule
 & \textbf{Logos} & \textbf{Qwen3} & \textbf{Qwen3.5} & \textbf{Gemma 4} & \textbf{Qwen3} & \textbf{Qwen3.5} & \textbf{Gemma 4} & \textbf{Granite 4.1} & \textbf{Olmo 3} & \textbf{LFM2.5} & \textbf{Granite 4.0} \\
\textbf{Benchmark} & \textbf{} & \textbf{4B} & \textbf{4B} & \textbf{E4B} & \textbf{8B} & \textbf{9B} & \textbf{12B} & \textbf{8B} & \textbf{7B} & \textbf{8B-A1B} & \textbf{H Tiny} \\
\midrule
Active params per pass & 1.18B & 3.63B & 3.57B & 4.03B & 6.95B & 6.92B & 10.94B & 7.97B & 6.48B & 1.42B & 1.31B \\
Training tokens & 1.02T & 36T & undisclosed & undisclosed & 36T & undisclosed & undisclosed & 14.9T & 6.08T & 38T & 22.5T \\
\midrule
AIME 2025 & 67.50 & 53.75 & 64.17 & 41.25 & 52.08 & 71.25 & \textbf{79.17} & 16.67 & 24.58 & 20.83 & 7.92 \\
AIME 2026 & 67.50 & 63.75 & 73.33 & 46.67 & 61.25 & 78.75 & \textbf{79.58} & 13.33 & 29.17 & 23.33 & 3.75 \\
HMMT Feb 2025 & \textbf{62.92} & 35.83 & 36.25 & 25.00 & 33.75 & 45.42 & 61.67 & 10.00 & 10.42 & 10.42 & 2.08 \\
HMMT Feb 2026 & \textbf{63.64} & 34.85 & 35.61 & 31.82 & 41.67 & 39.39 & 53.03 & 12.12 & 18.94 & 14.39 & 1.52 \\
IMO-AnswerBench & \textbf{57.00} & 37.00 & 33.50 & 29.50 & 37.00 & 38.50 & 44.50 & 10.00 & 12.00 & 4.00 & 4.00 \\
MATH-500 & 89.40 & 88.00 & 83.60 & 90.40 & 88.20 & 88.40 & \textbf{94.20} & 78.40 & 76.60 & 71.40 & 67.00 \\
GPQA Diamond & 55.05 & 52.02 & 76.26 & 61.62 & 53.54 & \textbf{81.31} & 74.24 & 40.40 & 37.37 & 28.28 & 30.30 \\
MMLU-Pro & 59.14 & 64.86 & 66.86 & 68.00 & 68.86 & \textbf{71.43} & 61.57 & 59.14 & 47.00 & 49.86 & 33.43 \\
MMLU & 68.00 & 78.14 & 73.57 & 79.57 & \textbf{81.43} & 74.29 & 76.29 & 77.29 & 62.86 & 65.57 & 59.14 \\
LiveCodeBench v6 & 24.57 & 34.29 & 17.14 & \textbf{42.29} & 32.00 & 20.57 & 39.43 & 22.29 & 0.57 & 18.29 & 20.57 \\
HumanEval+ & 62.20 & 75.61 & 74.39 & 79.88 & 74.39 & 77.44 & \textbf{83.54} & 70.12 & 67.68 & 64.63 & 64.02 \\
BFCL v4 & 58.60 & 58.20 & 48.20 & 66.20 & 59.80 & 54.60 & \textbf{67.20} & 60.00 & 44.80 & 36.60 & 55.80 \\
IFEval (strict) & 71.07 & 86.21 & 89.66 & 77.78 & 86.97 & 92.15 & \textbf{92.91} & 84.87 & 85.25 & 84.67 & 78.74 \\
IFBench & 39.00 & 34.00 & 60.33 & 41.00 & 35.00 & 63.67 & \textbf{74.00} & 36.67 & 34.33 & 52.00 & 20.33 \\
RULER 32K & 51.67 & 85.00 & \textbf{97.22} & 86.11 & 87.78 & 90.56 & 88.33 & 76.67 & 53.33 & 41.11 & 61.67 \\
LongBench v2 & 13.20 & 31.60 & 35.20 & 35.60 & 33.60 & 34.80 & \textbf{40.00} & 24.40 & 33.60 & 15.60 & 14.40 \\
SimpleQA Verified & 3.78 & 3.80 & 4.80 & 3.20 & 6.40 & 7.00 & 4.60 & 5.00 & 3.20 & \textbf{8.00} & 4.60 \\
\midrule
Mean of all 40 & 52.93 & 53.68 & 58.51 & 54.59 & 55.09 & 61.90 & \textbf{64.06} & 46.72 & 43.38 & 40.52 & 38.41 \\
\bottomrule
\end{tabular}
}
\par\bigskip
\textbf{(b) Praxis}\par\smallskip
\resizebox{\textwidth}{!}{\begin{tabular}{@{}lrrrrrrrrrr@{}}
\toprule
 & \textbf{Praxis} & \textbf{Qwen3.5} & \textbf{Qwen3.5} & \textbf{Gemma 4} & \textbf{LFM2.5} & \textbf{SmolLM3} & \textbf{Llama 3.2} & \textbf{Granite 4.1} & \textbf{Ministral 3} & \textbf{Qwen3} \\
\textbf{Benchmark} & \textbf{} & \textbf{0.8B} & \textbf{2B} & \textbf{E2B} & \textbf{2.6B} & \textbf{3B} & \textbf{3B} & \textbf{3B} & \textbf{3B} & \textbf{4B} \\
\midrule
Active params per pass & 0.464B & 0.498B & 1.37B & 1.90B & 2.44B & 2.81B & 2.82B & 3.15B & 3.03B & 3.63B \\
Training tokens & 1.02T & undisclosed & undisclosed & undisclosed & 34T & 11.2T & 9T & 14.9T & 1-3T & 36T \\
\midrule
AIME 2025 & 27.08 & 0.42 & 18.75 & 29.58 & 42.08 & 34.58 & 0.00 & 5.83 & 15.00 & \textbf{53.75} \\
AIME 2026 & 31.25 & 0.00 & 14.17 & 40.83 & 40.83 & 40.83 & 0.42 & 4.58 & 9.58 & \textbf{63.75} \\
HMMT Feb 2025 & 16.25 & 0.00 & 10.83 & 20.83 & 24.58 & 24.17 & 0.42 & 0.42 & 5.83 & \textbf{35.83} \\
HMMT Feb 2026 & 22.73 & 0.00 & 12.50 & 18.94 & 22.35 & 26.89 & 0.76 & 0.76 & 6.06 & \textbf{34.85} \\
IMO-AnswerBench & 16.00 & 5.00 & 11.50 & 20.50 & 23.00 & 27.50 & 0.00 & 7.00 & 9.50 & \textbf{37.00} \\
MATH-500 & 78.80 & 43.60 & 72.80 & 86.00 & 84.60 & 85.20 & 32.40 & 61.60 & 64.60 & \textbf{88.00} \\
GPQA Diamond & 27.78 & 23.23 & 51.01 & 42.93 & \textbf{54.04} & 33.33 & 24.75 & 32.32 & 33.33 & 52.02 \\
MMLU-Pro & 34.57 & 27.86 & 53.14 & 59.43 & 62.14 & 55.57 & 34.14 & 34.57 & 48.29 & \textbf{64.86} \\
MMLU & 53.43 & 38.00 & 70.00 & 73.43 & \textbf{78.29} & 71.14 & 57.86 & 52.86 & 67.57 & 78.14 \\
LiveCodeBench v6 & 12.57 & 4.00 & 9.14 & 19.43 & 30.86 & 23.43 & 12.00 & 19.43 & 24.00 & \textbf{34.29} \\
HumanEval+ & 42.07 & 13.41 & 40.85 & 72.56 & 74.39 & 71.34 & 46.95 & 65.24 & 65.85 & \textbf{75.61} \\
BFCL v4 & 52.40 & 34.20 & 45.60 & \textbf{65.20} & 31.80 & 38.40 & 24.20 & 44.20 & 52.20 & 58.20 \\
IFEval (strict) & 64.56 & 63.22 & 79.50 & 76.05 & \textbf{94.25} & 70.11 & 70.69 & 81.80 & 57.66 & 86.21 \\
IFBench & 34.33 & 32.33 & 44.33 & 39.33 & \textbf{57.67} & 25.00 & 30.00 & 35.67 & 26.33 & 34.00 \\
RULER 32K & 25.00 & 82.78 & \textbf{93.33} & 78.33 & 64.44 & 70.00 & 71.11 & 67.78 & 81.11 & 85.00 \\
LongBench v2 & 22.40 & 20.00 & 24.80 & 28.00 & 26.40 & 26.00 & 20.40 & 10.00 & 20.00 & \textbf{31.60} \\
SimpleQA Verified & 2.40 & 1.40 & 3.20 & 3.80 & 2.20 & 4.20 & 2.80 & \textbf{5.00} & 4.80 & 3.80 \\
\midrule
Mean of all 40 & 36.62 & 27.61 & 43.72 & 48.64 & 51.12 & 44.69 & 33.92 & 36.45 & 37.17 & \textbf{53.68} \\
\bottomrule
\end{tabular}
}
\end{table}

\begin{figure}[tbp]
\centering
\begin{tikzpicture}
\begin{groupplot}[group style={group size=4 by 2,horizontal sep=0.55cm,vertical sep=1.35cm},
  width=0.31\linewidth,height=4.3cm,xmode=log,xmin=0.7,xmax=60,xtick={1,10},
  log ticks with fixed point,grid=major,grid style={black!8},axis line style={black!35},tick style={black!35},tick label style={font=\scriptsize},label style={font=\small},clip=false,title style={font=\scriptsize,yshift=-3pt}]
\nextgroupplot[title={\textbf{Mathematics}\\Logos 2nd of 30},title style={align=center},ymin=0.7,ymax=75.9]
\fill[frontier,fill opacity=0.075] (axis cs:0.7,30.84) rectangle (axis cs:11.200,75.9);
\draw[black!30,densely dashed,line width=0.4pt] (axis cs:11.200,0.7) -- (axis cs:11.200,75.9);
\draw[black!30,densely dashed,line width=0.4pt] (axis cs:0.7,30.84) -- (axis cs:60,30.84);
\addplot[const plot,very thick,frontier] coordinates { (1.02,65.83) (60,65.83)};
\addplot[only marks,mark=*,mark size=1.4pt,black!35] coordinates { (36.0,55.05) (38.0,30.21) (22.5,21.00) (10.0,16.76) (5.13,8.00) (36.0,55.63) (12.0,30.84) (14.9,30.30) (6.08,33.05) (12.0,22.29) (28.0,65.31) (14.9,18.79) (2.5,25.93) (7.7,47.80) (11.2,45.23) (34.0,42.50) (4.0,24.25) (9.0,11.36) (7.7,54.76)};
\addplot[only marks,mark=diamond*,mark size=3.2pt,black] coordinates { (1.02,65.83)};
\addplot[only marks,mark=diamond*,mark size=3.2pt,mark options={fill=white,draw=black,thick}] coordinates { (1.02,35.97)};
\nextgroupplot[title={\textbf{Science and knowledge}\\Logos 9th of 30},title style={align=center},ymin=15.5,ymax=67.0]
\fill[frontier,fill opacity=0.075] (axis cs:0.7,38.98) rectangle (axis cs:11.200,67.0);
\draw[black!30,densely dashed,line width=0.4pt] (axis cs:11.200,15.5) -- (axis cs:11.200,67.0);
\draw[black!30,densely dashed,line width=0.4pt] (axis cs:0.7,38.98) -- (axis cs:60,38.98);
\addplot[const plot,very thick,frontier] coordinates { (1.02,46.43) (28.0,61.88) (60,61.88)};
\addplot[only marks,mark=*,mark size=1.4pt,black!35] coordinates { (36.0,49.11) (38.0,38.32) (22.5,31.01) (10.0,28.66) (5.13,20.45) (36.0,51.68) (12.0,42.14) (14.9,43.25) (6.08,36.95) (12.0,45.24) (28.0,61.88) (14.9,31.09) (2.5,36.69) (7.7,42.64) (11.2,38.98) (34.0,51.29) (4.0,31.69) (9.0,28.94) (7.7,44.57)};
\addplot[only marks,mark=diamond*,mark size=3.2pt,black] coordinates { (1.02,46.43)};
\addplot[only marks,mark=diamond*,mark size=3.2pt,mark options={fill=white,draw=black,thick}] coordinates { (1.02,29.32)};
\nextgroupplot[title={\textbf{Factual recall}\\Logos 10th of 30},title style={align=center},ymin=0,ymax=14.1]
\fill[frontier,fill opacity=0.075] (axis cs:0.7,5.96) rectangle (axis cs:11.200,14.1);
\draw[black!30,densely dashed,line width=0.4pt] (axis cs:11.200,0) -- (axis cs:11.200,14.1);
\draw[black!30,densely dashed,line width=0.4pt] (axis cs:0.7,5.96) -- (axis cs:60,5.96);
\addplot[const plot,very thick,frontier] coordinates { (1.02,7.11) (12.0,7.99) (14.9,9.54) (38.0,12.08) (60,12.08)};
\addplot[only marks,mark=*,mark size=1.4pt,black!35] coordinates { (36.0,6.39) (38.0,12.08) (22.5,6.75) (10.0,5.72) (5.13,3.97) (36.0,9.06) (12.0,7.99) (14.9,9.54) (6.08,4.17) (12.0,6.56) (28.0,8.00) (14.9,5.96) (2.5,3.50) (7.7,4.14) (11.2,5.70) (34.0,5.67) (4.0,5.43) (9.0,6.14) (7.7,5.28)};
\addplot[only marks,mark=diamond*,mark size=3.2pt,black] coordinates { (1.02,7.11)};
\addplot[only marks,mark=diamond*,mark size=3.2pt,mark options={fill=white,draw=black,thick}] coordinates { (1.02,2.81)};
\nextgroupplot[title={\textbf{Code and tools}\\Logos 17th of 30},title style={align=center},ymin=10.1,ymax=78.4]
\fill[frontier,fill opacity=0.075] (axis cs:0.7,55.40) rectangle (axis cs:11.200,78.4);
\draw[black!30,densely dashed,line width=0.4pt] (axis cs:11.200,10.1) -- (axis cs:11.200,78.4);
\draw[black!30,densely dashed,line width=0.4pt] (axis cs:0.7,55.40) -- (axis cs:60,55.40);
\addplot[const plot,very thick,frontier] coordinates { (1.02,55.40) (7.7,61.24) (28.0,65.73) (60,65.73)};
\addplot[only marks,mark=*,mark size=1.4pt,black!35] coordinates { (36.0,65.20) (38.0,49.70) (22.5,55.52) (10.0,36.34) (5.13,16.73) (36.0,63.65) (12.0,60.31) (14.9,59.45) (6.08,49.00) (12.0,51.91) (28.0,65.73) (14.9,52.80) (2.5,44.45) (7.7,56.41) (11.2,55.55) (34.0,57.28) (4.0,53.81) (9.0,35.88) (7.7,61.24)};
\addplot[only marks,mark=diamond*,mark size=3.2pt,black] coordinates { (1.02,55.40)};
\addplot[only marks,mark=diamond*,mark size=3.2pt,mark options={fill=white,draw=black,thick}] coordinates { (1.02,39.72)};
\nextgroupplot[title={\textbf{Commonsense}\\Logos 17th of 30},title style={align=center},ymin=45.2,ymax=97.3,xlabel={\scriptsize Tokens (T)}]
\fill[frontier,fill opacity=0.075] (axis cs:0.7,79.38) rectangle (axis cs:11.200,97.3);
\draw[black!30,densely dashed,line width=0.4pt] (axis cs:11.200,45.2) -- (axis cs:11.200,97.3);
\draw[black!30,densely dashed,line width=0.4pt] (axis cs:0.7,79.38) -- (axis cs:60,79.38);
\addplot[const plot,very thick,frontier] coordinates { (1.02,79.56) (7.7,80.09) (11.2,82.82) (12.0,89.29) (14.9,89.95) (36.0,92.21) (60,92.21)};
\addplot[only marks,mark=*,mark size=1.4pt,black!35] coordinates { (36.0,89.51) (38.0,79.31) (22.5,72.98) (10.0,79.38) (5.13,50.22) (36.0,92.21) (12.0,89.29) (14.9,89.95) (6.08,73.33) (12.0,87.94) (28.0,88.60) (14.9,70.87) (2.5,62.96) (7.7,80.09) (11.2,82.82) (34.0,86.32) (4.0,75.00) (9.0,74.56) (7.7,75.05)};
\addplot[only marks,mark=diamond*,mark size=3.2pt,black] coordinates { (1.02,79.56)};
\addplot[only marks,mark=diamond*,mark size=3.2pt,mark options={fill=white,draw=black,thick}] coordinates { (1.02,63.51)};
\nextgroupplot[title={\textbf{Instruction following}\\Logos 18th of 30},title style={align=center},ymin=29.7,ymax=93.4,xlabel={\scriptsize Tokens (T)}]
\fill[frontier,fill opacity=0.075] (axis cs:0.7,64.66) rectangle (axis cs:11.200,93.4);
\draw[black!30,densely dashed,line width=0.4pt] (axis cs:11.200,29.7) -- (axis cs:11.200,93.4);
\draw[black!30,densely dashed,line width=0.4pt] (axis cs:0.7,64.66) -- (axis cs:60,64.66);
\addplot[const plot,very thick,frontier] coordinates { (1.02,63.25) (2.5,64.66) (6.08,67.42) (12.0,68.96) (28.0,74.80) (34.0,81.64) (60,81.64)};
\addplot[only marks,mark=*,mark size=1.4pt,black!35] coordinates { (36.0,68.63) (38.0,74.00) (22.5,58.80) (10.0,40.87) (5.13,49.44) (36.0,70.10) (12.0,68.96) (14.9,68.40) (6.08,67.42) (12.0,67.01) (28.0,74.80) (14.9,65.82) (2.5,64.66) (7.7,35.84) (11.2,54.81) (34.0,81.64) (4.0,62.50) (9.0,56.90) (7.7,39.33)};
\addplot[only marks,mark=diamond*,mark size=3.2pt,black] coordinates { (1.02,63.25)};
\addplot[only marks,mark=diamond*,mark size=3.2pt,mark options={fill=white,draw=black,thick}] coordinates { (1.02,54.63)};
\nextgroupplot[title={\textbf{Safety}\\Logos 19th of 30},title style={align=center},ymin=30.7,ymax=101.9,xlabel={\scriptsize Tokens (T)}]
\fill[frontier,fill opacity=0.075] (axis cs:0.7,62.84) rectangle (axis cs:11.200,101.9);
\draw[black!30,densely dashed,line width=0.4pt] (axis cs:11.200,30.7) -- (axis cs:11.200,101.9);
\draw[black!30,densely dashed,line width=0.4pt] (axis cs:0.7,62.84) -- (axis cs:60,62.84);
\addplot[const plot,very thick,frontier] coordinates { (1.02,61.12) (2.5,65.85) (5.13,80.83) (6.08,85.92) (28.0,92.90) (60,92.90)};
\addplot[only marks,mark=*,mark size=1.4pt,black!35] coordinates { (36.0,57.71) (38.0,37.58) (22.5,75.39) (10.0,41.10) (5.13,80.83) (36.0,60.78) (12.0,57.91) (14.9,74.08) (6.08,85.92) (12.0,67.41) (28.0,92.90) (14.9,62.92) (2.5,65.85) (7.7,58.97) (11.2,39.91) (34.0,78.60) (4.0,55.24) (9.0,77.40) (7.7,62.84)};
\addplot[only marks,mark=diamond*,mark size=3.2pt,black] coordinates { (1.02,61.12)};
\addplot[only marks,mark=diamond*,mark size=3.2pt,mark options={fill=white,draw=black,thick}] coordinates { (1.02,56.20)};
\nextgroupplot[title={\textbf{Long context}\\Logos 24th of 30},title style={align=center},ymin=15.3,ymax=44.4,xlabel={\scriptsize Tokens (T)}]
\fill[frontier,fill opacity=0.075] (axis cs:0.7,26.35) rectangle (axis cs:11.200,44.4);
\draw[black!30,densely dashed,line width=0.4pt] (axis cs:11.200,15.3) -- (axis cs:11.200,44.4);
\draw[black!30,densely dashed,line width=0.4pt] (axis cs:0.7,26.35) -- (axis cs:60,26.35);
\addplot[const plot,very thick,frontier] coordinates { (1.02,23.34) (4.0,28.89) (12.0,34.02) (36.0,34.27) (60,34.27)};
\addplot[only marks,mark=*,mark size=1.4pt,black!35] coordinates { (36.0,33.48) (38.0,20.02) (22.5,25.73) (10.0,24.19) (5.13,18.10) (36.0,34.27) (12.0,34.02) (14.9,30.77) (6.08,27.61) (12.0,30.73) (28.0,28.41) (14.9,22.25) (2.5,22.00) (7.7,22.93) (11.2,28.25) (34.0,26.35) (4.0,28.89) (9.0,28.56) (7.7,23.52)};
\addplot[only marks,mark=diamond*,mark size=3.2pt,black] coordinates { (1.02,23.34)};
\addplot[only marks,mark=diamond*,mark size=3.2pt,mark options={fill=white,draw=black,thick}] coordinates { (1.02,18.23)};
\end{groupplot}
\end{tikzpicture}
\caption{Where Logos does well and where it does not. Each panel is one domain: the unweighted mean of its benchmarks (think mode) against published training tokens, for every model that publishes a count. The rank in each title counts all 30 models in both comparison sets, including those without a published count. The blue staircase is the Pareto frontier, the best score reached with at most that much data. Where the staircase stays flat after Logos (filled diamond), no comparator that publishes its data beats it; where it climbs early, cheap models beat it. Panels run from Logos's best domain to its worst. Praxis is the open diamond, and the dashed lines and shading are as in Figure~\ref{fig:frontier}.}
\label{fig:domains}
\end{figure}
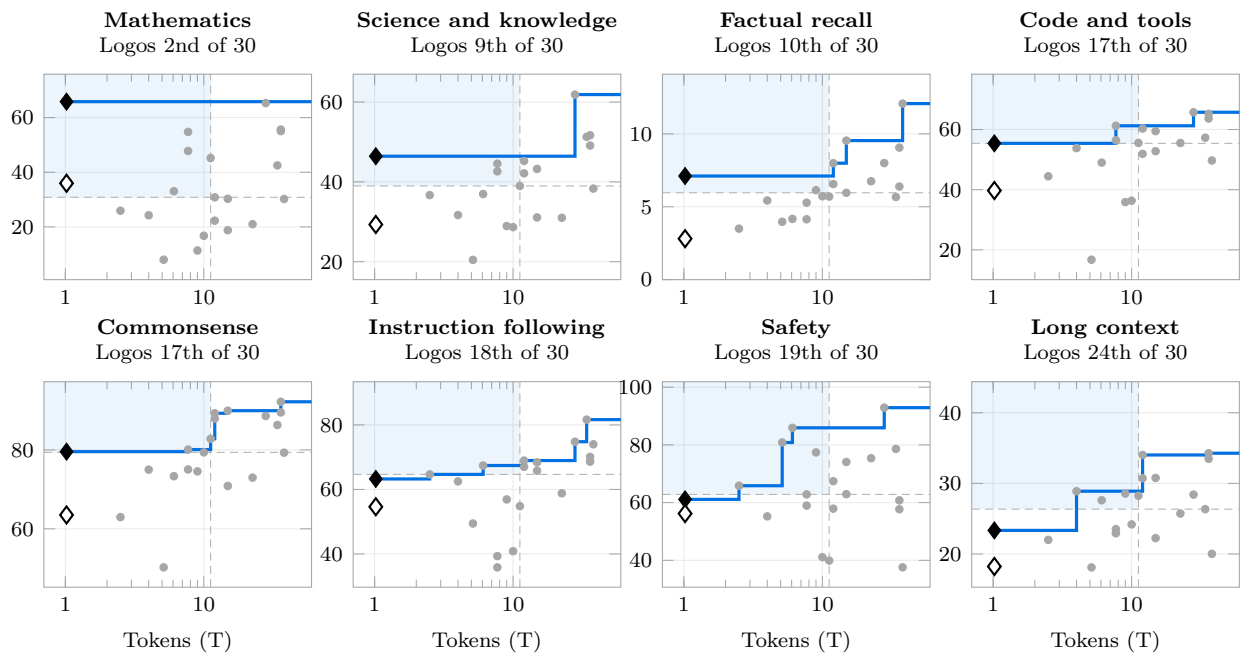

\begin{figure}[!t]
\centering
\input{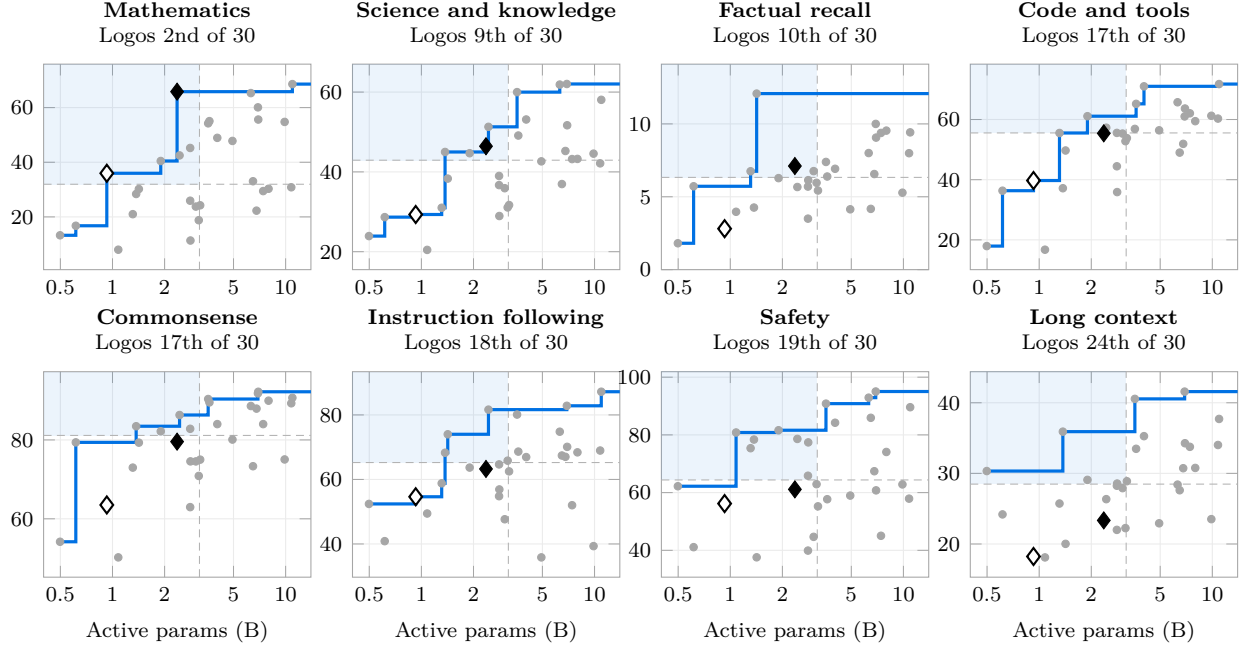}
\caption{The same eight domains against active parameters per token, the compute every generated token costs. Every model publishes this number, so all 30 appear. Panels, order and vertical axes match Figure~\ref{fig:domains}, so the two can be read side by side. The shaded corner holds the models that beat the median score while activating fewer parameters than the median model. Logos (filled diamond) holds the compute frontier in mathematics and in science and knowledge; in each of the other six domains, between one and five models with fewer active parameters beat it. Praxis (open diamond) holds it in four.}
\label{fig:domains_active}
\end{figure}

\paragraph{Logos.} Logos averages 52.93, seventh of 17 in its set. Ahead of it are Gemma~4~12B (64.06), Qwen3.5~9B (61.90), Qwen3.5~4B (58.51), Qwen3~8B (55.09), Gemma~4~E4B (54.59) and Qwen3~4B (53.68). Behind it are Granite~4.1~8B (46.72), Gemma~3~12B, Olmo~3~7B, Falcon-H1~7B, Ministral~3~8B, and the mixture-of-experts models LFM2.5-8B-A1B, Granite~4.0~H~Tiny, Trinity Nano and OLMoE.

Mathematics is where Logos earns its keep. Averaged over eleven math benchmarks it scores 65.83, second of 30 models and 2.77 points behind Gemma~4~12B, and no comparator that publishes its training data beats it. It has the highest score in its set on HMMT February 2025 (62.92), HMMT February 2026 (63.64) and IMO-AnswerBench (57.00), is second only to Gemma~4~12B on HMMT November 2025 and Brumo 2025, and third on AIME 2025 (67.50). On contest mathematics, a 1.02T-token model built by one researcher holds its own against the best small models the large labs have released. Science and knowledge is respectable too: ninth of 30 and fifth in its set on GPQA Diamond (55.05), and the first comparator with published data to beat it there was trained on 28T tokens.

The rest of Figure~\ref{fig:domains} is less flattering, and I would rather say so plainly. In code and tools Logos ranks 17th of 30: tool calling holds up (58.60 on BFCL v4, fifth in its set), code generation does not (62.20 on HumanEval+, behind 13 of 15 comparators). In instruction following it ranks 18th (71.07 on IFEval, behind 12), in safety 19th, and in long context 24th. In think mode, RULER falls from 68.33 at 4K to 34.17 at 128K (84.48 and 60.14 in direct mode) and LongBench v2 is 13.20. Factual recall is low for every model in the study, and Logos's 3.78 on SimpleQA Verified is no exception. In these domains, models trained on as little as 2.5T to 8T tokens already do better, which shows exactly where Metis-1.7's data has to go.

Measured against compute instead of data (Figure~\ref{fig:domains_active}), the picture keeps its shape. In mathematics and in science and knowledge, Logos holds the compute frontier: every model that beats it activates more parameters per token, up to 4.6 times as many for Gemma~4~12B in mathematics. Logos also sits in the shaded corner in factual recall, and Praxis joins it there in mathematics. In the other domains, a few small models from the large labs beat Logos while activating fewer parameters, among them Qwen3.5~2B, Gemma~4~E2B, Granite~4.0~H~Tiny and LFM2.5-8B-A1B, each with 1.3B to 1.9B active; the two of those that publish their data were trained on 22.5T and 38T tokens. Praxis holds up better for its size: no model with fewer active parameters beats it in mathematics, science and knowledge, code and tools, or instruction following.

\paragraph{Praxis.} Praxis averages 36.62, fifteenth of 18 in its set, ahead of Granite~4.1~3B (36.45), Llama~3.2~3B (33.92) and Qwen3.5~0.8B (27.61). Relative to its set, its best results are Humanity's Last Exam (5.85, behind two models) and BFCL v4 (52.40, behind five); it is near the bottom on knowledge, code and long-context tasks. Like Logos, it sits on the compute frontier of Figure~\ref{fig:frontier}b: no model in either set scores higher with fewer active parameters per token.

\paragraph{Efficiency.} The comparison in Figure~\ref{fig:frontier} is not controlled: the comparators differ in data, in post-training (several are distilled from larger models) and in size. Still, two things are clear. Every model that averages above Logos and publishes its data was trained on at least 28T tokens (Nanbeige4.2-3B on 28T, Qwen3-4B and Qwen3-8B on 36T each), against 1.02T for Logos. And every model that averages above Logos activates 3.57B to 10.94B parameters per token, against 2.36B. Mixture-of-experts is a trade: Logos stores 10.23B parameters to activate 2.36B, so it needs more memory than a 4B dense model but less compute for every token it generates, and with key-value caches in only 3 of its 20 layers, long contexts stay cheap.

\subsection{Base models}
As base models, both Metis models trail most open models of similar size (Appendix~\ref{app:base}). Over the 44 short-context benchmarks, the Logos base averages 31.03, ahead only of OLMoE-1B-7B (29.44) in its set, and the Praxis base averages 24.66, last in its set. The gap is widest on knowledge and reasoning benchmarks that are hard for base models, such as MMLU-Pro (22.45 for Logos) and MATH-500 (26.60), and narrowest on commonsense likelihood tasks such as HellaSwag, PIQA, WinoGrande and OpenBookQA. Including the six long-context benchmarks, Logos averages 34.99, and on RULER it holds 70.50 at 64K and 60.14 at 128K. Post-training lifts both models far more than their base scores would suggest, mostly through mathematics and reasoning.

%=====================================================================
\section{Toward Metis-1.7 and MoRE}
\label{sec:ablations}

Metis-1.7 is built on MoRE (Mixture of Recursive Experts), a looped architecture of my own design. It deserves more than a section, and it will get one: a dedicated paper will describe MoRE in full alongside the Metis-1.7 release. Here I report the two controlled campaigns that set its loop and its optimizer, because both grew directly out of what went wrong in Metis-1.6. For scale, Metis-1.7 has 12.48B neural parameters, with 0.794B active per pass and 1.59B per token over two passes, and adds a 422M-parameter vision encoder and a 204.8M-parameter speculative drafter.

Each arm in both campaigns trained for 10B tokens (9,535 steps) on a smaller proxy and was evaluated on the same 256 windows; differences are paired over the windows.

\paragraph{The loop.} Metis-1.6 ties the Mamba-2 and attention weights across both passes, with only a rank-8 attention update between them, so its sequence mixers must compromise between two quite different jobs. Table~\ref{tab:fd2} compares eleven variants. The winning design, \code{f2_smart_untie_layer}, combines five changes:
\begin{itemize}
  \item a model-wise loop;
  \item per-layer dynamic expert counts with per-pass routers;
  \item a gated cross-pass residual highway~\cite{loom2026};
  \item a cosine-decayed auxiliary loss on the first pass;
  \item mixers and embeddings fully untied across passes.
\end{itemize}
It reaches 2.5529 nats (perplexity 12.845), beating the Metis-1.6 loop (\code{fix1/f2}, 2.5993) by $-0.0464 \pm 0.0008$ nats, and a non-recursive 16-layer model with twice the expert parameters (\code{f1_L16}, 2.5774) by $-0.0245 \pm 0.0008$ nats. A looped model that beats an unrolled one with twice its expert parameters is the result MoRE is built on.

\paragraph{The optimizer.} Holding \code{f2_smart_untie_layer} fixed, I compared Muon variants against tuned AdamW (Table~\ref{tab:g13}). The winner combines three changes:
\begin{itemize}
  \item strict optimizer-group assertions at start-up, with fused query-key-value, Mamba input and SwiGLU gate-up matrices orthogonalized as separate submatrices;
  \item NorMuon's neuron-wise second-moment normalization~\cite{li2025normuon} ($\beta_2 = 0.95$, $\epsilon = 10^{-16}$), with five Polar Express steps~\cite{polarexpress2026} and one restarted Gram Newton--Schulz iteration;
  \item weight decay scaled with the learning-rate schedule, $\lambda_t = \lambda_0\,\eta_t/\eta_{\max}$.
\end{itemize}
It reaches 2.37104 nats (perplexity 10.708), $-0.18190 \pm 0.00299$ nats better than AdamW, and wins all 256 windows. A hard projection of weights onto a Frobenius sphere was rejected. That gap is the price of Section~\ref{sec:optim}: Metis-1.6 trained every attention matrix with AdamW.

\begin{table}[tbp]
\centering
\caption{Loop-design campaign (10B tokens, 256 paired evaluation windows). $\Delta$ is the cross-entropy difference from the Metis-1.6 loop design \code{fix1/f2}, with its $z$-score; negative is better.}
\label{tab:fd2}
\scriptsize
\setlength{\tabcolsep}{4pt}
\resizebox{\textwidth}{!}{%
\begin{tabular}{@{}llrrr@{}}
\toprule
\textbf{Arm} & \textbf{Configuration} & \textbf{CE} & \textbf{PPL} & \textbf{$\Delta$ vs \code{fix1/f2} ($z$)} \\
\midrule
\code{fix1/f2} & Metis-1.6 tied two-pass loop & 2.5993 & 13.454 & 0 \\
\code{f2_layer_loop} & layer-wise loop $[L_0, L_0, L_1, L_1, \dots]$ & 2.6124 & 13.632 & $+0.0131$ ($+34.5\sigma$) \\
\code{f2_glob_dynk} & pass-wide dynamic expert count & 2.6031 & 13.505 & $+0.0038$ ($+6.1\sigma$) \\
\code{f2_dynk_pr} & per-layer dynamic count + per-pass routers & 2.5884 & 13.308 & $-0.0109$ ($-15.6\sigma$) \\
\code{f2_p1_boot} & \code{f2_dynk_pr} + first-pass auxiliary loss & 2.5824 & 13.229 & $-0.0169$ ($-22.8\sigma$) \\
\code{f1_L16} & unrolled 16 layers, one pass, $2\times$ parameters & 2.5774 & 13.163 & $-0.0219$ ($-27.4\sigma$) \\
\code{f2_loom} & \code{f2_dynk_pr} + cross-pass highway & 2.5735 & 13.112 & $-0.0258$ ($-33.9\sigma$) \\
\code{f2_smart_vo_layer} & partial value/output untying + highway + first-pass loss & 2.5702 & 13.068 & $-0.0291$ ($-36.4\sigma$) \\
\code{f2_smart_untie} & pass-wide count + full untying + highway + first-pass loss & 2.5634 & 12.980 & $-0.0359$ ($-44.9\sigma$) \\
\code{f2_loom_boot} & tied attention + highway + first-pass loss & 2.5607 & 12.945 & $-0.0386$ ($-48.3\sigma$) \\
\textbf{\code{f2_smart_untie_layer}} & \textbf{full untying + per-layer count + highway + first-pass loss} & \textbf{2.5529} & \textbf{12.845} & $\mathbf{-0.0464 \pm 0.0008}$ ($\mathbf{-58.0\sigma}$) \\
\bottomrule
\end{tabular}}
\end{table}

\begin{table}[tbp]
\centering
\caption{Optimizer campaign on \code{f2_smart_untie_layer} (10B tokens, 9,535 steps, 256 paired windows).}
\label{tab:g13}
\scriptsize
\setlength{\tabcolsep}{4pt}
\resizebox{\textwidth}{!}{%
\begin{tabular}{@{}llrrr@{}}
\toprule
\textbf{Arm} & \textbf{Optimizer} & \textbf{CE} & \textbf{PPL} & \textbf{$\Delta$ vs AdamW ($z$)} \\
\midrule
\code{fd2_adamw} & tuned AdamW, $\beta = (0.9, 0.95)$ & 2.55293 & 12.845 & 0 \\
\code{opt3_normuon_hyperball} & NorMuon + hard Frobenius-sphere projection & \multicolumn{3}{l}{rejected ($+0.134$ nats at 500 steps; $+0.0384$ at 4K)} \\
\code{opt1_muon_split_sw} & split Muon + scheduled weight decay & 2.37445 & 10.745 & $-0.17848 \pm 0.00294$ ($-60.6\sigma$) \\
\code{opt2_normuon_sw} & \code{opt1} + NorMuon & 2.37105 & 10.709 & $-0.18189 \pm 0.00286$ ($-63.5\sigma$) \\
\textbf{\code{opt4_normuon_gram_pe}} & \textbf{\code{opt2} + Polar Express + restarted Gram iteration} & \textbf{2.37104} & \textbf{10.708} & $\mathbf{-0.18190 \pm 0.00299}$ ($\mathbf{-60.7\sigma}$) \\
\bottomrule
\end{tabular}}
\end{table}

%=====================================================================
\section{Limitations}
\label{sec:limits}

\paragraph{Recent contests.} The decontamination manifest predates several of the contests reported here: AIME 2025 and 2026, HMMT 2025 and 2026, Brumo 2025, IMO-AnswerBench and Apex 2025 are not on it. Logos is strongest relative to its comparators on exactly these benchmarks, so a dedicated overlap check against them is the first thing I will add.

\paragraph{Harness.} All comparisons use the Lernex harness, prompts and scorers. Comparators may do better under their developers' recommended settings.

\paragraph{Confounded architecture conclusions.} Metis-1.6 trained attention with AdamW and everything else with plain Muon (Section~\ref{sec:optim}), so its results do not isolate the architecture. The 10B-token ablations use a corrected optimizer, but on a smaller proxy with one run per arm, so their standard errors reflect evaluation noise rather than run-to-run variance.

\paragraph{Reinforcement learning.} The advantage clamp was used in the code, knowledge and agentic specialists, which were short runs, and not in the reasoning specialist where the sign inversion was observed.

%=====================================================================
\section{Conclusion: talk, use, and what comes next}

Metis-1.5 was the first Metis model you could talk to. Metis-1.6 is the first you can use. Built end to end by one researcher on 1.02T tokens, Logos finishes ahead of open models trained on up to 38T tokens and within a point of one trained on 36T, sits on the Pareto frontier for both data and compute, and on contest mathematics stands next to the best small models from the large labs. It is not a frontier model, it does not pretend to be, and it still loses most head-to-head comparisons outside mathematics. It was never meant to win them. It was meant to prove I can build, learn and build better, and the record says I can. It is also the first real evidence for the bet Metis is built on: that the cheaper path to a capable model runs through better data, algorithms and architectures, not more parameters.

That record is the other half of this report:
\begin{itemize}
  \item an optimizer that differed from its configuration for an entire run;
  \item a memory the models chose not to use;
  \item fine-tuning data that taught habits I did not want;
  \item a length reward that, under group normalization, punished correct reasoning;
  \item a distillation step too small to move the weights.
\end{itemize}
Every one of them has been diagnosed, and every one has already changed how Metis-1.7 and MoRE are built. That is the part of this work I am proudest of. Not any single score, but the loop behind it: build a model, measure it honestly, find exactly where I was wrong, and carry the lesson into the next one. Metis-1.3 taught me to train a hybrid. Metis-1.4 taught me what breaks in a looped model. Metis-1.5 learned to hold a conversation, and Metis-1.6 learned to be useful. Each one did something the one before it could not.

I am not going to make promises about Metis-1.7 in these pages. I would rather show you. What I can say is that I have never been more excited to share what comes next, and that I intend to keep working the way this report does: in the open, learning fast, publishing the mistakes alongside the results, and, I hope, surprising a few people along the way. Metis started as a bet that a better idea can do what a bigger budget cannot. Every release so far has made that bet look a little better. See you at Metis-1.7.

\bibliographystyle{plainnat}
\bibliography{references}

%=====================================================================
\clearpage
\newgeometry{margin=0.6in}
\appendix

\section{Base-model evaluation}
\label{app:base}
Tables~\ref{tab:app_base_logos} and~\ref{tab:app_base_praxis} give all 50 base benchmarks for the Logos and Praxis comparison sets. The short-context mean leaves out RULER beyond 4K and LongBench v2. The 50-benchmark mean is shown only for models that can run every length.

\begin{landscape}
{\tiny\setlength{\tabcolsep}{1.6pt}\renewcommand{\arraystretch}{1.08}
\begin{longtable}{@{}>{\raggedright\arraybackslash}p{2.35cm}rrrrrrrrrrrrrrrr@{}}
\caption{Base models, Logos comparison set: Metis-1.6 base checkpoints after context extension against open base models, on one harness with identical prompts and scorers (Appendix~\ref{app:protocols}). Columns are ordered by the short-context mean. Bold marks the best score in each row; -- means the length exceeds the model's context window or what fits the 16\,GB evaluation GPU.}\label{tab:app_base_logos}\\
\toprule
 &  &  &  &  & \textbf{Ministral} & \textbf{Falcon-} &  &  & \textbf{Granite} &  & \textbf{Granite} &  &  &  &  &  \\
 & \textbf{Logos} & \textbf{Praxis} & \textbf{Qwen3} & \textbf{Qwen3.5} & \textbf{3} & \textbf{H1} & \textbf{Qwen3.5} & \textbf{Qwen3} & \textbf{4.1} & \textbf{Gemma 4} & \textbf{4.0} & \textbf{Gemma 4} & \textbf{Olmo 3} & \textbf{Trinity} & \textbf{LFM2.5} & \textbf{OLMoE} \\
\textbf{Benchmark} &  &  & \textbf{8B} & \textbf{9B} & \textbf{8B} & \textbf{7B} & \textbf{4B} & \textbf{4B} & \textbf{8B} & \textbf{12B} & \textbf{H Tiny} & \textbf{E4B} & \textbf{7B} & \textbf{Nano} & \textbf{8B-A1B} & \textbf{1B-7B} \\
\midrule
\endfirsthead
\multicolumn{17}{@{}l}{\emph{(continued)}}\\
\toprule
 &  &  &  &  & \textbf{Ministral} & \textbf{Falcon-} &  &  & \textbf{Granite} &  & \textbf{Granite} &  &  &  &  &  \\
 & \textbf{Logos} & \textbf{Praxis} & \textbf{Qwen3} & \textbf{Qwen3.5} & \textbf{3} & \textbf{H1} & \textbf{Qwen3.5} & \textbf{Qwen3} & \textbf{4.1} & \textbf{Gemma 4} & \textbf{4.0} & \textbf{Gemma 4} & \textbf{Olmo 3} & \textbf{Trinity} & \textbf{LFM2.5} & \textbf{OLMoE} \\
\textbf{Benchmark} &  &  & \textbf{8B} & \textbf{9B} & \textbf{8B} & \textbf{7B} & \textbf{4B} & \textbf{4B} & \textbf{8B} & \textbf{12B} & \textbf{H Tiny} & \textbf{E4B} & \textbf{7B} & \textbf{Nano} & \textbf{8B-A1B} & \textbf{1B-7B} \\
\midrule
\endhead
\bottomrule
\endfoot
Mean of 44 (short context) & 31.03 & 24.66 & \textbf{44.65} & 44.41 & 43.08 & 43.00 & 42.93 & 42.88 & 41.54 & 41.19 & 39.01 & 38.55 & 38.16 & 36.48 & 35.46 & 29.44 \\
Mean of all 50 & 34.99 & 27.72 & -- & 46.05 & \textbf{46.75} & 46.55 & 45.57 & -- & 45.20 & 45.11 & 42.10 & 41.54 & -- & 39.64 & -- & -- \\
\midrule
Humanity's Last Exam & 4.59 & 3.89 & 5.61 & 2.87 & 5.33 & 5.24 & 3.34 & 3.99 & 5.19 & 4.77 & 4.17 & 4.22 & \textbf{6.67} & 5.42 & 4.54 & 4.73 \\
HLE-Diamond & 3.71 & 3.43 & 4.40 & 2.88 & 4.26 & 4.26 & 2.88 & 5.08 & \textbf{6.18} & 5.49 & 4.12 & 3.71 & 5.91 & 5.22 & 3.71 & 5.22 \\
HLE-Verified (Gold) & 5.91 & 4.00 & 5.74 & 3.65 & 5.91 & 6.61 & 4.00 & 4.35 & 6.09 & 6.43 & 3.83 & 5.04 & \textbf{7.65} & 6.26 & 5.04 & 5.57 \\
GPQA & 30.36 & 27.90 & \textbf{39.96} & 38.84 & 37.95 & 37.05 & 25.67 & 22.99 & 34.60 & 35.94 & 31.25 & 26.12 & 30.36 & 31.70 & 26.12 & 26.79 \\
GPQA Diamond & 27.27 & 24.24 & 36.36 & 33.84 & 37.37 & \textbf{38.89} & 24.24 & 25.25 & 35.86 & 36.36 & 30.81 & 27.27 & 32.32 & 31.31 & 22.73 & 23.23 \\
SuperGPQA & 15.29 & 11.12 & 33.21 & \textbf{34.49} & 30.63 & 27.68 & 27.11 & 28.07 & 23.28 & 26.99 & 19.16 & 24.71 & 21.65 & 21.36 & 23.53 & 13.55 \\
MMLU-Pro & 22.45 & 14.19 & 48.89 & \textbf{54.68} & 50.18 & 49.10 & 48.62 & 45.74 & 38.66 & 44.84 & 33.11 & 38.19 & 33.87 & 34.31 & 39.15 & 18.53 \\
BIG-Bench Extra Hard & 4.60 & 1.66 & \textbf{10.07} & 3.27 & 8.14 & 5.31 & 2.17 & 7.92 & 4.91 & 5.31 & 5.69 & 2.19 & 7.65 & 6.22 & 1.17 & 7.92 \\
ARC-AGI-2 & \textbf{0.00} & \textbf{0.00} & \textbf{0.00} & \textbf{0.00} & \textbf{0.00} & \textbf{0.00} & \textbf{0.00} & \textbf{0.00} & \textbf{0.00} & \textbf{0.00} & \textbf{0.00} & \textbf{0.00} & \textbf{0.00} & \textbf{0.00} & \textbf{0.00} & \textbf{0.00} \\
SimpleQA Verified & 3.40 & 2.70 & 4.40 & \textbf{6.70} & 5.20 & 4.10 & 4.80 & 3.20 & 4.50 & 5.80 & 4.50 & 4.00 & 4.00 & 4.80 & 5.30 & 3.50 \\
\midrule
AIME 2026 & 0.00 & 0.00 & \textbf{3.33} & 0.00 & 0.00 & \textbf{3.33} & 0.00 & \textbf{3.33} & 0.00 & 0.00 & 0.00 & 0.00 & 0.00 & 0.00 & 0.00 & 0.00 \\
HMMT Feb 2026 & 0.00 & 0.00 & 0.00 & \textbf{6.06} & 0.00 & 0.00 & 0.00 & 3.03 & 0.00 & 0.00 & 3.03 & 0.00 & 0.00 & 0.00 & 0.00 & 0.00 \\
AIME 2025 & 0.00 & 0.00 & 0.00 & 0.00 & 0.00 & 0.00 & 0.00 & 0.00 & 0.00 & 0.00 & \textbf{3.33} & 0.00 & \textbf{3.33} & 0.00 & 0.00 & 0.00 \\
MATH-500 & 26.60 & 16.20 & 52.20 & 52.20 & 42.40 & \textbf{53.20} & 44.00 & 48.20 & 44.60 & 31.20 & 40.80 & 20.60 & 39.80 & 36.40 & 41.00 & 14.40 \\
Omni-MATH & 6.93 & 5.83 & 13.46 & 12.71 & 10.64 & 13.48 & \textbf{14.79} & 13.30 & 11.16 & 7.16 & 10.19 & 6.14 & 10.14 & 9.01 & 10.46 & 4.56 \\
GSM8K & 40.03 & 22.59 & 83.17 & \textbf{86.20} & 75.06 & 48.29 & 80.89 & 72.33 & 75.36 & 62.24 & 70.96 & 56.41 & 71.72 & 55.72 & 67.02 & 50.64 \\
GSM8K-Platinum & 42.10 & 23.57 & 85.61 & \textbf{88.67} & 77.58 & 48.97 & 81.97 & 75.52 & 76.92 & 64.68 & 74.36 & 57.24 & 73.37 & 58.06 & 71.63 & 52.52 \\
MGSM & 18.76 & 6.25 & 65.71 & 68.84 & \textbf{69.20} & 59.64 & 50.98 & 51.13 & 56.29 & 63.42 & 48.40 & 43.20 & 39.09 & 34.84 & 44.04 & 14.22 \\
\midrule
LiveCodeBench & 7.31 & 0.00 & 10.97 & \textbf{20.40} & 0.63 & 15.94 & 16.59 & 15.44 & 7.26 & 9.14 & 6.17 & 13.68 & 10.40 & 7.77 & 6.40 & 0.46 \\
HumanEval+ & 29.88 & 16.46 & 48.78 & 42.68 & 25.00 & \textbf{53.05} & 44.51 & 51.22 & 46.95 & 30.49 & 35.98 & 25.00 & 34.76 & 24.39 & 43.29 & 7.93 \\
MBPP+ & 30.69 & 35.98 & 61.11 & 57.41 & 54.23 & 58.99 & 72.49 & \textbf{80.42} & 59.52 & 49.74 & 60.05 & 65.34 & 50.26 & 50.53 & 48.15 & 24.34 \\
CRUXEval-O & 23.39 & 12.49 & 40.38 & \textbf{41.67} & 35.97 & 41.22 & 30.32 & 28.66 & 39.73 & 35.69 & 34.26 & 24.29 & 30.70 & 29.18 & 24.00 & 17.59 \\
CRUXEval-I & 23.66 & 20.84 & 44.84 & 45.14 & 38.11 & 43.28 & \textbf{48.96} & 44.00 & 42.58 & 40.23 & 31.04 & 37.13 & 31.31 & 30.32 & 18.06 & 16.20 \\
BigCodeBench & 13.42 & 7.28 & 25.35 & 23.68 & 15.26 & 23.16 & 35.70 & \textbf{41.84} & 0.44 & 21.49 & 0.44 & 31.84 & 16.67 & 14.39 & 11.93 & 5.61 \\
MultiPL-E & 18.25 & 8.80 & \textbf{41.83} & 39.50 & 32.50 & 34.75 & 32.02 & 41.82 & 37.08 & 4.00 & 39.17 & 18.62 & 23.75 & 24.17 & 3.50 & 9.67 \\
SciCode & 0.69 & 0.00 & 0.69 & \textbf{1.03} & \textbf{1.03} & \textbf{1.03} & 0.00 & 0.00 & 0.69 & 0.34 & 0.69 & 0.00 & 0.69 & 0.69 & 0.69 & 0.69 \\
\midrule
RULER (4K) & 84.48 & 75.54 & 83.35 & 60.80 & \textbf{93.28} & 90.31 & 87.17 & 92.41 & 88.48 & 89.55 & 86.68 & 88.99 & 90.66 & 83.01 & 65.06 & 66.87 \\
RULER (8K) & 80.74 & 66.86 & 82.07 & 60.39 & 91.73 & \textbf{91.95} & 62.09 & 75.79 & 87.09 & 88.09 & 80.47 & 76.74 & 82.90 & 77.14 & 58.35 & -- \\
RULER (16K) & 74.57 & 61.53 & 74.97 & 56.22 & \textbf{91.57} & 86.25 & 65.65 & 70.17 & 83.08 & 85.03 & 77.07 & 74.40 & 75.61 & 70.42 & 45.69 & -- \\
RULER (32K) & 73.25 & 56.23 & 66.10 & 52.31 & 80.74 & 77.32 & 59.92 & 61.31 & 78.04 & \textbf{82.68} & 71.89 & 70.90 & 70.95 & 66.38 & 28.82 & -- \\
RULER (64K) & 70.50 & 50.76 & -- & 50.85 & 75.56 & 75.77 & \textbf{89.25} & -- & 77.31 & 80.66 & 65.41 & 70.38 & 65.04 & 71.92 & 12.91 & -- \\
RULER (128K) & 60.14 & 42.38 & -- & \textbf{91.92} & 71.79 & 72.72 & 80.58 & -- & 74.94 & 78.88 & 68.27 & 60.99 & -- & 62.88 & -- & -- \\
LongBench v2 & 25.05 & 23.26 & 32.21 & \textbf{36.58} & 30.62 & 31.41 & 32.01 & 31.61 & 32.01 & 27.83 & 25.25 & 27.24 & 31.01 & 28.23 & 25.84 & 26.44 \\
\midrule
TriviaQA & 60.17 & 38.29 & 64.94 & 65.07 & 70.36 & 63.85 & 50.58 & 50.62 & 61.47 & \textbf{76.06} & 66.51 & 64.93 & 59.77 & 70.11 & 57.24 & 65.50 \\
MMMLU & 31.93 & 28.57 & 58.34 & \textbf{62.12} & 58.27 & 55.07 & 59.91 & 57.26 & 50.16 & 58.70 & 49.46 & 56.91 & 37.89 & 39.77 & 47.16 & 30.43 \\
TruthfulQA MC2 & 37.65 & 37.49 & 41.46 & 42.19 & 38.24 & 37.61 & 52.30 & \textbf{53.38} & 39.14 & 35.38 & 42.82 & 44.98 & 38.69 & 37.21 & 45.62 & 35.98 \\
AGIEval (English) & 30.31 & 18.46 & 54.85 & \textbf{60.15} & 48.18 & 50.58 & 44.83 & 45.92 & 41.47 & 46.81 & 39.33 & 36.45 & 35.96 & 31.38 & 27.02 & 28.22 \\
\midrule
MMLU & 52.53 & 39.62 & 74.59 & \textbf{77.35} & 72.31 & 73.44 & 74.21 & 71.19 & 68.18 & 70.57 & 64.87 & 66.54 & 62.18 & 63.79 & 61.86 & 53.55 \\
MMLU (5-shot) & 54.80 & 43.26 & 76.70 & 71.91 & 76.67 & \textbf{77.46} & 72.97 & 73.07 & 71.73 & 73.89 & 66.19 & 70.65 & 64.64 & 66.07 & 63.27 & 55.11 \\
MMLU (cloze) & 42.71 & 38.80 & 47.99 & 49.18 & 51.71 & 49.74 & 48.41 & 49.32 & 47.22 & \textbf{51.98} & 46.77 & 48.72 & 45.41 & 46.12 & 48.05 & 41.76 \\
MMLU-Redux & 54.77 & 41.24 & 77.30 & \textbf{79.39} & 75.09 & 75.68 & 78.48 & 75.85 & 70.14 & 72.05 & 67.49 & 69.83 & 63.91 & 65.82 & 65.46 & 55.42 \\
MMLU-Redux (5-shot) & 56.11 & 44.41 & 78.37 & 72.44 & 78.14 & \textbf{78.49} & 76.44 & 77.47 & 73.07 & 74.96 & 68.68 & 73.38 & 66.63 & 67.98 & 67.43 & 56.61 \\
BIG-Bench Hard & 36.23 & 31.18 & 75.72 & \textbf{81.03} & 77.16 & 77.71 & 80.45 & 71.37 & 74.06 & 74.72 & 59.76 & 62.74 & 63.75 & 48.90 & 45.14 & 33.76 \\
DROP & 41.28 & 29.91 & 58.96 & 53.67 & 69.17 & 65.80 & 67.88 & 60.14 & \textbf{71.68} & 65.02 & 62.90 & 56.80 & 66.46 & 58.66 & 44.78 & 50.87 \\
HellaSwag & 74.69 & 66.53 & 79.35 & 80.35 & 79.33 & 78.58 & 75.23 & 73.75 & 80.96 & \textbf{81.92} & 78.59 & 78.51 & 75.01 & 75.88 & 73.81 & 78.00 \\
PIQA & 79.43 & 76.33 & 79.98 & 80.58 & 80.79 & 81.72 & 79.22 & 78.07 & 80.74 & \textbf{82.64} & 80.74 & 80.74 & 79.11 & 81.94 & 77.80 & 79.92 \\
ARC-Easy & 74.66 & 65.78 & 79.50 & 79.08 & 86.45 & 82.37 & 76.22 & 76.18 & 81.69 & \textbf{86.83} & 82.20 & 85.19 & 81.10 & 80.51 & 82.28 & 77.40 \\
ARC-Challenge & 45.22 & 38.23 & 55.89 & 56.74 & 61.69 & 59.90 & 55.29 & 51.96 & 54.78 & \textbf{62.12} & 54.10 & 59.73 & 53.33 & 56.06 & 59.47 & 49.49 \\
WinoGrande & 66.54 & 61.64 & 68.43 & 69.69 & 70.40 & \textbf{71.19} & 71.03 & 70.96 & 69.22 & 70.64 & 61.25 & 70.24 & 66.22 & 66.14 & 65.27 & 65.98 \\
OpenBookQA & 42.60 & 40.40 & 42.80 & 45.00 & 45.80 & 45.80 & 42.40 & 41.00 & 45.60 & \textbf{46.80} & 42.80 & 46.00 & 42.20 & 43.80 & 42.20 & 42.60 \\
\end{longtable}}
\end{landscape}

\begin{landscape}
{\tiny\setlength{\tabcolsep}{1.6pt}\renewcommand{\arraystretch}{1.08}
\begin{longtable}{@{}>{\raggedright\arraybackslash}p{2.35cm}rrrrrrrrrrrrrrrrr@{}}
\caption{Base models, Praxis comparison set: Metis-1.6 base checkpoints after context extension against open base models, on one harness with identical prompts and scorers (Appendix~\ref{app:protocols}). Columns are ordered by the short-context mean. Bold marks the best score in each row; -- means the length exceeds the model's context window or what fits the 16\,GB evaluation GPU.}\label{tab:app_base_praxis}\\
\toprule
 &  &  &  &  & \textbf{Nanbeige} & \textbf{Ministral} & \textbf{Granite} & \textbf{Falcon-} &  &  &  &  &  &  & \textbf{Llama} &  &  \\
 & \textbf{Praxis} & \textbf{Ouro} & \textbf{Qwen3.5} & \textbf{Qwen3} & \textbf{4.2} & \textbf{3} & \textbf{4.1} & \textbf{H1} & \textbf{Ouro} & \textbf{Gemma 4} & \textbf{SmolLM3} & \textbf{LFM2.5} & \textbf{Gemma 3} & \textbf{Qwen3.5} & \textbf{3.2} & \textbf{Gemma 4} & \textbf{Qwen3.5} \\
\textbf{Benchmark} &  & \textbf{2.6B} & \textbf{4B} & \textbf{4B} & \textbf{3B} & \textbf{3B} & \textbf{3B} & \textbf{3B} & \textbf{1.4B} & \textbf{E4B} & \textbf{3B} & \textbf{2.6B} & \textbf{4B} & \textbf{2B} & \textbf{3B} & \textbf{E2B} & \textbf{0.8B} \\
\midrule
\endfirsthead
\multicolumn{18}{@{}l}{\emph{(continued)}}\\
\toprule
 &  &  &  &  & \textbf{Nanbeige} & \textbf{Ministral} & \textbf{Granite} & \textbf{Falcon-} &  &  &  &  &  &  & \textbf{Llama} &  &  \\
 & \textbf{Praxis} & \textbf{Ouro} & \textbf{Qwen3.5} & \textbf{Qwen3} & \textbf{4.2} & \textbf{3} & \textbf{4.1} & \textbf{H1} & \textbf{Ouro} & \textbf{Gemma 4} & \textbf{SmolLM3} & \textbf{LFM2.5} & \textbf{Gemma 3} & \textbf{Qwen3.5} & \textbf{3.2} & \textbf{Gemma 4} & \textbf{Qwen3.5} \\
\textbf{Benchmark} &  & \textbf{2.6B} & \textbf{4B} & \textbf{4B} & \textbf{3B} & \textbf{3B} & \textbf{3B} & \textbf{3B} & \textbf{1.4B} & \textbf{E4B} & \textbf{3B} & \textbf{2.6B} & \textbf{4B} & \textbf{2B} & \textbf{3B} & \textbf{E2B} & \textbf{0.8B} \\
\midrule
\endhead
\bottomrule
\endfoot
Mean of 44 (short context) & 24.66 & \textbf{43.15} & 42.93 & 42.88 & 41.23 & 40.59 & 39.99 & 39.81 & 38.78 & 38.55 & 37.17 & 35.93 & 34.37 & 34.17 & 30.92 & 30.53 & 27.98 \\
Mean of all 50 & 27.72 & -- & \textbf{45.57} & -- & -- & -- & -- & 42.42 & -- & 41.54 & -- & -- & 37.20 & 38.72 & -- & 33.57 & 32.64 \\
\midrule
Humanity's Last Exam & 3.89 & 4.23 & 3.34 & 3.99 & 3.48 & 4.17 & 4.36 & 3.94 & 4.41 & 4.22 & 4.40 & \textbf{5.56} & 4.36 & 3.80 & 4.49 & 3.80 & 3.99 \\
HLE-Diamond & 3.43 & 3.32 & 2.88 & \textbf{5.08} & 1.65 & 4.12 & 4.81 & 3.98 & 4.00 & 3.71 & 3.98 & \textbf{5.08} & 4.67 & 3.30 & 3.71 & 4.12 & 3.57 \\
HLE-Verified (Gold) & 4.00 & 4.87 & 4.00 & 4.35 & 4.70 & 4.87 & 5.22 & 3.65 & 4.70 & 5.04 & 4.00 & \textbf{6.26} & 5.74 & 3.65 & 4.17 & 5.22 & 4.17 \\
GPQA & 27.90 & 26.34 & 25.67 & 22.99 & 25.89 & \textbf{28.57} & 24.78 & 23.44 & 25.89 & 26.12 & 25.45 & 22.32 & 26.12 & 24.33 & 25.45 & 26.56 & 27.01 \\
GPQA Diamond & 24.24 & 28.79 & 24.24 & 25.25 & 28.28 & 29.29 & 25.76 & 20.20 & 27.27 & 27.27 & \textbf{29.80} & 25.25 & 28.28 & 23.74 & 26.26 & \textbf{29.80} & 25.76 \\
SuperGPQA & 11.12 & 25.81 & 27.11 & 28.07 & \textbf{32.58} & 25.15 & 21.37 & 24.49 & 22.32 & 24.71 & 23.40 & 22.85 & 17.41 & 21.43 & 19.25 & 16.49 & 15.98 \\
MMLU-Pro & 14.19 & \textbf{48.84} & 48.62 & 45.74 & 45.84 & 40.92 & 33.36 & 40.63 & 41.43 & 38.19 & 33.18 & 37.01 & 27.92 & 36.61 & 24.30 & 24.74 & 25.05 \\
BIG-Bench Extra Hard & 1.66 & 7.01 & 2.17 & 7.92 & \textbf{12.06} & 7.48 & 7.85 & 4.18 & 3.89 & 2.19 & 6.28 & 7.63 & 3.21 & 4.09 & 2.85 & 3.69 & 1.84 \\
ARC-AGI-2 & \textbf{0.00} & \textbf{0.00} & \textbf{0.00} & \textbf{0.00} & \textbf{0.00} & \textbf{0.00} & \textbf{0.00} & \textbf{0.00} & \textbf{0.00} & \textbf{0.00} & \textbf{0.00} & \textbf{0.00} & \textbf{0.00} & \textbf{0.00} & \textbf{0.00} & \textbf{0.00} & \textbf{0.00} \\
SimpleQA Verified & 2.70 & 3.00 & \textbf{4.80} & 3.20 & 4.30 & 3.00 & 2.40 & 2.50 & 1.50 & 4.00 & 4.00 & 3.60 & 3.10 & 2.60 & 4.20 & 2.30 & 1.80 \\
\midrule
AIME 2026 & 0.00 & 0.00 & 0.00 & 3.33 & 0.00 & 0.00 & 0.00 & 0.00 & \textbf{6.67} & 0.00 & 0.00 & 0.00 & 0.00 & 0.00 & 0.00 & 0.00 & 0.00 \\
HMMT Feb 2026 & 0.00 & \textbf{3.03} & 0.00 & \textbf{3.03} & \textbf{3.03} & 0.00 & \textbf{3.03} & \textbf{3.03} & \textbf{3.03} & 0.00 & 0.00 & 0.00 & 0.00 & 0.00 & 0.00 & 0.00 & 0.00 \\
AIME 2025 & 0.00 & 0.00 & 0.00 & 0.00 & 0.00 & 0.00 & 0.00 & 0.00 & \textbf{3.33} & 0.00 & 0.00 & 0.00 & 0.00 & 0.00 & 0.00 & 0.00 & 0.00 \\
MATH-500 & 16.20 & 40.80 & 44.00 & \textbf{48.20} & 20.60 & 36.00 & 39.00 & 43.40 & 1.80 & 20.60 & 38.80 & 41.20 & 24.00 & 32.80 & 6.00 & 9.20 & 20.40 \\
Omni-MATH & 5.83 & 15.27 & 14.79 & 13.30 & 12.08 & 11.68 & 10.48 & 12.83 & \textbf{21.41} & 6.14 & 10.34 & 12.40 & 6.50 & 9.24 & 3.50 & 3.75 & 7.25 \\
GSM8K & 22.59 & \textbf{82.34} & 80.89 & 72.33 & 69.14 & 71.11 & 72.48 & 68.99 & 78.47 & 56.41 & 67.93 & 73.16 & 36.85 & 59.06 & 26.38 & 21.83 & 38.59 \\
GSM8K-Platinum & 23.57 & \textbf{85.28} & 81.97 & 75.52 & 79.57 & 71.63 & 72.87 & 71.88 & 79.40 & 57.24 & 69.48 & 77.50 & 39.62 & 63.03 & 27.05 & 23.24 & 39.70 \\
MGSM & 6.25 & 37.67 & 50.98 & \textbf{51.13} & 45.35 & 44.44 & 40.84 & 40.73 & 38.22 & 43.20 & 35.31 & 44.11 & 29.89 & 35.71 & 16.04 & 16.80 & 22.95 \\
\midrule
LiveCodeBench & 0.00 & 21.59 & 16.59 & 15.44 & \textbf{24.73} & 16.04 & 16.32 & 16.43 & 17.42 & 13.68 & 12.97 & 2.09 & 7.42 & 4.56 & 7.86 & 4.73 & 1.15 \\
HumanEval+ & 16.46 & 34.76 & 44.51 & 51.22 & 44.51 & 35.98 & 43.29 & \textbf{55.49} & 33.54 & 25.00 & 31.10 & 31.10 & 28.66 & 29.88 & 25.00 & 21.34 & 19.51 \\
MBPP+ & 35.98 & \textbf{82.80} & 72.49 & 80.42 & 78.04 & 70.63 & 76.19 & 75.66 & 76.98 & 65.34 & 66.14 & 47.88 & 62.43 & 49.47 & 55.29 & 50.79 & 36.24 \\
CRUXEval-O & 12.49 & \textbf{33.74} & 30.32 & 28.66 & 31.52 & 25.42 & 27.06 & 27.97 & 28.09 & 24.29 & 22.04 & 18.14 & 21.24 & 20.45 & 18.52 & 16.05 & 13.01 \\
CRUXEval-I & 20.84 & 50.59 & 48.96 & 44.00 & \textbf{50.94} & 40.55 & 44.01 & 40.09 & 41.57 & 37.13 & 35.73 & 27.24 & 32.82 & 36.23 & 28.91 & 26.85 & 30.91 \\
BigCodeBench & 7.28 & 40.44 & 35.70 & 41.84 & 9.12 & 33.95 & \textbf{43.16} & 32.72 & 28.77 & 31.84 & 28.77 & 22.19 & 24.39 & 17.54 & 19.39 & 14.56 & 7.63 \\
MultiPL-E & 8.80 & 30.41 & 32.02 & 41.82 & \textbf{44.61} & 27.66 & 42.39 & 39.75 & 25.22 & 18.62 & 21.72 & 12.82 & 21.52 & 17.48 & 17.11 & 13.17 & 8.99 \\
SciCode & 0.00 & 0.00 & 0.00 & 0.00 & 0.00 & 0.00 & \textbf{1.54} & 0.00 & 0.00 & 0.00 & 0.00 & 0.00 & 0.00 & 0.00 & 0.00 & 0.00 & 0.00 \\
\midrule
RULER (4K) & 75.54 & 91.69 & 87.17 & 92.41 & \textbf{92.91} & 88.51 & 86.55 & 91.24 & 86.68 & 88.99 & 87.92 & 84.61 & 86.78 & 91.16 & 87.60 & 89.68 & 86.49 \\
RULER (8K) & 66.86 & -- & 62.09 & 75.79 & 75.87 & 85.94 & 82.55 & 83.05 & 72.84 & 76.74 & 80.96 & 79.35 & 74.51 & \textbf{87.45} & 81.84 & 66.55 & 81.81 \\
RULER (16K) & 61.53 & -- & 65.65 & 70.17 & \textbf{85.91} & 82.05 & 76.95 & 75.47 & -- & 74.40 & 78.29 & 72.59 & 69.07 & 83.84 & 78.85 & 68.25 & 76.20 \\
RULER (32K) & 56.23 & -- & 59.92 & 61.31 & -- & 76.78 & 75.35 & 72.00 & -- & 70.90 & 71.80 & 63.58 & 68.16 & \textbf{80.63} & 75.73 & 66.54 & 72.65 \\
RULER (64K) & 50.76 & -- & \textbf{89.25} & -- & -- & -- & 69.53 & 58.59 & -- & 70.38 & 73.27 & 54.37 & 57.16 & 80.51 & 74.24 & 59.73 & 75.45 \\
RULER (128K) & 42.38 & -- & \textbf{80.58} & -- & -- & -- & -- & 53.33 & -- & 60.99 & -- & -- & 51.83 & 69.15 & -- & 47.69 & 68.14 \\
LongBench v2 & 23.26 & 28.83 & 32.01 & 31.61 & \textbf{37.18} & 25.65 & 26.24 & 27.24 & 23.26 & 27.24 & 28.63 & 28.43 & 26.84 & 30.82 & 26.84 & 26.24 & 26.84 \\
\midrule
TriviaQA & 38.29 & 53.06 & 50.58 & 50.62 & 27.89 & 56.24 & 50.39 & 40.73 & 43.92 & \textbf{64.93} & 58.13 & 43.40 & 60.59 & 32.71 & 59.54 & 53.46 & 19.54 \\
MMMLU & 28.57 & 48.33 & \textbf{59.91} & 57.26 & 54.66 & 55.59 & 46.68 & 46.76 & 40.99 & 56.91 & 45.87 & 48.54 & 47.26 & 46.93 & 40.93 & 41.91 & 40.17 \\
TruthfulQA MC2 & 37.49 & 49.40 & 52.30 & \textbf{53.38} & 50.36 & 44.96 & 47.68 & 41.85 & 46.72 & 44.98 & 45.98 & 43.65 & 39.67 & 45.86 & 39.30 & 38.18 & 44.93 \\
AGIEval (English) & 18.46 & 42.89 & 44.83 & \textbf{45.92} & 43.74 & 35.72 & 35.36 & 40.60 & 34.87 & 36.45 & 29.60 & 33.80 & 31.67 & 33.13 & 24.12 & 28.84 & 23.73 \\
\midrule
MMLU & 39.62 & 73.74 & 74.21 & 71.19 & \textbf{75.76} & 68.67 & 61.56 & 65.55 & 68.31 & 66.54 & 59.71 & 64.09 & 57.89 & 61.52 & 54.94 & 53.95 & 52.90 \\
MMLU (5-shot) & 43.26 & 74.62 & 72.97 & 73.07 & \textbf{75.45} & 71.19 & 64.40 & 68.20 & 67.42 & 70.65 & 61.44 & 63.81 & 59.61 & 53.88 & 56.59 & 57.66 & 48.11 \\
MMLU (cloze) & 38.80 & 48.82 & 48.41 & 49.32 & \textbf{51.99} & 49.26 & 45.39 & 44.84 & 46.67 & 48.72 & 44.64 & 38.34 & 46.35 & 38.22 & 42.54 & 44.08 & 35.28 \\
MMLU-Redux & 41.24 & 76.29 & 78.48 & 75.85 & \textbf{78.99} & 72.57 & 65.44 & 69.79 & 71.09 & 69.83 & 62.95 & 66.47 & 61.54 & 66.32 & 57.69 & 57.04 & 56.59 \\
MMLU-Redux (5-shot) & 44.41 & 77.47 & 76.44 & 77.47 & \textbf{78.80} & 74.86 & 68.01 & 72.27 & 70.38 & 73.38 & 64.09 & 66.10 & 62.96 & 56.85 & 59.19 & 60.83 & 50.49 \\
BIG-Bench Hard & 31.18 & 80.34 & \textbf{80.45} & 71.37 & 80.03 & 71.17 & 63.95 & 70.86 & 71.08 & 62.74 & 59.75 & 51.82 & 51.60 & 63.43 & 47.29 & 42.88 & 46.81 \\
DROP & 29.91 & 62.55 & \textbf{67.88} & 60.14 & 26.16 & 59.11 & 65.62 & 54.06 & 51.05 & 56.80 & 44.52 & 47.10 & 45.64 & 54.22 & 43.91 & 41.16 & 39.77 \\
HellaSwag & 66.53 & 76.43 & 75.23 & 73.75 & 75.88 & 73.76 & 74.67 & 71.22 & 71.61 & \textbf{78.51} & 75.13 & 66.18 & 75.81 & 65.94 & 74.22 & 73.71 & 54.92 \\
PIQA & 76.33 & 80.03 & 79.22 & 78.07 & 80.14 & 79.27 & 79.33 & 77.69 & 78.51 & \textbf{80.74} & 78.67 & 75.84 & 80.09 & 74.81 & 77.91 & 78.78 & 71.87 \\
ARC-Easy & 65.78 & 80.60 & 76.22 & 76.18 & 79.46 & 83.00 & 79.25 & 77.06 & 76.89 & \textbf{85.19} & 76.85 & 80.89 & 81.57 & 71.42 & 72.14 & 80.05 & 67.42 \\
ARC-Challenge & 38.23 & 58.70 & 55.29 & 51.96 & 55.63 & 58.02 & 49.83 & 54.10 & 52.56 & \textbf{59.73} & 52.13 & 51.62 & 54.52 & 45.82 & 46.25 & 52.30 & 40.44 \\
WinoGrande & 61.64 & 68.75 & \textbf{71.03} & 70.96 & 70.09 & 69.69 & 69.06 & 66.30 & 66.54 & 70.24 & 68.27 & 65.98 & 69.69 & 64.64 & 70.09 & 65.43 & 60.06 \\
OpenBookQA & 40.40 & 43.80 & 42.40 & 41.00 & 44.00 & 41.60 & 43.80 & 42.40 & 41.60 & \textbf{46.00} & 44.80 & 43.20 & 43.00 & 37.80 & 40.60 & 44.40 & 36.00 \\
\end{longtable}}
\end{landscape}

\section{Post-trained evaluation}
\label{app:post}
Tables~\ref{tab:app_post_logos} and~\ref{tab:app_post_praxis} give every post-trained benchmark for the Logos and Praxis comparison sets.

\begin{landscape}
{\tiny\setlength{\tabcolsep}{1.6pt}\renewcommand{\arraystretch}{1.08}
\begin{longtable}{@{}>{\raggedright\arraybackslash}p{2.35cm}rrrrrrrrrrrrrrrrrrr@{}}
\caption{Post-trained Metis-1.6 Logos and Praxis (think and direct modes) against the post-trained releases of open models in the Logos comparison set, on the harness of Section~\ref{sec:evalprotocol}. Bold marks the best score in each row.}\label{tab:app_post_logos}\\
\toprule
 &  &  &  &  &  &  &  & \textbf{Granite} &  &  &  &  &  &  & \textbf{Granite} & \textbf{Ministral} &  & \textbf{Falcon-} &  \\
 & \textbf{Logos} & \textbf{Logos} & \textbf{Praxis} & \textbf{Praxis} & \textbf{Qwen3.5} & \textbf{Qwen3} & \textbf{LFM2.5} & \textbf{4.0} & \textbf{Trinity} & \textbf{OLMoE} & \textbf{Qwen3.5} & \textbf{Qwen3} & \textbf{Gemma 4} & \textbf{Gemma 3} & \textbf{4.1} & \textbf{3} & \textbf{Olmo 3} & \textbf{H1} & \textbf{Gemma 4} \\
\textbf{Benchmark} & \textbf{think} & \textbf{direct} & \textbf{think} & \textbf{direct} & \textbf{4B} & \textbf{4B} & \textbf{8B-A1B} & \textbf{H Tiny} & \textbf{Nano} & \textbf{1B-7B} & \textbf{9B} & \textbf{8B} & \textbf{12B} & \textbf{12B} & \textbf{8B} & \textbf{8B} & \textbf{7B} & \textbf{7B} & \textbf{E4B} \\
\midrule
\endfirsthead
\multicolumn{20}{@{}l}{\emph{(continued)}}\\
\toprule
 &  &  &  &  &  &  &  & \textbf{Granite} &  &  &  &  &  &  & \textbf{Granite} & \textbf{Ministral} &  & \textbf{Falcon-} &  \\
 & \textbf{Logos} & \textbf{Logos} & \textbf{Praxis} & \textbf{Praxis} & \textbf{Qwen3.5} & \textbf{Qwen3} & \textbf{LFM2.5} & \textbf{4.0} & \textbf{Trinity} & \textbf{OLMoE} & \textbf{Qwen3.5} & \textbf{Qwen3} & \textbf{Gemma 4} & \textbf{Gemma 3} & \textbf{4.1} & \textbf{3} & \textbf{Olmo 3} & \textbf{H1} & \textbf{Gemma 4} \\
\textbf{Benchmark} & \textbf{think} & \textbf{direct} & \textbf{think} & \textbf{direct} & \textbf{4B} & \textbf{4B} & \textbf{8B-A1B} & \textbf{H Tiny} & \textbf{Nano} & \textbf{1B-7B} & \textbf{9B} & \textbf{8B} & \textbf{12B} & \textbf{12B} & \textbf{8B} & \textbf{8B} & \textbf{7B} & \textbf{7B} & \textbf{E4B} \\
\midrule
\endhead
\bottomrule
\endfoot
Parameters & 10.23B &  & 2.94B &  & 4.21B & 4.02B & 8.47B & 6.94B & 6.12B & 6.92B & 8.95B & 8.19B & 11.94B & 12.19B & 8.38B & 8.49B & 7.30B & 7.59B & 7.52B \\
Active per pass & 1.18B &  & 0.464B &  & 3.57B & 3.63B & 1.42B & 1.31B & 0.614B & 1.08B & 6.92B & 6.95B & 10.94B & 10.78B & 7.97B & 7.42B & 6.48B & 6.79B & 4.03B \\
Training tokens & 1.02T &  & 1.02T &  & undisclosed & 36T & 38T & 22.5T & 10T & 5.13T & undisclosed & 36T & undisclosed & 12T & 14.9T & 1-3T & 6.08T & 12T & undisclosed \\
\midrule
AIME 2025 & 67.50 & 0.00 & 27.08 & 0.00 & 64.17 & 53.75 & 20.83 & 7.92 & 1.67 & 0.00 & 71.25 & 52.08 & \textbf{79.17} & 18.75 & 16.67 & 20.83 & 24.58 & 10.83 & 41.25 \\
AIME 2026 & 67.50 & 6.67 & 31.25 & 0.00 & 73.33 & 63.75 & 23.33 & 3.75 & 2.92 & 0.00 & 78.75 & 61.25 & \textbf{79.58} & 16.25 & 13.33 & 17.50 & 29.17 & 7.08 & 46.67 \\
HMMT Feb 2025 & \textbf{62.92} & 13.33 & 16.25 & 0.00 & 36.25 & 35.83 & 10.42 & 2.08 & 1.25 & 0.00 & 45.42 & 33.75 & 61.67 & 9.17 & 10.00 & 12.50 & 10.42 & 3.75 & 25.00 \\
HMMT Nov 2025 & 57.92 & 10.00 & 24.17 & 0.00 & 39.58 & 44.17 & 7.08 & 3.33 & 0.42 & 1.67 & 46.67 & 43.75 & \textbf{69.58} & 8.33 & 9.58 & 12.92 & 12.08 & 3.33 & 38.33 \\
HMMT Feb 2026 & \textbf{63.64} & 0.00 & 22.73 & 0.00 & 35.61 & 34.85 & 14.39 & 1.52 & 0.76 & 0.38 & 39.39 & 41.67 & 53.03 & 10.23 & 12.12 & 9.85 & 18.94 & 2.27 & 31.82 \\
AMC 2023 & 92.50 & 60.00 & 60.31 & 22.50 & 84.38 & 93.44 & 61.25 & 41.56 & 28.44 & 7.50 & 93.12 & 90.94 & \textbf{97.81} & 64.38 & 65.00 & 58.12 & 67.19 & 47.50 & 86.25 \\
Brumo 2025 & 72.08 & 33.33 & 41.25 & 6.67 & 49.58 & 59.58 & 29.17 & 11.67 & 9.58 & 2.08 & 62.08 & 67.08 & \textbf{76.25} & 26.25 & 23.75 & 22.50 & 29.58 & 18.75 & 54.17 \\
IMO-AnswerBench & \textbf{57.00} & 19.50 & 16.00 & 2.50 & 33.50 & 37.00 & 4.00 & 4.00 & 2.50 & 2.00 & 38.50 & 37.00 & 44.50 & 12.50 & 10.00 & 12.50 & 12.00 & 7.00 & 29.50 \\
Apex 2025 & 3.12 & \textbf{16.67} & 1.04 & 0.00 & 0.00 & 0.00 & 0.00 & 0.00 & 0.00 & 0.00 & 1.04 & 0.00 & 1.04 & 0.00 & 0.00 & 0.00 & 0.00 & 1.04 & 2.08 \\
MATH-500 & 89.40 & 70.60 & 78.80 & 51.80 & 83.60 & 88.00 & 71.40 & 67.00 & 56.60 & 20.60 & 88.40 & 88.20 & \textbf{94.20} & 83.20 & 78.40 & 71.60 & 76.60 & 56.40 & 90.40 \\
GSM8K & 90.60 & 81.00 & 76.80 & 63.60 & 96.60 & 95.20 & 90.40 & 88.20 & 80.20 & 53.80 & 96.40 & 96.20 & \textbf{97.80} & 90.20 & 94.40 & 85.60 & 83.00 & 87.20 & 92.80 \\
\midrule
GPQA Diamond & 55.05 & 34.34 & 27.78 & 22.22 & 76.26 & 52.02 & 28.28 & 30.30 & 23.74 & 23.74 & \textbf{81.31} & 53.54 & 74.24 & 39.39 & 40.40 & 40.91 & 37.37 & 44.95 & 61.62 \\
GPQA Main & 50.45 & 33.03 & 31.03 & 25.22 & 70.76 & 47.32 & 39.28 & 28.35 & 24.55 & 23.89 & \textbf{72.10} & 52.46 & 66.07 & 37.72 & 39.51 & 41.74 & 34.37 & 41.96 & 55.14 \\
GPQA Extended & 41.58 & 30.22 & 23.26 & 17.58 & \textbf{65.57} & 48.35 & 43.22 & 29.30 & 23.26 & 21.79 & 65.20 & 49.82 & 60.81 & 34.43 & 38.83 & 39.56 & 35.53 & 42.31 & 50.00 \\
Humanity's Last Exam (text) & 4.38 & 6.91 & 5.85 & 6.50 & 6.90 & 3.95 & 3.70 & 5.55 & 4.41 & 2.71 & 7.82 & 3.98 & \textbf{9.39} & 4.01 & 4.31 & 2.07 & 4.59 & 3.48 & 4.47 \\
MMLU-Pro & 59.14 & 44.00 & 34.57 & 32.29 & 66.86 & 64.86 & 49.86 & 33.43 & 33.43 & 14.29 & \textbf{71.43} & 68.86 & 61.57 & 60.57 & 59.14 & 59.29 & 47.00 & 61.71 & 68.00 \\
MMLU & 68.00 & 61.57 & 53.43 & 50.43 & 73.57 & 78.14 & 65.57 & 59.14 & 62.57 & 36.29 & 74.29 & \textbf{81.43} & 76.29 & 76.71 & 77.29 & 75.71 & 62.86 & 77.00 & 79.57 \\
\midrule
LiveCodeBench v6 & 24.57 & 13.71 & 12.57 & 11.43 & 17.14 & 34.29 & 18.29 & 20.57 & 10.29 & 2.86 & 20.57 & 32.00 & 39.43 & 39.43 & 22.29 & 33.71 & 0.57 & 22.86 & \textbf{42.29} \\
HumanEval & 76.22 & 72.56 & 51.83 & 62.80 & 87.80 & 92.68 & 79.27 & 81.71 & 59.76 & 30.49 & 91.46 & 88.41 & \textbf{96.95} & 84.76 & 85.37 & 87.20 & 82.93 & 82.93 & 95.73 \\
HumanEval+ & 62.20 & 59.76 & 42.07 & 50.00 & 74.39 & 75.61 & 64.63 & 64.02 & 45.12 & 22.56 & 77.44 & 74.39 & \textbf{83.54} & 65.24 & 70.12 & 70.73 & 67.68 & 65.85 & 79.88 \\
BFCL v4 & 58.60 & 56.40 & 52.40 & 50.40 & 48.20 & 58.20 & 36.60 & 55.80 & 30.20 & 11.00 & 54.60 & 59.80 & \textbf{67.20} & 51.80 & 60.00 & 56.60 & 44.80 & 36.00 & 66.20 \\
\midrule
IFEval (strict) & 71.07 & 69.92 & 64.56 & 66.48 & 89.66 & 86.21 & 84.67 & 78.74 & 53.26 & 63.98 & 92.15 & 86.97 & \textbf{92.91} & 84.87 & 84.87 & 61.69 & 85.25 & 86.02 & 77.78 \\
IFBench & 39.00 & 28.67 & 34.33 & 24.33 & 60.33 & 34.00 & 52.00 & 20.33 & 16.67 & 21.33 & 63.67 & 35.00 & \textbf{74.00} & 36.67 & 36.67 & 30.67 & 34.33 & 30.00 & 41.00 \\
Multi-IF (English) & 79.67 & 76.00 & 65.00 & 66.00 & 90.33 & 85.67 & 85.33 & 77.33 & 52.67 & 63.00 & 92.67 & 88.33 & \textbf{94.67} & 85.33 & 83.67 & 63.67 & 82.67 & 85.00 & 82.00 \\
\midrule
RULER 4K & 68.33 & 84.48 & 48.67 & 75.54 & 88.00 & 90.33 & 63.67 & 81.67 & 78.33 & 59.00 & 88.33 & 92.33 & 90.00 & 94.00 & 86.33 & 92.00 & 73.00 & 93.00 & \textbf{95.00} \\
RULER 32K & 51.67 & 73.25 & 25.00 & 56.23 & \textbf{97.22} & 85.00 & 41.11 & 61.67 & 50.56 & -- & 90.56 & 87.78 & 88.33 & 71.11 & 76.67 & 87.78 & 53.33 & 79.44 & 86.11 \\
RULER 128K & 34.17 & 60.14 & 21.67 & 42.38 & \textbf{94.17} & -- & 5.83 & 47.50 & 16.67 & -- & 85.00 & -- & 83.33 & 45.00 & 50.83 & 78.33 & -- & 45.83 & 81.67 \\
LongBench v2 & 13.20 & 12.80 & 22.40 & 15.60 & 35.20 & 31.60 & 15.60 & 14.40 & 13.60 & 10.40 & 34.80 & 33.60 & \textbf{40.00} & 29.60 & 24.40 & 31.20 & 33.60 & 20.40 & 35.60 \\
MRCR (2/4/8-needle) & 5.83 & 4.17 & 0.83 & 1.67 & \textbf{35.00} & 10.00 & 0.83 & 5.83 & 0.83 & 0.00 & 34.17 & 9.17 & 15.83 & 7.50 & 8.33 & 5.83 & 0.83 & 7.50 & 7.50 \\
AA-LCR v1.1 & 6.00 & 4.00 & 1.00 & 3.00 & 4.00 & 2.00 & 0.00 & 1.00 & 4.00 & 3.00 & \textbf{9.00} & 2.00 & 5.00 & 5.00 & 4.00 & 6.00 & 3.00 & 2.00 & 3.00 \\
\midrule
SimpleQA Verified & 3.78 & 4.00 & 2.40 & 1.40 & 4.80 & 3.80 & \textbf{8.00} & 4.60 & 3.40 & 3.80 & 7.00 & 6.40 & 4.60 & 4.80 & 5.00 & 5.00 & 3.20 & 4.00 & 3.20 \\
FRAMES & 10.75 & 10.75 & 3.75 & 4.50 & 10.75 & 10.75 & \textbf{17.50} & 10.25 & 9.00 & 4.75 & \textbf{17.50} & 16.00 & 17.25 & 14.50 & \textbf{17.50} & 16.00 & 7.00 & 8.75 & 10.50 \\
AA-Omniscience & 13.50 & 10.83 & 4.67 & 5.83 & 13.67 & 10.33 & \textbf{20.83} & 12.17 & 9.83 & 6.67 & 15.50 & 12.17 & 15.50 & 12.00 & 15.00 & 15.17 & 6.50 & 12.83 & 13.67 \\
BrowseComp (closed-book) & 0.42 & 0.00 & 0.42 & 0.33 & 0.33 & 0.67 & \textbf{2.00} & 0.00 & 0.67 & 0.67 & 0.00 & 1.67 & 0.33 & 0.67 & 0.67 & 1.33 & 0.00 & 0.67 & 0.33 \\
\midrule
ARC-Challenge & 94.00 & 86.40 & 79.20 & 76.40 & \textbf{97.60} & 96.60 & 93.60 & 81.00 & 88.20 & 48.20 & 96.60 & \textbf{97.60} & 96.20 & 96.00 & 94.00 & 94.40 & 83.00 & 95.20 & 84.80 \\
ARC-Easy & 93.60 & 91.60 & 81.60 & 79.80 & 97.40 & 97.80 & 93.80 & 82.60 & 95.00 & 50.00 & 97.40 & \textbf{98.40} & 97.00 & 97.60 & 97.80 & 96.60 & 86.00 & 96.80 & 89.20 \\
PIQA & 86.17 & 77.27 & 66.27 & 64.35 & 90.43 & 87.32 & 83.73 & 73.92 & 81.82 & 53.11 & \textbf{91.39} & 89.71 & 91.15 & 88.04 & 89.00 & 83.49 & 73.92 & 85.65 & 85.17 \\
WinoGrande & 71.01 & 63.98 & 54.66 & 55.07 & 88.82 & 89.65 & 77.23 & 69.77 & 72.26 & 58.80 & 89.44 & \textbf{92.55} & 90.06 & 83.23 & 86.75 & 75.78 & 68.32 & 83.85 & 84.89 \\
HellaSwag & 53.00 & 52.40 & 35.80 & 39.00 & 77.60 & 76.20 & 48.20 & 57.60 & 59.60 & 41.00 & \textbf{84.20} & 82.80 & 79.00 & 81.60 & 82.20 & 69.80 & 55.40 & 78.20 & 76.00 \\
\midrule
XSTest (benign compliance) & 99.20 & 98.80 & 99.20 & 96.80 & 98.80 & \textbf{100.00} & 99.20 & 92.80 & 99.60 & 92.40 & 99.20 & \textbf{100.00} & \textbf{100.00} & 98.40 & 94.40 & 96.80 & 91.60 & 98.80 & 96.40 \\
StrongREJECT (refusal) & 49.84 & 76.04 & 44.41 & 70.61 & 92.01 & 43.45 & 2.88 & 76.36 & 15.02 & 90.10 & 94.89 & 46.33 & \textbf{97.44} & 46.01 & 72.84 & 23.00 & 88.50 & 67.09 & 94.89 \\
HarmBench (refusal) & 34.33 & 41.00 & 25.00 & 42.67 & 81.67 & 29.67 & 10.67 & 57.00 & 8.67 & 60.00 & \textbf{91.00} & 36.00 & 71.33 & 29.33 & 55.00 & 15.33 & 77.67 & 36.33 & 61.33 \\
Mean of 40 benchmarks (think) & 52.93 &  & 36.62 &  & 58.51 & 53.68 & 40.52 & 38.41 & 31.60 & 25.20 & 61.90 & 55.09 & \textbf{64.06} & 45.71 & 46.72 & 42.88 & 43.38 & 42.91 & 54.59 \\
\end{longtable}}
{\footnotesize The mean leaves out RULER 32K and 128K, which some models cannot run, and uses think mode for Metis. ScaleAI MultiChallenge, listed on the model cards, is omitted. Source: the benchmark table of \texttt{Lernex/Metis-1.6-Logos} at revision \texttt{ca062e2e}.\par}
\end{landscape}

\begin{landscape}
{\tiny\setlength{\tabcolsep}{1.6pt}\renewcommand{\arraystretch}{1.08}
\begin{longtable}{@{}>{\raggedright\arraybackslash}p{2.35cm}rrrrrrrrrrrrrrrrrrrr@{}}
\caption{Post-trained Metis-1.6 Praxis and Logos (think and direct modes) against the post-trained releases of open models in the Praxis comparison set, on the harness of Section~\ref{sec:evalprotocol}. Bold marks the best score in each row.}\label{tab:app_post_praxis}\\
\toprule
 &  &  &  &  &  &  & \textbf{Nanbeige} & \textbf{Ministral} & \textbf{Granite} & \textbf{Falcon-} &  &  &  &  &  &  & \textbf{Llama} &  &  &  \\
 & \textbf{Praxis} & \textbf{Praxis} & \textbf{Logos} & \textbf{Logos} & \textbf{Qwen3.5} & \textbf{Qwen3} & \textbf{4.2} & \textbf{3} & \textbf{4.1} & \textbf{H1} & \textbf{Ouro} & \textbf{Gemma 4} & \textbf{SmolLM3} & \textbf{LFM2.5} & \textbf{Gemma 3} & \textbf{Qwen3.5} & \textbf{3.2} & \textbf{Gemma 4} & \textbf{Qwen3.5} & \textbf{Ouro} \\
\textbf{Benchmark} & \textbf{think} & \textbf{direct} & \textbf{think} & \textbf{direct} & \textbf{4B} & \textbf{4B} & \textbf{3B} & \textbf{3B} & \textbf{3B} & \textbf{3B} & \textbf{1.4B} & \textbf{E4B} & \textbf{3B} & \textbf{2.6B} & \textbf{4B} & \textbf{2B} & \textbf{3B} & \textbf{E2B} & \textbf{0.8B} & \textbf{2.6B} \\
\midrule
\endfirsthead
\multicolumn{21}{@{}l}{\emph{(continued)}}\\
\toprule
 &  &  &  &  &  &  & \textbf{Nanbeige} & \textbf{Ministral} & \textbf{Granite} & \textbf{Falcon-} &  &  &  &  &  &  & \textbf{Llama} &  &  &  \\
 & \textbf{Praxis} & \textbf{Praxis} & \textbf{Logos} & \textbf{Logos} & \textbf{Qwen3.5} & \textbf{Qwen3} & \textbf{4.2} & \textbf{3} & \textbf{4.1} & \textbf{H1} & \textbf{Ouro} & \textbf{Gemma 4} & \textbf{SmolLM3} & \textbf{LFM2.5} & \textbf{Gemma 3} & \textbf{Qwen3.5} & \textbf{3.2} & \textbf{Gemma 4} & \textbf{Qwen3.5} & \textbf{Ouro} \\
\textbf{Benchmark} & \textbf{think} & \textbf{direct} & \textbf{think} & \textbf{direct} & \textbf{4B} & \textbf{4B} & \textbf{3B} & \textbf{3B} & \textbf{3B} & \textbf{3B} & \textbf{1.4B} & \textbf{E4B} & \textbf{3B} & \textbf{2.6B} & \textbf{4B} & \textbf{2B} & \textbf{3B} & \textbf{E2B} & \textbf{0.8B} & \textbf{2.6B} \\
\midrule
\endhead
\bottomrule
\endfoot
Parameters & 2.94B &  & 10.23B &  & 4.21B & 4.02B & 4.17B & 3.43B & 3.40B & 3.15B & 1.43B & 7.52B & 3.08B & 2.70B & 3.88B & 1.88B & 3.21B & 4.65B & 0.752B & 2.67B \\
Active per pass & 0.464B &  & 1.18B &  & 3.57B & 3.63B & 3.15B & 3.03B & 3.15B & 2.81B & 1.23B & 4.03B & 2.81B & 2.44B & 3.21B & 1.37B & 2.82B & 1.90B & 0.498B & 2.47B \\
Training tokens & 1.02T &  & 1.02T &  & undisclosed & 36T & 28T & 1-3T & 14.9T & 2.5T & 7.7T & undisclosed & 11.2T & 34T & 4T & undisclosed & 9T & undisclosed & undisclosed & 7.7T \\
\midrule
AIME 2025 & 27.08 & 0.00 & 67.50 & 0.00 & 64.17 & 53.75 & \textbf{68.75} & 15.00 & 5.83 & 14.17 & 37.08 & 41.25 & 34.58 & 42.08 & 11.25 & 18.75 & 0.00 & 29.58 & 0.42 & 54.58 \\
AIME 2026 & 31.25 & 0.00 & 67.50 & 6.67 & 73.33 & 63.75 & \textbf{73.75} & 9.58 & 4.58 & 7.08 & 44.17 & 46.67 & 40.83 & 40.83 & 10.00 & 14.17 & 0.42 & 40.83 & 0.00 & 60.83 \\
HMMT Feb 2025 & 16.25 & 0.00 & 62.92 & 13.33 & 36.25 & 35.83 & \textbf{63.75} & 5.83 & 0.42 & 5.00 & 27.92 & 25.00 & 24.17 & 24.58 & 1.25 & 10.83 & 0.42 & 20.83 & 0.00 & 30.42 \\
HMMT Nov 2025 & 24.17 & 0.00 & 57.92 & 10.00 & 39.58 & 44.17 & \textbf{60.42} & 4.58 & 1.25 & 5.00 & 32.50 & 38.33 & 25.83 & 26.67 & 4.58 & 5.00 & 1.25 & 16.25 & 0.00 & 40.83 \\
HMMT Feb 2026 & 22.73 & 0.00 & \textbf{63.64} & 0.00 & 35.61 & 34.85 & 53.79 & 6.06 & 0.76 & 6.06 & 28.41 & 31.82 & 26.89 & 22.35 & 5.30 & 12.50 & 0.76 & 18.94 & 0.00 & 32.20 \\
AMC 2023 & 60.31 & 22.50 & 92.50 & 60.00 & 84.38 & 93.44 & \textbf{95.00} & 42.81 & 34.06 & 55.00 & 86.88 & 86.25 & 85.94 & 71.56 & 44.69 & 55.00 & 15.31 & 76.56 & 25.00 & 93.44 \\
Brumo 2025 & 41.25 & 6.67 & 72.08 & 33.33 & 49.58 & 59.58 & \textbf{75.00} & 15.42 & 9.17 & 20.00 & 54.58 & 54.17 & 54.17 & 40.00 & 20.00 & 24.58 & 0.42 & 43.75 & 4.58 & 65.83 \\
IMO-AnswerBench & 16.00 & 2.50 & \textbf{57.00} & 19.50 & 33.50 & 37.00 & 41.00 & 9.50 & 7.00 & 8.50 & 29.00 & 29.50 & 27.50 & 23.00 & 6.50 & 11.50 & 0.00 & 20.50 & 5.00 & 38.00 \\
Apex 2025 & 1.04 & 0.00 & 3.12 & \textbf{16.67} & 0.00 & 0.00 & 0.00 & 1.04 & 0.00 & 0.00 & 2.08 & 2.08 & 0.00 & 1.04 & 0.00 & 0.00 & 0.00 & 1.04 & 3.12 & 0.00 \\
MATH-500 & 78.80 & 51.80 & 89.40 & 70.60 & 83.60 & 88.00 & \textbf{91.80} & 64.60 & 61.60 & 76.40 & 88.60 & 90.40 & 85.20 & 84.60 & 72.60 & 72.80 & 32.40 & 86.00 & 43.60 & 90.20 \\
GSM8K & 76.80 & 63.60 & 90.60 & 81.00 & \textbf{96.60} & 95.20 & 95.20 & 87.20 & 82.00 & 88.00 & 94.60 & 92.80 & 92.40 & 90.80 & 90.60 & 87.20 & 74.00 & 90.80 & 64.20 & 96.00 \\
\midrule
GPQA Diamond & 27.78 & 22.22 & 55.05 & 34.34 & 76.26 & 52.02 & \textbf{76.77} & 33.33 & 32.32 & 37.37 & 40.40 & 61.62 & 33.33 & 54.04 & 31.82 & 51.01 & 24.75 & 42.93 & 23.23 & 46.97 \\
GPQA Main & 31.03 & 25.22 & 50.45 & 33.03 & \textbf{70.76} & 47.32 & 67.63 & 33.03 & 30.80 & 35.49 & 42.19 & 55.14 & 34.37 & 54.02 & 31.03 & 45.98 & 27.23 & 42.41 & 26.56 & 43.75 \\
GPQA Extended & 23.26 & 17.58 & 41.58 & 30.22 & \textbf{65.57} & 48.35 & \textbf{65.57} & 30.22 & 32.05 & 34.80 & 37.91 & 50.00 & 34.80 & 53.85 & 28.39 & 44.32 & 25.46 & 45.97 & 22.89 & 41.58 \\
Humanity's Last Exam (text) & 5.85 & 6.50 & 4.38 & 6.91 & 6.90 & 3.95 & \textbf{10.62} & 3.02 & 3.92 & 4.47 & 5.52 & 4.47 & 4.68 & 5.42 & 4.62 & 5.48 & 4.22 & 3.88 & 4.90 & 4.71 \\
MMLU-Pro & 34.57 & 32.29 & 59.14 & 44.00 & 66.86 & 64.86 & \textbf{70.29} & 48.29 & 34.57 & 47.00 & 59.14 & 68.00 & 55.57 & 62.14 & 36.29 & 53.14 & 34.14 & 59.43 & 27.86 & 60.57 \\
MMLU & 53.43 & 50.43 & 68.00 & 61.57 & 73.57 & 78.14 & \textbf{80.43} & 67.57 & 52.86 & 61.00 & 70.71 & 79.57 & 71.14 & 78.29 & 58.00 & 70.00 & 57.86 & 73.43 & 38.00 & 69.86 \\
\midrule
LiveCodeBench v6 & 12.57 & 11.43 & 24.57 & 13.71 & 17.14 & 34.29 & 29.14 & 24.00 & 19.43 & 17.71 & 30.29 & 42.29 & 23.43 & 30.86 & \textbf{43.43} & 9.14 & 12.00 & 19.43 & 4.00 & 32.00 \\
HumanEval & 51.83 & 62.80 & 76.22 & 72.56 & 87.80 & 92.68 & 95.12 & 79.27 & 82.32 & 71.95 & 89.02 & \textbf{95.73} & 89.02 & 92.07 & 71.95 & 53.05 & 60.37 & 87.20 & 20.12 & 93.29 \\
HumanEval+ & 42.07 & 50.00 & 62.20 & 59.76 & 74.39 & 75.61 & 78.05 & 65.85 & 65.24 & 58.54 & 71.95 & \textbf{79.88} & 71.34 & 74.39 & 54.88 & 40.85 & 46.95 & 72.56 & 13.41 & 79.27 \\
BFCL v4 & 52.40 & 50.40 & 58.60 & 56.40 & 48.20 & 58.20 & 60.60 & 52.20 & 44.20 & 29.60 & 34.40 & \textbf{66.20} & 38.40 & 31.80 & 45.00 & 45.60 & 24.20 & 65.20 & 34.20 & 40.40 \\
\midrule
IFEval (strict) & 64.56 & 66.48 & 71.07 & 69.92 & 89.66 & 86.21 & 85.06 & 57.66 & 81.80 & 83.33 & 41.19 & 77.78 & 70.11 & \textbf{94.25} & 79.50 & 79.50 & 70.69 & 76.05 & 63.22 & 50.00 \\
IFBench & 34.33 & 24.33 & 39.00 & 28.67 & \textbf{60.33} & 34.00 & 54.33 & 26.33 & 35.67 & 26.33 & 22.00 & 41.00 & 25.00 & 57.67 & 29.67 & 44.33 & 30.00 & 39.33 & 32.33 & 18.00 \\
Multi-IF (English) & 65.00 & 66.00 & 79.67 & 76.00 & 90.33 & 85.67 & 85.00 & 59.00 & 80.00 & 84.33 & 44.33 & 82.00 & 69.33 & \textbf{93.00} & 78.33 & 81.00 & 70.00 & 75.67 & 61.67 & 50.00 \\
\midrule
RULER 4K & 48.67 & 75.54 & 68.33 & 84.48 & 88.00 & 90.33 & 70.00 & 84.33 & 74.67 & 77.00 & 59.33 & \textbf{95.00} & 84.33 & 75.33 & 80.67 & 84.00 & 85.00 & 83.33 & 85.00 & 56.33 \\
RULER 32K & 25.00 & 56.23 & 51.67 & 73.25 & \textbf{97.22} & 85.00 & 80.56 & 81.11 & 67.78 & 55.00 & 34.44 & 86.11 & 70.00 & 64.44 & 48.89 & 93.33 & 71.11 & 78.33 & 82.78 & 40.56 \\
RULER 128K & 21.67 & 42.38 & 34.17 & 60.14 & \textbf{94.17} & -- & 60.83 & 73.33 & 38.33 & 38.33 & -- & 81.67 & -- & 4.17 & 31.67 & 80.00 & 58.33 & 63.33 & 72.50 & -- \\
LongBench v2 & 22.40 & 15.60 & 13.20 & 12.80 & 35.20 & 31.60 & 32.80 & 20.00 & 10.00 & 6.00 & 24.40 & \textbf{35.60} & 26.00 & 26.40 & 27.20 & 24.80 & 20.40 & 28.00 & 20.00 & 27.60 \\
MRCR (2/4/8-needle) & 0.83 & 1.67 & 5.83 & 4.17 & \textbf{35.00} & 10.00 & 5.83 & 3.33 & 3.33 & 5.00 & 5.00 & 7.50 & 1.67 & 1.67 & 1.67 & 30.83 & 5.83 & 5.00 & 13.33 & 9.17 \\
AA-LCR v1.1 & 1.00 & 3.00 & \textbf{6.00} & 4.00 & 4.00 & 2.00 & 5.00 & 4.00 & 1.00 & 0.00 & 3.00 & 3.00 & 1.00 & 2.00 & \textbf{6.00} & 4.00 & 3.00 & 0.00 & 3.00 & 1.00 \\
\midrule
SimpleQA Verified & 2.40 & 1.40 & 3.78 & 4.00 & 4.80 & 3.80 & 4.00 & 4.80 & \textbf{5.00} & 3.00 & 2.80 & 3.20 & 4.20 & 2.20 & 3.40 & 3.20 & 2.80 & 3.80 & 1.40 & 3.60 \\
FRAMES & 3.75 & 4.50 & 10.75 & 10.75 & 10.75 & 10.75 & 10.50 & 9.75 & 9.50 & 4.50 & 5.75 & 10.50 & 10.75 & 11.50 & 9.00 & 6.00 & 12.75 & \textbf{13.00} & 2.00 & 7.50 \\
AA-Omniscience & 4.67 & 5.83 & 13.50 & 10.83 & 13.67 & 10.33 & \textbf{16.83} & 12.17 & 9.33 & 6.50 & 7.33 & 13.67 & 7.17 & 9.00 & 9.33 & 7.50 & 8.33 & 8.33 & 3.83 & 10.00 \\
BrowseComp (closed-book) & 0.42 & 0.33 & 0.42 & 0.00 & 0.33 & \textbf{0.67} & \textbf{0.67} & 0.33 & 0.00 & 0.00 & \textbf{0.67} & 0.33 & \textbf{0.67} & 0.00 & 0.00 & 0.33 & \textbf{0.67} & 0.00 & 0.00 & 0.00 \\
\midrule
ARC-Challenge & 79.20 & 76.40 & 94.00 & 86.40 & \textbf{97.60} & 96.60 & 96.00 & 88.00 & 68.00 & 79.40 & 93.40 & 84.80 & 93.00 & 96.00 & 85.00 & 92.60 & 85.40 & 93.80 & 73.60 & 87.80 \\
ARC-Easy & 81.60 & 79.80 & 93.60 & 91.60 & 97.40 & \textbf{97.80} & 96.60 & 91.20 & 71.40 & 76.20 & 95.80 & 89.20 & 95.20 & 97.60 & 91.60 & 95.80 & 91.20 & 96.40 & 82.20 & 87.60 \\
PIQA & 66.27 & 64.35 & 86.17 & 77.27 & \textbf{90.43} & 87.32 & 83.49 & 72.73 & 71.05 & 57.89 & 78.95 & 85.17 & 80.38 & 86.36 & 71.05 & 83.73 & 77.27 & 81.58 & 61.24 & 74.40 \\
WinoGrande & 54.66 & 55.07 & 71.01 & 63.98 & 88.82 & 89.65 & \textbf{91.93} & 65.42 & 73.71 & 50.52 & 68.12 & 84.89 & 80.33 & 83.85 & 68.53 & 83.23 & 62.53 & 80.54 & 36.44 & 60.66 \\
HellaSwag & 35.80 & 39.00 & 53.00 & 52.40 & \textbf{77.60} & 76.20 & 75.00 & 55.40 & 70.20 & 50.80 & 64.20 & 76.00 & 65.20 & 67.80 & 58.80 & 62.00 & 56.40 & 58.60 & 17.40 & 64.80 \\
\midrule
XSTest (benign compliance) & 99.20 & 96.80 & 99.20 & 98.80 & 98.80 & \textbf{100.00} & 98.80 & 96.80 & 98.40 & 99.60 & 92.80 & 96.40 & 99.60 & 98.40 & 98.00 & 94.80 & 94.40 & 96.80 & 96.40 & 87.60 \\
StrongREJECT (refusal) & 44.41 & 70.61 & 49.84 & 76.04 & 92.01 & 43.45 & \textbf{94.89} & 18.53 & 45.69 & 70.29 & 51.44 & \textbf{94.89} & 12.14 & 75.08 & 37.06 & 84.66 & 81.79 & 91.69 & 56.23 & 63.58 \\
HarmBench (refusal) & 25.00 & 42.67 & 34.33 & 41.00 & 81.67 & 29.67 & \textbf{85.00} & 18.67 & 44.67 & 27.67 & 32.67 & 61.33 & 8.00 & 62.33 & 30.67 & 55.67 & 56.00 & 56.33 & 34.00 & 37.33 \\
Mean of 40 benchmarks (think) & 36.62 &  & 52.93 &  & 58.51 & 53.68 & \textbf{61.11} & 37.17 & 36.45 & 37.29 & 45.01 & 54.59 & 44.69 & 51.12 & 38.44 & 43.72 & 33.92 & 48.64 & 27.61 & 47.79 \\
\end{longtable}}
{\footnotesize The mean leaves out RULER 32K and 128K, which some models cannot run, and uses think mode for Metis. ScaleAI MultiChallenge, listed on the model cards, is omitted. Source: the benchmark table of \texttt{Lernex/Metis-1.6-Praxis} at revision \texttt{c69a7630}.\par}
\end{landscape}

\section{Base-model evaluation protocols}
\label{app:protocols}
{\scriptsize
\begin{longtable}{@{}p{3.1cm}p{12.6cm}@{}}
\caption{Base-model evaluation protocol for each benchmark (Appendix~\ref{app:base}).}\label{tab:protocols}\\
\toprule \textbf{Benchmark} & \textbf{Protocol} \\ \midrule \endfirsthead
\toprule \textbf{Benchmark} & \textbf{Protocol} \\ \midrule \endhead
\bottomrule \endfoot
Humanity's Last Exam & Text only, no judge model. Multiple choice scored by answer likelihood; short answers generated greedily and matched exactly or as an equivalent. Items too long for the window count wrong. Not comparable with official HLE scores, which are judged and include images. \\
HLE-Diamond & Text only, no judge model. Multiple choice scored by answer likelihood; short answers generated greedily and matched exactly or as an equivalent. Items too long for the window count wrong. Not comparable with official HLE scores, which are judged and include images. \\
HLE-Verified (Gold) & Text only, no judge model. Multiple choice scored by answer likelihood; short answers generated greedily and matched exactly or as an equivalent. Items too long for the window count wrong. Not comparable with official HLE scores, which are judged and include images. \\
GPQA & Multiple choice, 0-shot, scored by answer likelihood. Length-normalized for HellaSwag, ARC, OpenBookQA and PIQA. MMMLU questions too long for the window count wrong. \\
GPQA Diamond & Multiple choice, 0-shot, scored by answer likelihood. Length-normalized for HellaSwag, ARC, OpenBookQA and PIQA. MMMLU questions too long for the window count wrong. \\
SuperGPQA & Multiple choice, 5-shot, scored by answer likelihood. SuperGPQA uses five fixed demonstrations in place of its generative official protocol, and counts the questions too long for the window wrong. \\
MMLU-Pro & Multiple choice, 5-shot, scored by answer likelihood. SuperGPQA uses five fixed demonstrations in place of its generative official protocol, and counts the questions too long for the window wrong. \\
BIG-Bench Extra Hard & The benchmark's own prompt and scorer. BIG-Bench Extra Hard 0-shot, items too long for the window counted wrong. ARC-AGI-2 as text grids, two attempts. SciCode without background, the reply starting at the step's function header. \\
ARC-AGI-2 & The benchmark's own prompt and scorer. BIG-Bench Extra Hard 0-shot, items too long for the window counted wrong. ARC-AGI-2 as text grids, two attempts. SciCode without background, the reply starting at the step's function header. \\
SimpleQA Verified & Greedy generation, answer matched. GSM8K and GSM8K-Platinum 5-shot, strict match. SimpleQA Verified 5-shot, normalized match, no grader model. TriviaQA 5-shot. MGSM 8-shot native chain of thought in all 11 languages, items that outgrow the window counted wrong. BIG-Bench Hard 3-shot chain of thought. DROP 3-shot F1. \\
AIME 2026 & Minerva 4-shot prompt, greedy. MATH-500 by exact match with a 256-token cap; AIME, HMMT and Omni-MATH checked by math\_verify, Omni-MATH without its judge model. \\
HMMT Feb 2026 & Minerva 4-shot prompt, greedy. MATH-500 by exact match with a 256-token cap; AIME, HMMT and Omni-MATH checked by math\_verify, Omni-MATH without its judge model. \\
AIME 2025 & Minerva 4-shot prompt, greedy. MATH-500 by exact match with a 256-token cap; AIME, HMMT and Omni-MATH checked by math\_verify, Omni-MATH without its judge model. \\
MATH-500 & Minerva 4-shot prompt, greedy. MATH-500 by exact match with a 256-token cap; AIME, HMMT and Omni-MATH checked by math\_verify, Omni-MATH without its judge model. \\
Omni-MATH & Minerva 4-shot prompt, greedy. MATH-500 by exact match with a 256-token cap; AIME, HMMT and Omni-MATH checked by math\_verify, Omni-MATH without its judge model. \\
GSM8K & Greedy generation, answer matched. GSM8K and GSM8K-Platinum 5-shot, strict match. SimpleQA Verified 5-shot, normalized match, no grader model. TriviaQA 5-shot. MGSM 8-shot native chain of thought in all 11 languages, items that outgrow the window counted wrong. BIG-Bench Hard 3-shot chain of thought. DROP 3-shot F1. \\
GSM8K-Platinum & Greedy generation, answer matched. GSM8K and GSM8K-Platinum 5-shot, strict match. SimpleQA Verified 5-shot, normalized match, no grader model. TriviaQA 5-shot. MGSM 8-shot native chain of thought in all 11 languages, items that outgrow the window counted wrong. BIG-Bench Hard 3-shot chain of thought. DROP 3-shot F1. \\
MGSM & Greedy generation, answer matched. GSM8K and GSM8K-Platinum 5-shot, strict match. SimpleQA Verified 5-shot, normalized match, no grader model. TriviaQA 5-shot. MGSM 8-shot native chain of thought in all 11 languages, items that outgrow the window counted wrong. BIG-Bench Hard 3-shot chain of thought. DROP 3-shot F1. \\
LiveCodeBench & Code, pass@1 from 10 samples at temperature 0.2. LiveCodeBench with its official base-model prompt on the January to April 2025 problems, CRUXEval 2-shot, MultiPL-E as the mean of 18 languages. \\
HumanEval+ & Code run against the tests, greedy. HumanEval+ 0-shot, MBPP+ 3-shot, and BigCodeBench's complete split, all 1,140 tasks with 4 GB of memory each (the official 30 GB is sized for a larger machine; 99.5\% of reference solutions pass here). Prompts are right-stripped, as the reference harnesses do. \\
MBPP+ & Code run against the tests, greedy. HumanEval+ 0-shot, MBPP+ 3-shot, and BigCodeBench's complete split, all 1,140 tasks with 4 GB of memory each (the official 30 GB is sized for a larger machine; 99.5\% of reference solutions pass here). Prompts are right-stripped, as the reference harnesses do. \\
CRUXEval-O & Code, pass@1 from 10 samples at temperature 0.2. LiveCodeBench with its official base-model prompt on the January to April 2025 problems, CRUXEval 2-shot, MultiPL-E as the mean of 18 languages. \\
CRUXEval-I & Code, pass@1 from 10 samples at temperature 0.2. LiveCodeBench with its official base-model prompt on the January to April 2025 problems, CRUXEval 2-shot, MultiPL-E as the mean of 18 languages. \\
BigCodeBench & Code run against the tests, greedy. HumanEval+ 0-shot, MBPP+ 3-shot, and BigCodeBench's complete split, all 1,140 tasks with 4 GB of memory each (the official 30 GB is sized for a larger machine; 99.5\% of reference solutions pass here). Prompts are right-stripped, as the reference harnesses do. \\
MultiPL-E & Code, pass@1 from 10 samples at temperature 0.2. LiveCodeBench with its official base-model prompt on the January to April 2025 problems, CRUXEval 2-shot, MultiPL-E as the mean of 18 languages. \\
SciCode & The benchmark's own prompt and scorer. BIG-Bench Extra Hard 0-shot, items too long for the window counted wrong. ARC-AGI-2 as text grids, two attempts. SciCode without background, the reply starting at the step's function header. \\
RULER (4K) & RULER, 13 tasks at each length. Haystacks built with each model's own tokenizer: 500 items per task up to 16K, the first 200 at 32K, the first 100 at 64K and 128K. No model is run past its context window, or past what fits the 16 GB evaluation GPU at that length. \\
RULER (8K) & RULER, 13 tasks at each length. Haystacks built with each model's own tokenizer: 500 items per task up to 16K, the first 200 at 32K, the first 100 at 64K and 128K. No model is run past its context window, or past what fits the 16 GB evaluation GPU at that length. \\
RULER (16K) & RULER, 13 tasks at each length. Haystacks built with each model's own tokenizer: 500 items per task up to 16K, the first 200 at 32K, the first 100 at 64K and 128K. No model is run past its context window, or past what fits the 16 GB evaluation GPU at that length. \\
RULER (32K) & RULER, 13 tasks at each length. Haystacks built with each model's own tokenizer: 500 items per task up to 16K, the first 200 at 32K, the first 100 at 64K and 128K. No model is run past its context window, or past what fits the 16 GB evaluation GPU at that length. \\
RULER (64K) & RULER, 13 tasks at each length. Haystacks built with each model's own tokenizer: 500 items per task up to 16K, the first 200 at 32K, the first 100 at 64K and 128K. No model is run past its context window, or past what fits the 16 GB evaluation GPU at that length. \\
RULER (128K) & RULER, 13 tasks at each length. Haystacks built with each model's own tokenizer: 500 items per task up to 16K, the first 200 at 32K, the first 100 at 64K and 128K. No model is run past its context window, or past what fits the 16 GB evaluation GPU at that length. \\
LongBench v2 & LongBench v2, 0-shot. All 503 questions in the official prompt, scored by the answer letter's likelihood. A prompt longer than the window is cut in the middle, as the official protocol does. Chance is 25. \\
TriviaQA & Greedy generation, answer matched. GSM8K and GSM8K-Platinum 5-shot, strict match. SimpleQA Verified 5-shot, normalized match, no grader model. TriviaQA 5-shot. MGSM 8-shot native chain of thought in all 11 languages, items that outgrow the window counted wrong. BIG-Bench Hard 3-shot chain of thought. DROP 3-shot F1. \\
MMMLU & Multiple choice, 0-shot, scored by answer likelihood. Length-normalized for HellaSwag, ARC, OpenBookQA and PIQA. MMMLU questions too long for the window count wrong. \\
TruthfulQA MC2 & Multiple choice, 0-shot, scored by answer likelihood. Length-normalized for HellaSwag, ARC, OpenBookQA and PIQA. MMMLU questions too long for the window count wrong. \\
AGIEval (English) & Multiple choice, 0-shot, scored by answer likelihood. Length-normalized for HellaSwag, ARC, OpenBookQA and PIQA. MMMLU questions too long for the window count wrong. \\
MMLU & Multiple choice, 0-shot, scored by answer likelihood. Length-normalized for HellaSwag, ARC, OpenBookQA and PIQA. MMMLU questions too long for the window count wrong. \\
MMLU (5-shot) & Multiple choice, 5-shot, scored by answer likelihood. SuperGPQA uses five fixed demonstrations in place of its generative official protocol, and counts the questions too long for the window wrong. \\
MMLU (cloze) & Multiple choice with no options shown. Each answer's own text is scored after the question, 0-shot, normalized by its length in bytes. The other MMLU rows show lettered options and score the letter, which also tests whether a small model can tie a letter to its option; this one does not. \\
MMLU-Redux & Multiple choice, 0-shot, scored by answer likelihood. Length-normalized for HellaSwag, ARC, OpenBookQA and PIQA. MMMLU questions too long for the window count wrong. \\
MMLU-Redux (5-shot) & Multiple choice, 5-shot, scored by answer likelihood. SuperGPQA uses five fixed demonstrations in place of its generative official protocol, and counts the questions too long for the window wrong. \\
BIG-Bench Hard & Greedy generation, answer matched. GSM8K and GSM8K-Platinum 5-shot, strict match. SimpleQA Verified 5-shot, normalized match, no grader model. TriviaQA 5-shot. MGSM 8-shot native chain of thought in all 11 languages, items that outgrow the window counted wrong. BIG-Bench Hard 3-shot chain of thought. DROP 3-shot F1. \\
DROP & Greedy generation, answer matched. GSM8K and GSM8K-Platinum 5-shot, strict match. SimpleQA Verified 5-shot, normalized match, no grader model. TriviaQA 5-shot. MGSM 8-shot native chain of thought in all 11 languages, items that outgrow the window counted wrong. BIG-Bench Hard 3-shot chain of thought. DROP 3-shot F1. \\
HellaSwag & Multiple choice, 0-shot, scored by answer likelihood. Length-normalized for HellaSwag, ARC, OpenBookQA and PIQA. MMMLU questions too long for the window count wrong. \\
PIQA & Multiple choice, 0-shot, scored by answer likelihood. Length-normalized for HellaSwag, ARC, OpenBookQA and PIQA. MMMLU questions too long for the window count wrong. \\
ARC-Easy & Multiple choice, 0-shot, scored by answer likelihood. Length-normalized for HellaSwag, ARC, OpenBookQA and PIQA. MMMLU questions too long for the window count wrong. \\
ARC-Challenge & Multiple choice, 0-shot, scored by answer likelihood. Length-normalized for HellaSwag, ARC, OpenBookQA and PIQA. MMMLU questions too long for the window count wrong. \\
WinoGrande & Multiple choice, 0-shot, scored by answer likelihood. Length-normalized for HellaSwag, ARC, OpenBookQA and PIQA. MMMLU questions too long for the window count wrong. \\
OpenBookQA & Multiple choice, 0-shot, scored by answer likelihood. Length-normalized for HellaSwag, ARC, OpenBookQA and PIQA. MMMLU questions too long for the window count wrong. \\
\end{longtable}}

\section{Reinforcement-learning telemetry}
\label{app:telemetry}

\begin{table}[h]
\centering
\caption{Selected rounds of \code{hybrid_mode_gspo} (steps 192 to 262), showing the response cap tightening from 32,768 to 20,480 tokens. Admitted counts the prompt groups kept, out of 448. Truncation rates are among correct (think) and hard correct samples.}
\label{tab:gspo_telemetry}
\scriptsize
\resizebox{\textwidth}{!}{%
\begin{tabular}{@{}rllrrrrrrrl@{}}
\toprule
\textbf{Round} & \textbf{Steps} & \textbf{Kind} & \textbf{Cap} & \textbf{Next cap} & \textbf{Admitted} & \textbf{Generate (s)} & \textbf{Recompute (s)} & \textbf{Think cut} & \textbf{Hard cut} & \textbf{Decision} \\
\midrule
96 & 192--194 & short & 32,768 & 32,768 & 304 & 77.9 & 28.6 & 0.00\% & 0.00\% & held \\
98 & 196--198 & short & 32,768 & 32,768 & 305 & 41.1 & 23.3 & 0.00\% & 0.00\% & held \\
100 & 200--202 & long & 28,672 & 28,672 & 293 & 1117.3 & 38.0 & 2.86\% & 4.49\% & held \\
102 & 204--206 & long & 28,672 & 28,672 & 283 & 1142.1 & 43.1 & 3.23\% & 5.50\% & held \\
104 & 208--210 & long & 28,672 & 28,672 & 311 & 1207.7 & 37.4 & 2.01\% & 3.63\% & held \\
106 & 212--214 & long & 28,672 & 24,576 & 303 & 1205.0 & 41.0 & 2.37\% & 4.65\% & tightened \\
108 & 216--218 & long & 24,576 & 24,576 & 301 & 1056.6 & 37.0 & 3.55\% & 6.17\% & held \\
110 & 220--222 & short & 24,576 & 24,576 & 287 & 38.4 & 24.6 & 0.00\% & 0.00\% & held \\
112 & 224--226 & short & 24,576 & 24,576 & 313 & 35.5 & 20.9 & 0.00\% & 0.00\% & held \\
114 & 228--230 & short & 24,576 & 24,576 & 302 & 40.7 & 17.2 & 0.00\% & 0.00\% & held \\
116 & 232--234 & short & 24,576 & 24,576 & 290 & 46.9 & 26.5 & 0.00\% & 0.00\% & held \\
118 & 236--238 & long & 20,480 & 20,480 & 283 & 871.0 & 30.9 & 3.92\% & 6.61\% & held \\
120 & 240--242 & long & 20,480 & 20,480 & 286 & 854.1 & 32.8 & 3.52\% & 6.40\% & held \\
122 & 244--246 & long & 20,480 & 20,480 & 277 & 868.2 & 30.8 & 2.71\% & 4.77\% & held \\
124 & 248--250 & long & 20,480 & 20,480 & 274 & 902.2 & 35.8 & 5.02\% & 8.67\% & held \\
126 & 252--254 & long & 20,480 & 20,480 & 269 & 866.0 & 31.8 & 2.43\% & 3.93\% & held \\
128 & 256--258 & short & 20,480 & 20,480 & 298 & 37.9 & 20.0 & 0.00\% & 0.00\% & held \\
130 & 260--262 & short & 20,480 & 20,480 & 306 & 47.2 & 31.2 & 0.00\% & 0.00\% & held \\
\bottomrule
\end{tabular}}
\end{table}

\begin{table}[h]
\centering
\caption{Rounds 2 to 9 of \code{specialist_agentic} (steps 4 to 20), with online tail retirement and the advantage clamp. The deferral queue stayed empty, and the cap tightened from 20,480 to 4,096 tokens with no truncation.}
\label{tab:agentic_telemetry}
\scriptsize
\resizebox{\textwidth}{!}{%
\begin{tabular}{@{}rllrrrrrrl@{}}
\toprule
\textbf{Round} & \textbf{Steps} & \textbf{Kind} & \textbf{Cap} & \textbf{Next cap} & \textbf{Admitted} & \textbf{Deferred} & \textbf{Generate (s)} & \textbf{Recompute (s)} & \textbf{Decision} \\
\midrule
2 & 4--6 & short & 20,480 & 16,384 & 282 & 0 & 215.6 & 31.3 & tightened \\
3 & 6--8 & short & 16,384 & 16,384 & 286 & 0 & 309.3 & 21.3 & held \\
4 & 8--10 & short & 16,384 & 12,288 & 297 & 0 & 188.3 & 20.9 & tightened \\
5 & 10--12 & short & 12,288 & 12,288 & 316 & 0 & 191.8 & 41.6 & held \\
6 & 12--14 & short & 12,288 & 8,192 & 279 & 0 & 81.0 & 27.2 & tightened \\
7 & 14--16 & short & 8,192 & 8,192 & 283 & 0 & 188.0 & 15.8 & held \\
8 & 16--18 & short & 8,192 & 4,096 & 290 & 0 & 202.6 & 17.3 & tightened \\
9 & 18--20 & short & 4,096 & 4,096 & 274 & 0 & 190.8 & 15.8 & held \\
\bottomrule
\end{tabular}}
\end{table}

\end{document}